\documentclass[11pt,letterpaper]{article}
\usepackage[margin=1in]{geometry}
\usepackage[T1]{fontenc}
\usepackage[utf8]{inputenc}
\usepackage{times}
\usepackage{microtype}
\usepackage[authoryear,round]{natbib}
\setcitestyle{citesep={;},aysep={,},yysep={;}}
\usepackage{xcolor}

\usepackage{amsmath,amsfonts,bm}

\def\1{\bm{1}}

\DeclareMathAlphabet{\mathsfit}{\encodingdefault}{\sfdefault}{m}{sl}
\SetMathAlphabet{\mathsfit}{bold}{\encodingdefault}{\sfdefault}{bx}{n}

\newcommand{\bc}{\begin{center}}
\newcommand{\ec}{\end{center}}
\newcommand{\be}{\begin{equation}}
\newcommand{\ee}{\end{equation}}
\newcommand{\ba}{\begin{array}}
\newcommand{\ea}{\end{array}}
\newcommand{\bea}{\setlength\arraycolsep{1pt}\begin{eqnarray}}
\newcommand{\eea}{\end{eqnarray}}
\newcommand{\ben}{\begin{enumerate}}
\newcommand{\een}{\end{enumerate}}
\newcommand{\bed}{\begin{itemize}}
\newcommand{\eed}{\end{itemize}}

\usepackage{comment}
\usepackage[hidelinks]{hyperref}
\usepackage{url}
\usepackage{amsmath, amssymb, amsthm}
\usepackage{algorithm}
\usepackage{algpseudocode}
\usepackage{booktabs}
\usepackage{longtable}
\usepackage{multirow}
\usepackage{amsmath}
\usepackage{graphicx}
\usepackage{float}
\usepackage{tikz}
\usetikzlibrary{arrows.meta,positioning,fit,backgrounds,calc}
\newcommand{\sd}[1]{{\fontsize{6}{6}\selectfont (#1)}}
\newcommand{\afterdisplay}{}
\usepackage{array}
\usepackage{tabularx}
\usepackage{subcaption} 

\numberwithin{lemma}{subsection}

\newtheorem{theorem}{Theorem}[section]

\newtheorem{proposition}[theorem]{Proposition}
\newtheorem{corollary}[theorem]{Corollary}
\theoremstyle{remark}

\usepackage{enumitem}

\newcommand{\bean}{\setlength\arraycolsep{1pt}\begin{eqnarray*}}
\newcommand{\eean}{\end{eqnarray*}}

\graphicspath{{images/}}

\title{Parameter-Efficient Distributionally Robust Adaptation
of Tabular Foundation Models under Subpopulation Shift}

\author{Seonghwi Kim \qquad Sung Ho Jo \qquad
Minwoo Chae\thanks{Corresponding author.} \\[0.5em]
\normalsize Pohang University of Science and Technology \\
\normalsize\texttt{\{kshwi,tjdgh1813,mchae\}@postech.ac.kr}}
\date{}

\hypersetup{
  pdftitle={Parameter-Efficient Distributionally Robust Adaptation of Tabular Foundation Models under Subpopulation Shift},
  pdfauthor={Seonghwi Kim, Sung Ho Jo, Minwoo Chae},
  pdfsubject={Distributionally robust adaptation of tabular foundation models},
  pdfkeywords={DR-TFM, DR-TabPFN, tabular foundation models, distributionally robust optimization, subpopulation shift}
}

\begin{document}
\addtocontents{toc}{\protect\setcounter{tocdepth}{-1}}

\maketitle
\begin{abstract}
Despite strong mean accuracy, tabular foundation models (TFMs)
can perform poorly on underrepresented groups under subpopulation shift,
where group proportions change between training and deployment.
We propose DR-TFM, a parameter-efficient distributionally robust
adaptation framework that requires no true group annotations.
DR-TFM adjusts attention to labeled context examples by fine-tuning
an existing query scaling network or adding and training one,
while keeping all other parameters fixed.
We instantiate the framework with two robust objectives using
estimated groups or source conditional distributions derived from
training data.
For TabPFN-3, adaptation updates only $0.016\%$ of the pretrained
model's parameters.
Across five tabular benchmarks, DR-TFM achieves substantially higher
average worst-group accuracy than pretrained TFMs and the compared
robust baselines without true group annotations, while maintaining
competitive mean group accuracy.
DR-TFM also improves average worst-group accuracy on ACS Income
and across four additional TFMs.
\end{abstract}

\begingroup
\setlength{\textfloatsep}{8pt plus 2pt minus 2pt}
\setlength{\floatsep}{6pt plus 2pt minus 2pt}
\setlength{\intextsep}{8pt plus 2pt minus 2pt}
\captionsetup[table]{skip=4pt}
\captionsetup[figure]{skip=6pt}
\setlength{\aboverulesep}{1.4pt}
\setlength{\belowrulesep}{2pt}
\section{Introduction}

Tabular foundation models (TFMs), such as the TabPFN\footnote{Unless otherwise specified, TabPFN refers to version~3.}, have emerged as a promising approach to general-purpose prediction on tabular data \citep{hollmann2025accurate,qu2025tabicl,jiang2026tabularsurvey}.
They have demonstrated strong predictive performance, while recent advances have further improved their scalability and computational efficiency~\citep{grinsztajn2026tabpfn}.

However, strong predictive performance does not necessarily imply robustness to distribution shift, where the training and test distributions differ. We focus on subpopulation shift, a practically important form of distribution shift in which the training and test distributions consist of the same underlying subpopulations, which we also refer to as groups, but differ in their mixture proportions \citep{duchi2021learning, sagawa2019distributionally, yang2023change}. For example, the relative frequencies of demographic groups may vary across data collection sites or deployment environments. When some groups are underrepresented in the training data, predictive methods optimized for average performance can exhibit substantial variation in performance across groups, with particularly poor performance on underrepresented groups. As the training mixture becomes more imbalanced, these disparities can become more pronounced, while aggregate performance may obscure degradation in the worst-performing group.

\begin{figure}[t]
    \centering
    \includegraphics[width=0.78\textwidth]{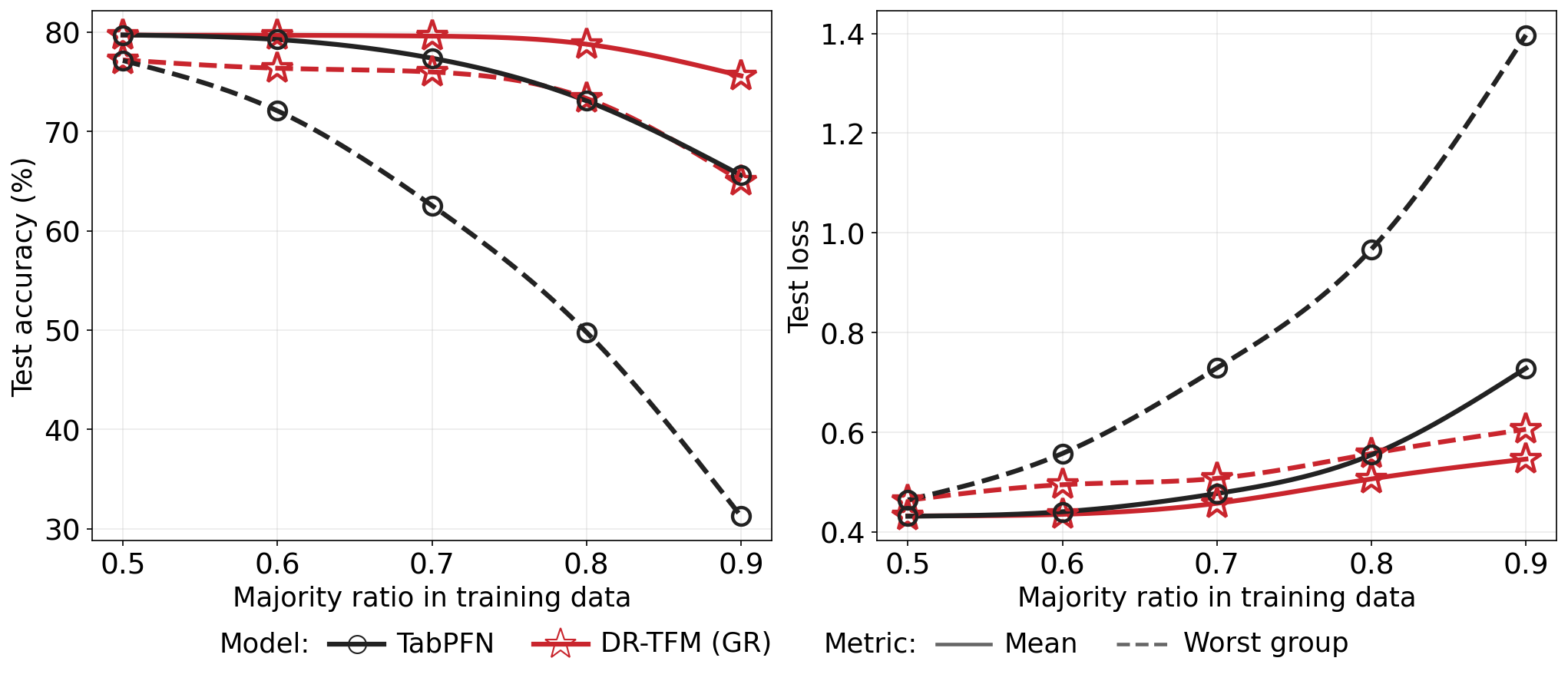}
    \caption{Classification results under subpopulation shift with balanced test groups. TabPFN's worst-group accuracy deteriorates more sharply than its mean group accuracy, whereas DR-TFM (GR) mitigates this degradation without true group labels. Curves show averages over 50 runs.}
    \label{fig:intro_failure}
\end{figure}

Figure~\ref{fig:intro_failure} illustrates this challenge. It shows TabPFN's mean group and worst-group accuracies as the combined proportion of the majority groups in the training data increases, while test groups remain balanced. When the majority groups together account for $90\%$ of the training data, the worst-group accuracy falls to approximately $31\%$, despite a mean group accuracy of approximately $66\%$.

To address this challenge, we propose Distributionally Robust Tabular Foundation Models (DR-TFM), a framework that adapts TFMs to subpopulation shift by updating only a small set of model parameters using distributionally robust optimization (DRO) objectives. We instantiate this framework with two objectives. DR-TFM (GR) uses GroupDRO~\citep{sagawa2019distributionally}. For DR-TFM (MS), we adopt the multi-source DRO objective of \citet{kim2026distributionally}, motivated by the empirical gains reported on spurious-correlation benchmarks. Both objectives rely on auxiliary group or source structure, which we infer or construct from the training data. Specifically, we estimate groups by clustering pretrained TFM embeddings and form source subsets through clustering or resampling, enabling adaptation without true group annotations. To the best of our knowledge, DR-TFM is the first adaptation framework for TFMs that explicitly targets subpopulation shift.

The central idea of DR-TFM is to improve robustness by adapting how
pretrained representations are used for prediction, rather than modifying
the representations themselves.
For classification, we implement this idea through query scaling:
a lightweight neural network learns coordinate-wise rescalings of query
vectors, thereby adjusting how each query attends to labeled context
examples.
Both DRO instantiations optimize this query scaling network while
keeping all remaining model parameters fixed, concentrating robust
adaptation on how the context contributes to predictions.

This targeted adaptation offers computational and memory efficiency,
scalability, and applicability across TFMs.
First, restricting adaptation to the scaling network reduces the
computational and memory costs of parameter updates.
For example, adapting TabPFN-3 updates only 8,320 parameters,
approximately $0.016\%$ of the model.
Second, sharing the network across attention heads allows larger
encoders and more heads without increasing the trainable parameter
count, provided that the head dimension and scaling network architecture
remain fixed.
Third, the mechanism does not depend on query scaling being part of
the pretrained architecture.
We fine-tune the existing query scaling network in TabPFN's
classification decoder; for TFMs without this component, we add
and train one in the final attention layer that attends to the context.

We evaluate DR-TFM through simulations and experiments on five tabular
benchmarks and ACS Income. Across these benchmarks, both instantiations achieve substantially higher average worst-group accuracy than pretrained TabPFN while maintaining
competitive mean group accuracy. Both also require substantially less time than TabPFN fine-tuning. We further apply DR-TFM (GR) to four additional TFMs and
observe consistent improvements in average worst-group accuracy.

Our main contributions can be summarized as follows.
\begin{itemize}[left=1pt]
    \item We propose DR-TFM, a parameter-efficient distributionally robust
    adaptation framework for TFMs under subpopulation shift, instantiated
    with two robust objectives without requiring true group annotations.
    \item We develop an efficient and scalable adaptation strategy that reweights attention through a small, shared query scaling network while keeping pretrained representations fixed.
    \item Through extensive simulations and real-world experiments,
    we demonstrate that both instantiations of DR-TFM improve robustness
    to subpopulation shift, achieving higher average worst-group accuracy
    than the compared baselines that do not use true group annotations,
    while maintaining competitive mean group accuracy.
    \item We demonstrate the applicability of our robust adaptation
    across multiple TFMs, with consistent improvements in average
    worst-group accuracy.
\end{itemize}

\section{Related Work}
\label{sec:related}

\paragraph{Tabular foundation models and their adaptation.}
Tabular foundation models (TFMs) transfer knowledge across datasets
through pretraining and in-context learning.
TabPFN~\citep{hollmann2025accurate,grinsztajn2026tabpfn} and
TabICL~\citep{qu2025tabicl} are pretrained on synthetic datasets,
whereas TabDPT~\citep{ma2025tabdpt} uses real tabular datasets.

TFM adaptation methods modify the context, model parameters, or prediction outputs.
TuneTables~\citep{feuer2024tunetables} compresses the training data into a
compact learned context through prompt tuning.
\citet{thomas2024retrieval} propose TabPFN-$k$NN, which constructs
query-specific contexts through nearest-neighbor retrieval, and LoCalPFN,
which further combines this retrieval with task-specific fine-tuning.
MixturePFN~\citep{xu2025mixture} clusters the training data and fine-tunes
cluster-specific experts, while BETA~\citep{LiuBeta2025} adapts lightweight input encoders
while keeping the pretrained TabPFN transformer fixed.
DistPFN~\citep{lee2026distpfn} targets label shift through test-time
posterior adjustment without updating model parameters.
We address the distinct challenge of subpopulation shift through
parameter-efficient decoder-side adaptation: we keep the pretrained
representations fixed and update only a small query-scaling network that
modifies how each query attends to labeled context examples.

Beyond generic adaptation, recent work has also examined fairness in TFMs.
FairPFN~\citep{robertson2025fairpfn} incorporates causal fairness into
pretraining using synthetically generated causal data, while
\citet{kenfack2026towards} study fairness in tabular in-context learning
through context selection that either balances groups or ranks examples by
the uncertainty in predicting sensitive attributes.
These methods target fairness with respect to explicitly observed sensitive
attributes, whereas we address robustness to changes in latent subpopulation
proportions without access to true group annotations during adaptation or
model selection.

\paragraph{Robust learning under subpopulation shift.}
A broad line of work seeks to improve robustness to subpopulation shifts,
where changes in group proportions can disproportionately degrade performance
on particular subpopulations
\citep{yang2023change}.
In settings where group annotations are available,
GroupDRO~\citep{sagawa2019distributionally} is a representative approach
that minimizes the worst-group loss over predefined groups.
The ambiguity set determines which distribution shifts are considered.
Subsequent work has expanded these sets.
\citet{krueger2021out} allow mixture weights outside the probability simplex,
while \citet{jo2025mitigating} introduce hierarchical ambiguity sets that
account for both inter- and intra-group uncertainty.

Without group annotations, existing approaches instead rely on proxy signals
or latent structure inferred from the data.
JTT~\citep{liu2021just} upweights examples misclassified by an initial model.
Among latent-group approaches,
GEORGE~\citep{sohoni2020no} clusters learned representations within each
class and applies GroupDRO to the resulting pseudo-groups.
BPA~\citep{seo2022unsupervised} and
SPARE~\citep{yang2024identifying} similarly reweight or resample groups
inferred from learned representations or early-training outputs.
In a related direction, \citet{kim2026distributionally} construct
pseudo-sources through subsampling and optimize over mixtures of their
estimated conditional label distributions.
Much of this literature has focused on vision-based benchmarks, while
comparatively less attention has been given to tabular data.
A recent tabular-specific approach is LSR~\citep{tong2025latent}, which
reweights samples using score-based density estimates in a learned latent
space.
Unlike our approach, however, LSR learns its latent representation and
predictor from the training data rather than building on a pretrained TFM.
Despite growing work on both TFM adaptation and subpopulation robustness,
their intersection remains largely unexplored.
To our knowledge, no prior method explicitly adapts TabPFN to subpopulation
shift without true group annotations.

\section{Preliminaries}
\label{sec:preliminaries}

\paragraph{Setup and notation.} Let $X\in\mathcal{X} \subseteq \mathbb{R}^{d} $ and $Y\in\mathcal{Y}$ denote the input and response variables.
Let $P^{\rm tr}$ and $P^{\rm te}$ denote the training and test joint distributions of $(X,Y)$, respectively.
For any joint distribution $P$ of $(X,Y)$, let $P_X$ denote its marginal distribution of $X$, $P_{Y\mid X}$ its conditional distribution of $Y$ given $X$, and $\mathbb{E}_P[\cdot]$ the corresponding expectation.
Let $\mathcal D=\{(x_i,y_i)\}_{i=1}^{n}$ be the labeled training dataset and 
$\mathcal D'_X=\{x'_j\}_{j=1}^{m}$ be the unlabeled test dataset.
For a parametrized predictor $f_\theta$, let $\ell(f_\theta(X),Y)$ denote its loss.
For a positive integer $K$, let $[K]=\{1,\ldots,K\}$ and let $\Delta_{K-1}$ denote the $(K-1)$-dimensional probability simplex.

\paragraph{Subpopulation shift.}
We consider training and test distributions that share the same $K$ subpopulation
distributions $P_1,\ldots,P_K$ but differ in their mixture proportions.
Specifically, $P^{\rm tr}=\sum_{k=1}^{K}\pi_k^{\rm tr}P_k$ and
$P^{\rm te}=\sum_{k=1}^{K}\pi_k^{\rm te}P_k$, where
$\pi^{\rm tr},\pi^{\rm te}\in\Delta_{K-1}$ denote the training and test mixture weight
vectors, respectively, with $\pi^{\rm tr}\neq\pi^{\rm te}$.
Standard empirical risk minimization (ERM) does not explicitly account for such changes
in mixture proportions and can therefore be vulnerable to subpopulation shift.
A detailed theoretical discussion is provided in Appendix~\ref{appendix:rationale}.

\paragraph{Distributionally robust optimization.}
DRO is a promising approach to
addressing such distributional shifts~\citep{mohajerin2018data,duchi2021learning}.
A standard DRO problem solves
\begin{equation}
\min_{\theta}
\sup_{Q\in\mathcal B}
\mathbb E_{Q}
\left[\ell\!\left(f_\theta(X),Y\right)\right],
\label{eq:dro}
\end{equation}
where $\mathcal B$ is an ambiguity set specifying the collection of distributions
over which robustness is sought, often defined as a suitable neighborhood of
the empirical distribution.

\paragraph{Fine-tuning TFM.}
Given the training dataset $\mathcal D$ and a test input $x$, a pretrained TFM with parameters $\theta_0$ produces a prediction $f_{\theta_0}(x;\mathcal D)$, where the labeled dataset $\mathcal D$ serves as the \emph{context} and $x$ serves as the \emph{query}. A pretrained TFM can be further adapted to $\mathcal D$ through
fine-tuning~\citep{grinsztajn2026tabpfn,LiuBeta2025}.
For this purpose, we split $\mathcal D$ into a context set
$\mathcal D^{(c)}$ and a loss-evaluation set $\mathcal D^{(e)}$.
Let $n_c = |\mathcal D^{(c)}|$ and $n_e = |\mathcal D^{(e)}|$, and let
$\widehat P^{(e)}$ denote the empirical distribution of $\mathcal D^{(e)}$.
We partition the model parameters as $\theta=(\theta_F,\theta_L)$, where
$\theta_F$ and $\theta_L$ denote the frozen and learnable parameters, respectively.
Standard fine-tuning initializes both components from their corresponding pretrained
values in $\theta_0$, keeps $\theta_F$ fixed, and solves
\begin{equation}
\min_{\theta_L}
\mathbb E_{\widehat P^{(e)}}
\left[
\ell\!\left(
f_\theta(X;\mathcal D^{(c)}),Y
\right)
\right].
\label{eq:standard_finetuning}
\end{equation}

\section{Proposed Method}
\label{sec:method}
\subsection{Distributionally Robust Formulation}
\label{sec:robust-formulation}

\paragraph{General formulation.}
Combining the fine-tuning objective in \eqref{eq:standard_finetuning} with the
standard DRO formulation in \eqref{eq:dro}, we formulate DR-TFM as
\begin{equation}
\min_{\theta_L}
\sup_{Q\in\mathcal B}
\mathbb E_{ Q}
\left[
\ell\!\left(
f_\theta(X;\mathcal D^{(c)}),Y
\right)
\right].
\label{eq:general_dr_tabpfn}
\end{equation}
A key ingredient is the choice of an ambiguity set $\mathcal B$ that captures
uncertainty arising from subpopulation shift.
Drawing on established formulations in distributionally robust learning, we consider two ambiguity-set constructions tailored to subpopulation shift.
Both rely on auxiliary structure, such as group assignments or source subsets, but this structure is derived from the observed training data rather than supplied as group annotations.
Consequently, DR-TFM does not require true group labels for training.
We present the key ideas below; detailed optimization formulations and algorithms are provided in Appendix~\ref{appendix:method}.

\paragraph{DR-TFM (GR).}
To motivate the construction, suppose first that each observation $i$ in the training dataset $\mathcal D$
is associated with a group index $g_i \in [G]$.
Following the GroupDRO framework~\citep{sagawa2019distributionally}, we
instantiate \eqref{eq:general_dr_tabpfn} with the ambiguity set
\begin{equation}
\mathcal B_{\mathrm{GR}}
=\left\{
Q=\sum_{k=1}^{G}q_k\widehat P_k^{(e)}
\;:\;q\in\Delta_{G-1}
\right\},
\label{eq:gr_ambiguity}
\end{equation}
where $\widehat P_k^{(e)}$ is the empirical distribution of the samples
$(x_i,y_i)\in\mathcal D^{(e)}$ with $g_i=k$.
In practice, however, we do not assume that these group indices are observed. Instead, we estimate them from pretrained TFM representations. Specifically, we first obtain an embedding from the frozen pretrained TFM for each training observation and then cluster the embeddings within each response class using a simple method such as $k$-means or a Gaussian mixture model. The resulting class--cluster pairs define the estimated group indices. In \eqref{eq:gr_ambiguity}, we use these estimated assignments in place of the unobserved \(g_i\).


\paragraph{DR-TFM (MS).}
Motivated by the strong empirical performance of the multi-source DRO formulation of \citet{kim2026distributionally} under subpopulation shift, we construct an ambiguity set from multiple source-conditioned predictive distributions.
We first form source subsets $\mathcal S_1,\ldots,\mathcal S_M\subseteq\mathcal D^{(c)}$ using either resampling or clustering. Using each $\mathcal S_k$ as context, the pretrained TFM with fixed parameters $\theta_0$ defines a source-conditioned predictive distribution $\widehat P^{(k)}_{Y\mid X}(\cdot\mid x)
= f_{\theta_0}(x;\mathcal S_k)$.
For radii $\epsilon_1,\epsilon_2\geq0$, we define
\begin{equation}
\mathcal B_{\mathrm{MS}}
=\left\{
Q=Q_X Q_{Y\mid X}:\begin{array}{l}
Q_{Y\mid X}=\displaystyle\sum_{k=1}^{M}\beta_k\widehat P^{(k)}_{Y\mid X},
\quad\beta\in\Delta_{M-1},\\[2pt]
W_\infty(Q_X,\widehat P_X^{(e)})\leq\epsilon_1,
\quad\|\beta-\bar\beta\|_2\leq\epsilon_2
\end{array}
\right\},
\label{eq:ms_ambiguity}
\end{equation}
where $W_\infty$ denotes the infinite-order Wasserstein distance and $\bar\beta$ is the uniform probability vector.
The radii $\epsilon_1$ and $\epsilon_2$ control the uncertainty in the input marginal and the source-mixture weights, respectively.



\subsection{Parameter-Efficient Adaptation}
\label{sec:efficient-adaptation}

\begin{samepage}
We now describe a parameter-efficient adaptation procedure for both DRO
instantiations defined by \eqref{eq:gr_ambiguity} and \eqref{eq:ms_ambiguity},
focusing on TabPFN.
We optimize the objective in~\eqref{eq:general_dr_tabpfn} by fine-tuning only
the \emph{query scaling function} $s_{\theta_L}$,
parameterized by a small neural network within the classification decoder
(Figure~\ref{fig:module}).
The function produces scaling factors that multiply the query vector
element-wise, thereby reweighting attention to context examples.
\par
\end{samepage}

\begin{figure}[t]
    \centering
    \includegraphics[width=\textwidth,trim=0bp 7bp 0bp 4bp,clip]{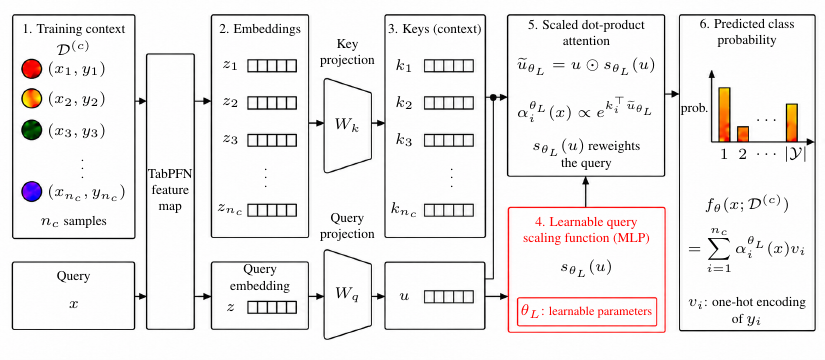}
    \caption{Query scaling for classification during adaptation. A query is
    predicted from the labeled training context. The scaling function $s_{\theta_L}$ adjusts
    the attention weights used to aggregate context labels. Only $\theta_L$ is
    updated; $\theta_F$ remains fixed.}
    \label{fig:module}
\end{figure}

The key idea is to improve robustness by adapting how pretrained representations are used for prediction while preserving the representations themselves. TabPFN’s pretrained embeddings have been shown to capture useful latent structure in tabular data~\citep{ye2025closer,grinsztajn2026tabpfn}. However, these informative embeddings alone do not ensure accurate predictions for underrepresented subpopulations. We implement this idea through query scaling, which adjusts how each query attends to labeled context examples. We fine-tune only the scaling function under either DRO objective while keeping all remaining model parameters fixed, concentrating adaptation on how the context contributes to predictions.

This adaptation strategy offers three advantages: computational and memory
efficiency, scalability, and applicability across TFMs.
First, we update only $8{,}320$ parameters ($0.016\%$ of TabPFN), compared with
nearly all parameters in full-model fine-tuning.
This effectively reduces the computational and memory costs of parameter updates.
Second, the query scaling network is shared across attention heads.
The number of trainable parameters does not increase with encoder size or the
number of heads, provided that the head dimension and scaling network
architecture remain fixed.
Third, the same mechanism can be applied to other TFMs: we fine-tune an existing
query scaling network when available, or add and train one to scale queries
before attention weights are computed.

We next describe how query scaling reweights attention to context examples.
Specifically, we consider a single query $x$ from a sample
$(x,y)\in\mathcal D^{(e)}$, given the context set $\mathcal D^{(c)}$.
Let $z$ denote the query embedding and $z_1,\ldots,z_{n_c}$ the context embeddings,
all obtained from the pretrained TabPFN feature map.
For a single attention head, the fixed projections $W_k$ and $W_q$
map these embeddings to context keys $k_i\in\mathbb R^{r}$
and a query vector $u\in\mathbb R^{r}$,
where $r$ is the dimension of the key and query vectors in each attention head.
Query scaling is applied as follows:
\begin{equation}
k_i=W_kz_i,\qquad u=W_qz,\qquad
\widetilde u_{\theta_L}=u\odot s_{\theta_L}(u).
\label{eq:query-scaling}
\end{equation}
Here $s_{\theta_L}:\mathbb R^{r}\to\mathbb R^{r}$ assigns a scaling factor to each
coordinate of $u$, and $\odot$ denotes element-wise multiplication. Then, we compute the attention weights and class probabilities as
\[
\alpha_i^{\theta_L}(x)
=
\frac{\exp(k_i^\top\widetilde u_{\theta_L})}
{\sum_{j=1}^{n_c}\exp(k_j^\top\widetilde u_{\theta_L})},
\qquad
f_\theta(x;\mathcal D^{(c)})
=\sum_{i=1}^{n_c}\alpha_i^{\theta_L}(x)v_i,
\]
where $v_i\in\{0,1\}^{|\mathcal Y|}$ is the one-hot vector corresponding to the
context label $y_i$.
The probability assigned to each class is the total attention weight on context
samples with that label. 

To see how scaling changes the relative contribution of two context samples, note that
\begin{equation*}
\log\frac{\alpha_i^{\theta_L}(x)}{\alpha_j^{\theta_L}(x)}
=(k_i-k_j)^\top\bigl(u\odot s_{\theta_L}(u)\bigr).
\end{equation*}
A positive right-hand side indicates that context example $i$ receives more
attention than example $j$; a negative value indicates the reverse.
Query scaling can change this attention ratio by rescaling individual
dimensions of the query vector.

Appendix~\ref{appendix:scaling-implementation} specifies the scaling network
and its adapted parameters, while Appendix~\ref{appendix:regression}
discusses the regression setting.

\begin{figure}[t]
    \centering
    \includegraphics[width=0.80\textwidth,trim=0bp 12bp 0bp 5bp,clip]{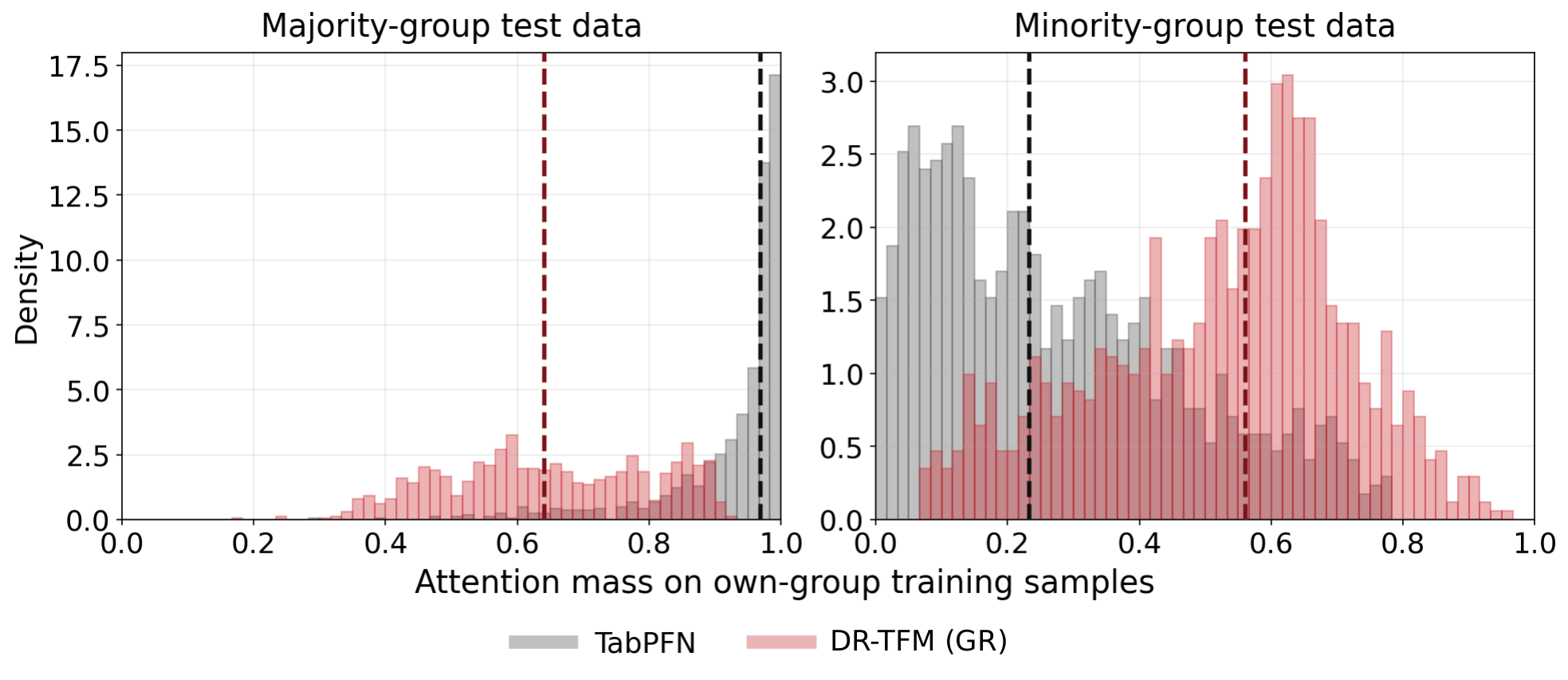}
    \caption{Histograms of the sum of attention weights assigned to each test sample's
    own true group, for majority (left) and minority (right) test samples.
    Gray and red denote TabPFN and DR-TFM (GR), respectively;
    dashed lines mark medians.
    For adaptation, DR-TFM estimates groups using a Gaussian mixture model (GMM).}
    \label{fig:simulation}
\end{figure}

\section{Experiments}
\subsection{Simulation}

In this section, we analyze changes in attention weights before and after adaptation.
We use the same data-generating model as in Figure~\ref{fig:intro_failure}.
Specifically, the training distribution assigns probability $45\%$ to each of two majority
groups and $5\%$ to each of two minority groups. At test time, all four groups
have probability $25\%$.
Only group proportions change; the distribution within each group remains fixed.
Appendix~\ref{appendix:simulation} provides the simulation details.

For each test input $x'\in\mathcal D'_X$ with true group $g'$, we first compute
the decoder attention weights $\alpha_i(x')$ over the context samples.
We then sum these weights over context samples in the same true group to obtain
$A_{\mathrm{own}}(x')=\sum_{i:\,g_i=g'}\alpha_i(x')$.
Figure~\ref{fig:simulation} shows histograms of $A_{\mathrm{own}}(x')$ separately
for majority- and minority-group test samples.
Since the attention weights sum to one, the remaining weight is assigned to context
samples in the other three groups.
Before adaptation, the median own-group attention weight is $0.97$ for majority
test samples and $0.23$ for minority test samples. Thus, majority test samples
predominantly attend to their own group, whereas minority test samples
predominantly attend to other groups.

Updating only the query scaling function under our DRO objective changes the
attention weights assigned to context samples.
After adaptation, the median attention weights on the own group and all other
groups change from $(0.97, 0.03)$ to $(0.64, 0.36)$ for majority test samples.
For minority test samples, they change from $(0.23, 0.77)$ to $(0.56, 0.44)$.
After adaptation, minority test samples assign a much larger share of their
attention to context samples in their own group.

\begin{table}[t]
\caption{Classification results on five tabular benchmarks for methods without true group labels. Throughout the tables, Avg.\ denotes the average of each metric across the datasets or settings in its block. Bold and underlined values indicate the best and second-best results in each column. $^{\dagger}$: accuracy results from \citet{tong2025latent}.}
\label{tab:main-table1}
\centering
\fontsize{7}{8}\selectfont
\renewcommand{\arraystretch}{0.92}
\setlength{\tabcolsep}{4.5pt}
\begin{tabular*}{\textwidth}{@{\extracolsep{\fill}}lcccccccccc|cc}
\toprule
\multirow{2}{*}{Methods} & \multicolumn{2}{c}{Adult} & \multicolumn{2}{c}{Bank} & \multicolumn{2}{c}{Default} & \multicolumn{2}{c}{Shoppers} & \multicolumn{2}{c|}{Taxi} & \multicolumn{2}{c}{Avg.} \\
\cmidrule(lr){2-3}\cmidrule(lr){4-5}\cmidrule(lr){6-7}\cmidrule(lr){8-9}\cmidrule(lr){10-11}\cmidrule(lr){12-13}
 & Mean & Worst & Mean & Worst & Mean & Worst & Mean & Worst & Mean & Worst & Mean & Worst \\
\midrule
\multicolumn{13}{@{}l}{\emph{Standard methods}}\\
ERM-MLP & 73.45 & 45.29 & 71.86 & 42.03 & 63.70 & 28.52 & 76.44 & 50.11 & 75.27 & 57.37 & 72.14 & 44.67\\
XGBoost & 77.09 & 52.58 & 68.69 & 33.87 & 64.78 & 30.72 & 78.03 & 55.42 & 75.81 & 57.57 & 72.88 & 46.03\\
CatBoost & 76.50 & 51.33 & 68.94 & 34.43 & 64.75 & 31.03 & 77.50 & 54.24 & 76.04 & 58.71 & 72.75 & 45.95\\
\midrule
\multicolumn{13}{@{}l}{\emph{TFMs}}\\
TabICL & 76.40 & 51.17 & 71.90 & 40.38 & 64.70 & 31.48 & 78.71 & 56.70 & 76.07 & 57.86 & 73.56 & 47.52\\
BETA & 74.70 & 47.36 & 66.26 & 30.51 & 64.46 & 30.93 & 78.44 & 56.84 & 75.03 & 57.14 & 71.78 & 44.56\\
DistPFN & 75.02 & 48.37 & 72.63 & 42.04 & 64.61 & 30.71 & 78.00 & 54.63 & 75.76 & 57.26 & 73.20 & 46.60\\
EXAONE & 77.47 & 53.46 & 72.67 & 42.35 & 64.74 & 31.10 & 77.74 & 53.72 & 75.99 & 57.77 & 73.72 & 47.68\\
TabFM & 77.39 & 52.88 & 73.75 & 44.42 & 65.02 & 31.91 & 78.28 & 54.96 & 76.34 & 58.11 & 74.15 & 48.46\\
Causilo & 77.25 & 52.89 & 72.95 & 43.27 & 64.70 & 31.27 & 78.60 & 56.44 & 76.28 & 58.98 & 73.95 & 48.57\\
\midrule
\multicolumn{13}{@{}l}{\emph{Robust methods}}\\
EIIL$^{\dagger}$ & 69.37 & 38.97 & 61.85 & 21.69 & 65.05 & 28.91 & 74.82 & 46.18 & 69.52 & 58.00 & 68.12 & 38.75\\
LSR$^{\dagger}$ & 74.33 & 54.79 & 69.50 & 41.28 & 62.62 & 38.78 & 79.68 & 60.73 & 67.85 & 63.14 & 70.80 & 51.74\\
GEORGE & 77.48 & 60.18 & 81.72 & 73.32 & 62.73 & 56.60 & 77.47 & 71.56 & 64.03 & 52.82 & 72.69 & 62.89\\
BPA & 64.21 & 54.34 & 65.52 & 54.04 & 57.45 & 51.97 & 62.88 & 50.18 & 59.96 & 50.19 & 62.00 & 52.14\\
SPARE & 69.81 & 46.64 & 76.34 & 57.52 & 57.68 & 34.70 & 79.93 & 69.67 & 66.08 & 51.24 & 69.97 & 51.95\\
\midrule
\multicolumn{13}{@{}l}{\emph{TabPFN adaptation}}\\
TabPFN & 75.33 & 49.17 & 73.07 & 42.96 & 64.77 & 31.09 & 79.59 & 58.29 & 76.01 & 58.19 & 73.75 & 47.94\\
TabPFN (fine-tuned) & 76.33 & 50.80 & 73.10 & 42.99 & 65.03 & 31.71 & 78.88 & 56.25 & 76.11 & 58.23 & 73.89 & 48.00\\
\cmidrule[0.03em](lr){1-13}
DR-TFM (GR) & \underline{79.64} & \textbf{67.90} & \underline{85.01} & \textbf{80.14} & \textbf{69.42} & \textbf{60.37} & \textbf{85.79} & \textbf{81.66} & \underline{76.69} & \textbf{72.79} & \textbf{79.31} & \textbf{72.57}\\
DR-TFM (MS) & \textbf{79.78} & \underline{67.41} & \textbf{85.02} & \underline{78.41} & \underline{69.18} & \underline{57.86} & \underline{84.48} & \underline{77.09} & \textbf{77.13} & \underline{65.80} & \underline{79.12} & \underline{69.31}\\
\bottomrule
\end{tabular*}
\end{table}

\subsection{Real data analysis}

We follow the experimental setup of \citet{tong2025latent}, adopting their
dataset settings, evaluation group definitions, and evaluation metrics.
Details are provided in Appendix~\ref{appendix:setup}.

\subsubsection{Experimental setup}
\paragraph{Datasets.}
We evaluate the proposed method on Adult~\citep{becker1996adult},
Bank~\citep{moro2012bank}, Default~\citep{yeh2016default},
Shoppers~\citep{sakar2018online}, Taxi~\citep{navas2018taxi},
and ACS Income~\citep{ding2021retiring}.
For ACS Income, we consider three within-state settings and three transfers between
Arizona, Massachusetts, and Michigan.
All tasks are binary classification. Each evaluated binary attribute is paired
with the label to define four groups.

\paragraph{Baselines and evaluation.}
We compare DR-TFM with 30 baseline methods, including standard classifiers,
TFMs and their adaptations, and robust learning methods.
Tables~\ref{tab:main-table1} and~\ref{tab:main-table2} report results for our
TabPFN adaptations and representative baselines that do not use true group labels.
Appendix~\ref{appendix:baselines} describes the baselines.
We evaluate performance using mean group accuracy $\mathrm{Acc}_{\mathrm{mean}}$ and
worst-group accuracy $\mathrm{Acc}_{\mathrm{worst}}$.
Let $\mathcal A$ denote the set of evaluated attributes and $a_{A,g}$ the test
accuracy of group $g\in[4]$ for attribute $A\in\mathcal A$. These metrics are defined as
$\mathrm{Acc}_{\mathrm{mean}}
=\frac{1}{|\mathcal A|}\sum_{A\in\mathcal A}
\frac{1}{4}\sum_{g=1}^{4}a_{A,g}$ and
$\mathrm{Acc}_{\mathrm{worst}}
=\frac{1}{|\mathcal A|}\sum_{A\in\mathcal A}
\min_{g\in[4]}a_{A,g}$.
Both metrics are expressed as percentages and denoted by Mean and Worst in the tables.

DR-TFM does not require true group labels for either training or model selection.
We use $k$-means to estimate groups for GR and construct source subsets for MS.
Appendix~\ref{appendix:setup} provides the training protocols, hyperparameter
selection procedures, and ablations on the number of clusters.

\subsubsection{Results}
Tables~\ref{tab:main-table1} and~\ref{tab:main-table2} summarize benchmark
performance on the five tabular benchmarks and ACS Income, respectively.
Full results, including additional baselines and standard deviations, are reported in
Appendices~\ref{appendix:full-tables} and~\ref{appendix:tfm-extension-results}.
We report accuracy differences in percentage points (pp).

\paragraph{Five tabular benchmarks.}
Table~\ref{tab:main-table1} shows that our proposed methods achieve the best performance
on all five tabular benchmarks.
The corresponding averages of worst-group accuracies are $72.57\%$ and $69.31\%$,
yielding improvements of $24.63$ and $21.37$ pp over TabPFN.
DR-TFM (GR) and DR-TFM (MS) improve the average of mean group accuracies
over TabPFN by $5.56$ and $5.37$ pp and the average of worst-group accuracies
over TabPFN fine-tuning by $24.57$ and $21.31$ pp, respectively.
The corresponding gains over GEORGE ($k$-means) are $9.68$ and $6.42$ pp.

\begin{table}[t]
\caption{Classification results on ACS Income for methods without true group labels:
within-state prediction (left)
and transfers between states (right).}
\label{tab:main-table2}
\centering
\fontsize{7}{8}\selectfont
\renewcommand{\arraystretch}{0.92}
\setlength{\tabcolsep}{1.4pt}
\begin{tabular*}{\textwidth}{@{\extracolsep{\fill}}lcccccc|cc|cccccc|cc}
\toprule
\multirow{2}{*}{Methods} & \multicolumn{2}{c}{AZ} & \multicolumn{2}{c}{MA} & \multicolumn{2}{c|}{MI} & \multicolumn{2}{c|}{Avg.} & \multicolumn{2}{c}{AZ $\rightarrow$ MA} & \multicolumn{2}{c}{MA $\rightarrow$ MI} & \multicolumn{2}{c|}{MI $\rightarrow$ AZ} & \multicolumn{2}{c}{Avg.} \\
\cmidrule(lr){2-3}\cmidrule(lr){4-5}\cmidrule(lr){6-7}\cmidrule(lr){8-9}\cmidrule(lr){10-11}\cmidrule(lr){12-13}\cmidrule(lr){14-15}\cmidrule(lr){16-17}
 & Mean & Worst & Mean & Worst & Mean & Worst & Mean & Worst & Mean & Worst & Mean & Worst & Mean & Worst & Mean & Worst \\
\midrule
\multicolumn{17}{@{}l}{\emph{Standard methods}}\\
ERM-MLP & 77.20 & 59.42 & 80.82 & 73.82 & 75.69 & 57.99 & 77.90 & 63.74 & 76.82 & 60.23 & 79.13 & 71.34 & 75.56 & 56.31 & 77.17 & 62.63\\
XGBoost & 77.64 & 58.99 & 82.26 & 76.38 & 77.50 & 62.77 & 79.13 & 66.05 & 77.96 & 61.99 & 79.67 & 72.01 & 76.82 & 58.99 & 78.15 & 64.33\\
CatBoost & 77.35 & 58.26 & 81.97 & 75.59 & 77.27 & 62.43 & 78.86 & 65.43 & 77.66 & 61.15 & 79.60 & \underline{72.31} & 76.68 & 58.71 & 77.98 & 64.06\\
\midrule
\multicolumn{17}{@{}l}{\emph{TFMs}}\\
TabICL & 78.14 & 59.46 & 82.89 & 76.50 & 77.84 & 64.30 & 79.62 & 66.75 & 78.03 & 62.08 & 79.73 & 71.78 & 76.98 & 59.69 & 78.25 & 64.52\\
BETA & 76.73 & 58.08 & 80.67 & 73.25 & 76.80 & 62.12 & 78.07 & 64.49 & 77.25 & 60.70 & 79.39 & \textbf{72.55} & 76.30 & 58.35 & 77.65 & 63.87\\
DistPFN & 78.09 & 59.55 & \underline{82.98} & 76.96 & 77.82 & 63.78 & 79.63 & 66.76 & \underline{78.76} & 65.16 & 79.43 & 71.58 & 76.70 & 58.56 & 78.29 & 65.10\\
EXAONE & 78.06 & 58.93 & 82.70 & 76.05 & 77.80 & 63.50 & 79.52 & 66.16 & 78.06 & 61.52 & \underline{79.85} & 72.02 & 76.90 & 59.03 & 78.27 & 64.19\\
TabFM & 77.90 & 59.64 & 82.95 & \textbf{77.10} & 78.05 & 64.60 & 79.63 & 67.11 & 77.97 & 61.11 & \textbf{79.90} & 72.14 & 77.09 & 59.86 & 78.32 & 64.37\\
Causilo & 77.77 & 59.45 & 82.87 & 76.52 & 77.77 & 64.04 & 79.47 & 66.67 & 77.90 & 61.53 & \underline{79.85} & 72.02 & 76.93 & 59.31 & 78.23 & 64.29\\
\midrule
\multicolumn{17}{@{}l}{\emph{Robust methods}}\\
EIIL$^{\dagger}$ & 74.98 & 55.22 & 78.18 & 68.42 & 75.60 & 63.77 & 76.25 & 62.47 & 75.35 & 56.80 & 76.92 & 65.62 & 74.78 & 57.32 & 75.68 & 59.91\\
LSR$^{\dagger}$ & 76.28 & 66.42 & 78.35 & 73.33 & 75.53 & 67.10 & 76.72 & 68.95 & 75.00 & 62.85 & 74.75 & 68.75 & 74.73 & 64.10 & 74.83 & 65.23\\
GEORGE & 74.76 & 58.70 & 76.65 & 68.82 & 72.82 & 66.52 & 74.74 & 64.68 & 73.25 & 61.69 & 73.22 & 64.80 & 74.21 & 62.10 & 73.56 & 62.86\\
BPA & 63.26 & 54.44 & 65.78 & 58.90 & 63.42 & 58.18 & 64.15 & 57.17 & 65.65 & 59.36 & 62.72 & 57.64 & 63.57 & 58.96 & 63.98 & 58.65\\
SPARE & 74.35 & \underline{69.94} & 77.51 & 70.17 & 71.05 & 61.67 & 74.31 & 67.26 & 74.23 & 65.44 & 73.23 & 65.67 & 70.27 & 58.96 & 72.58 & 63.36\\
\midrule
\multicolumn{17}{@{}l}{\emph{TabPFN adaptation}}\\
TabPFN & 78.19 & 59.87 & \textbf{83.01} & 77.02 & 77.72 & 63.50 & 79.64 & 66.80 & 77.95 & 61.57 & 79.69 & 71.86 & 76.81 & 59.08 & 78.15 & 64.17\\
TabPFN (fine-tuned) & 78.23 & 60.35 & 82.95 & \underline{77.03} & 78.23 & 65.02 & \underline{79.80} & 67.47 & 77.95 & 62.05 & 79.62 & 72.20 & 76.76 & 58.77 & 78.11 & 64.34\\
\cmidrule[0.03em](lr){1-17}
DR-TFM (GR) & \underline{78.81} & \textbf{70.36} & 79.96 & 70.93 & \underline{78.87} & \textbf{71.90} & 79.21 & \underline{71.06} & 78.43 & \underline{69.70} & 78.73 & 66.43 & \underline{78.46} & \underline{68.84} & \underline{78.54} & \underline{68.32}\\
DR-TFM (MS) & \textbf{79.46} & 69.30 & 81.24 & 72.75 & \textbf{79.30} & \underline{71.42} & \textbf{80.00} & \textbf{71.16} & \textbf{80.08} & \textbf{73.75} & 78.82 & 70.53 & \textbf{78.55} & \textbf{71.38} & \textbf{79.15} & \textbf{71.89}\\
\bottomrule
\end{tabular*}
\end{table}

\begin{table}[t]
\caption{Runtime in seconds and peak GPU memory in GB per result, each averaged over the
five tabular benchmarks and three seeds. “ft” denotes fine-tuning.
Runtime includes hyperparameter selection.}
\label{tab:main-runtime}
\centering
\scriptsize
\setlength{\tabcolsep}{1pt}
\begin{tabular*}{\textwidth}{@{\extracolsep{\fill}}lcccccc!{\hspace{6pt}\vrule\hspace{-6pt}}cc}
\toprule
 & GEORGE & LSR & MixturePFN & BETA & TabPFN & TabPFN (ft) & \textbf{DR-TFM (GR)} & \textbf{DR-TFM (MS)} \\
\midrule
Seconds & 77.0 & 5,156.4 & 229.5 & 1,020.6 & 5.4 & 214.1 & 23.8 & 30.6 \\
Memory (GB) & 1.21 & 3.12 & 2.17 & 14.33 & 0.61 & 9.25 & 1.07 & 0.98 \\
\bottomrule
\end{tabular*}
\end{table}

\paragraph{ACS Income.}
In Table~\ref{tab:main-table2}, our proposed methods achieve the best and second-best
averages of worst-group accuracies in both within-state and transfer settings.
Within states, DR-TFM (GR) and DR-TFM (MS) achieve $71.06\%$ and $71.16\%$,
exceeding TabPFN by $4.26$ and $4.36$ pp, respectively.
Across transfers, they achieve $68.32\%$ and $71.89\%$.
These exceed TabPFN by $4.15$ and $7.72$ pp and LSR by $3.09$ and $6.66$ pp, respectively.
Both also outperform TabPFN and fine-tuned TabPFN in the average of mean group
accuracies in the transfer settings.

\paragraph{Computational and memory efficiency.}
Table~\ref{tab:main-runtime} reports runtime in seconds and peak GPU memory in GB
per result, each averaged over the five tabular benchmarks.
DR-TFM (GR) and DR-TFM (MS) require $23.8$ and $30.6$ seconds per result,
respectively. These correspond to $9.0$- and $7.0$-fold speedups over
TabPFN standard fine-tuning ($214.1$ seconds), and $42.9$- and $33.3$-fold speedups over
BETA ($1{,}020.6$ seconds).
Both also use less GPU memory: their average peak memory is $1.07$ and $0.98$ GB,
respectively, compared with $9.25$ GB for TabPFN standard fine-tuning.
Further details are provided in Appendix~\ref{appendix:runtime}.

\paragraph{Training objective and adaptation strategy.}
Table~\ref{tab:main-objswap} compares five fine-tuning strategies under
ERM~\eqref{eq:standard_finetuning} and DRO~\eqref{eq:gr_ambiguity}.
Our DRO-based adaptation substantially improves the average of worst-group accuracies
over ERM across all five strategies.
Under DRO-based adaptation, query scaling achieves the highest average of
worst-group accuracies ($72.57\%$) with the fewest trainable parameters
($0.016\%$ of the model).
Its accuracy is slightly higher than that of decoder fine-tuning ($71.93\%$),
while using approximately $1/51$ as many trainable parameters.
Appendix~\ref{appendix:objswap} provides the adaptation settings and per-dataset results.

\begin{table}[!t]
\captionsetup{skip=4pt}
\caption{Average of worst-group accuracies over five tabular benchmarks and three seeds.
Trainable (\%) is the percentage of model parameters updated.
Encoder refers to all model parameters outside the decoder.
Parenthesized values denote improvements over ERM in percentage points.
Seconds and Memory (GB) are runtime and peak GPU memory per result under DRO-based adaptation.}
\label{tab:main-objswap}
\centering
\fontsize{8}{9}\selectfont
\setlength{\tabcolsep}{4pt}
\begin{tabular*}{\textwidth}{@{\extracolsep{\fill}}lccccc}
\toprule
Fine-tuned parameters & Trainable (\%) & ERM (\%) & DRO (\%) & Seconds & Memory (GB) \\
\midrule
Input adapter & 0.063--0.155 & 45.45 & 66.64\,{\fontsize{6}{6}\selectfont (+21.19)} & 43.5 & 9.80 \\
Encoder & 99.196 & 48.39 & 68.92\,{\fontsize{6}{6}\selectfont (+20.53)} & 241.2 & 10.22 \\
Full backbone & $\approx 100$ & 48.00 & 68.98\,{\fontsize{6}{6}\selectfont (+20.98)} & 246.6 & 10.31 \\
Decoder & 0.804 & 48.68 & 71.93\,{\fontsize{6}{6}\selectfont (+23.25)} & 24.0 & 1.10 \\
\cmidrule[0.03em](lr){1-6}
\textbf{Query scaling} & 0.016 & 47.34 & 72.57\,{\fontsize{6}{6}\selectfont (+25.23)} & 23.8 & 1.07 \\
\bottomrule
\end{tabular*}
\end{table}

\begin{table}[!t]
\captionsetup{skip=4pt}
\caption{GR adaptation across tabular foundation models without true group labels.
Results are averaged over three seeds.
``+ Ours'' denotes DR-TFM (GR) applied to the model above, with parentheses indicating the percentage of model parameters updated.
Bold parentheses report the change from the corresponding original model, in percentage points.}
\label{tab:tfm-extension}
\centering
\fontsize{7}{8.5}\selectfont
\renewcommand{\arraystretch}{1.05}
\setlength{\tabcolsep}{1.5pt}
\newcommand{\gain}[1]{{\bfseries (#1)}}
\begin{tabularx}{\textwidth}{@{}l*{2}{>{\centering\arraybackslash}X}|*{2}{>{\centering\arraybackslash}X}|*{2}{>{\centering\arraybackslash}X}@{}}
\toprule
\multirow{2}{*}{Method} & \multicolumn{2}{c|}{Five tabular benchmarks (Avg.)} & \multicolumn{2}{c|}{ACS within-state (Avg.)} & \multicolumn{2}{c}{ACS transfer (Avg.)} \\
\cmidrule(lr){2-3}\cmidrule(lr){4-5}\cmidrule(lr){6-7}
 & Mean & Worst & Mean & Worst & Mean & Worst \\
\midrule
TabPFN & 73.75 & 47.94 & 79.64 & 66.80 & 78.15 & 64.17\\
\hspace{0.5em}+ Ours (0.016\%) & 79.31\,\gain{+5.56} & 72.57\,\gain{+24.63} & 79.21\,\gain{-0.43} & 71.06\,\gain{+4.26} & 78.54\,\gain{+0.39} & 68.32\,\gain{+4.15}\\
\midrule
TabPFN-3.5 & 74.39 & 49.26 & 79.79 & 67.10 & 78.25 & 64.34\\
\hspace{0.5em}+ Ours (0.004\%) & 79.42\,\gain{+5.03} & 71.47\,\gain{+22.21} & 79.88\,\gain{+0.09} & 72.48\,\gain{+5.38} & 79.19\,\gain{+0.94} & 71.51\,\gain{+7.17}\\
\midrule
EXAONE & 73.72 & 47.68 & 79.52 & 66.16 & 78.27 & 64.19\\
\hspace{0.5em}+ Ours (0.020\%) & 79.44\,\gain{+5.72} & 69.45\,\gain{+21.77} & 80.28\,\gain{+0.76} & 71.39\,\gain{+5.23} & 79.63\,\gain{+1.36} & 70.37\,\gain{+6.18}\\
\midrule
TabFM & 74.15 & 48.46 & 79.63 & 67.11 & 78.32 & 64.37\\
\hspace{0.5em}+ Ours (0.002\%) & 79.43\,\gain{+5.28} & 66.47\,\gain{+18.01} & 79.65\,\gain{+0.02} & 69.41\,\gain{+2.30} & 78.98\,\gain{+0.66} & 66.60\,\gain{+2.23}\\
\midrule
Causilo & 73.95 & 48.57 & 79.47 & 66.67 & 78.23 & 64.29\\
\hspace{0.5em}+ Ours (0.046\%) & 78.47\,\gain{+4.52} & 63.56\,\gain{+14.99} & 80.22\,\gain{+0.75} & 72.15\,\gain{+5.48} & 78.37\,\gain{+0.14} & 68.21\,\gain{+3.92}\\
\bottomrule
\end{tabularx}
\end{table}

\paragraph{Extension to other tabular foundation models.}
Table~\ref{tab:tfm-extension} evaluates our robust adaptation on
TabPFN-3.5~\citep{priorlabs2026tabpfn35} and three other TFMs:
EXAONE~\citep{eo2026exaonetabular}, TabFM~\citep{google2026tabfm}, and
Causilo~\citep{nums2026causilo}.
We fine-tune the pretrained query-scaling network when present, and otherwise
add and train a new one.
All four models consistently improve the average of worst-group accuracies
on the five tabular benchmarks, with an average gain of 19.25 pp across models.
These results support the applicability of our adaptation beyond TabPFN.
Detailed results and experimental settings are provided in
Appendix~\ref{appendix:tfm-extension-results}.

\section{Conclusion}
We presented DR-TFM, a parameter-efficient framework for distributionally robust
TFM adaptation under subpopulation shift without true group labels.
Its two instantiations construct ambiguity sets from estimated groups or source
conditional distributions.
Query scaling preserves pretrained representations while offering computational
and memory efficiency and scalability.
Experiments demonstrate improved average worst-group accuracy with competitive
mean group accuracy across multiple TFMs.
Future work will extend the framework to a broader range of distribution shifts.

\clearpage
\endgroup

\bibliography{ref}
\bibliographystyle{plainnat}

\clearpage

\appendix
\section{Appendix}
\addtocontents{toc}{\protect\setcounter{tocdepth}{3}}
\begingroup
\setlength{\parskip}{0pt}
\linespread{1.25}\selectfont
\tableofcontents
\endgroup
\clearpage
\subsection{Theoretical analysis of subpopulation shift}
\label{appendix:rationale}
We characterize how changes in subpopulation proportions affect predictive risk,
then examine how group imbalance can produce differences in group risks.

\begin{itemize}[left=1pt]
    \item \textbf{Proposition \ref{prop:risk_variation}} shows when subpopulation shift can have a large effect on predictive risk. The effect depends on how much the group proportions change, how different the group risks are, and the direction of the shift.

    \item \textbf{Proposition \ref{prop:imbalance-spurious-risk}} explains how group imbalance can create differences in risk across subpopulations. Under average-risk minimization, stronger imbalance-induced spurious correlation can lead to a larger group risk gap.
\end{itemize}
\bigskip

\begin{proposition}[Risk variation under subpopulation mixture shift]
\label{prop:risk_variation}
Suppose that the training and test distributions are mixtures of the same
$K$ subpopulation distributions,
\[
P^{\rm tr}
=
\sum_{k=1}^K \pi_k^{\rm tr} P_k,
\qquad
P^{\rm te}
=
\sum_{k=1}^K \pi_k^{\rm te} P_k,
\]
where $\pi^{\rm tr},\pi^{\rm te}\in\Delta_{K-1}$.

For a fixed model parameter $\theta$, define the group risks by
\[
R_k(\theta)
:=
\mathbb{E}_{ P_k}
\bigl[\ell(f_\theta(X),Y)\bigr].
\]
Assume these risks are finite, and define the training and test risks by
$R^{\rm tr}(\theta)=\sum_k\pi_k^{\rm tr}R_k(\theta)$ and
$R^{\rm te}(\theta)=\sum_k\pi_k^{\rm te}R_k(\theta)$.
Let
\[
R_{\max}(\theta)
:=
\max_{k=1,\ldots,K} R_k(\theta),
\qquad
R_{\min}(\theta)
:=
\min_{k=1,\ldots,K} R_k(\theta).
\]

\afterdisplay Then, there exists
$A(\theta;\pi^{\rm tr},\pi^{\rm te})\in[-1,1]$ such that
\begin{align*}
R^{\rm te}(\theta)-R^{\rm tr}(\theta)
&=
d_{\rm TV}(\pi^{\rm te},\pi^{\rm tr})
\bigl(R_{\max}(\theta)-R_{\min}(\theta)\bigr)
A(\theta;\pi^{\rm tr},\pi^{\rm te}),
\end{align*}
where
\[
d_{\rm TV}(\pi^{\rm te},\pi^{\rm tr})
:=
\frac{1}{2}
\sum_{k=1}^K
\left|
\pi_k^{\rm te}-\pi_k^{\rm tr}
\right|
\]
denotes the total variation distance.

Consequently,
\begin{align*}
\left|
R^{\rm te}(\theta)-R^{\rm tr}(\theta)
\right|
\leq
d_{\rm TV}(\pi^{\rm te},\pi^{\rm tr})
\bigl(R_{\max}(\theta)-R_{\min}(\theta)\bigr).
\end{align*}
\end{proposition}

\bigskip
The bound depends on the magnitude of the mixture shift and the range of group
risks. The alignment factor $A$ determines the direction and size of the actual risk change.
\bigskip

\begin{proof}[Proof of Proposition~\ref{prop:risk_variation}]
For a fixed $\theta$,
\begin{align*}
R^{\rm te}(\theta)-R^{\rm tr}(\theta)
=
\sum_{k=1}^K
\left(\pi_k^{\rm te}-\pi_k^{\rm tr}\right)
R_k(\theta).
\end{align*}

\afterdisplay Let
\[
\Delta_k
:=
\pi_k^{\rm te}-\pi_k^{\rm tr}.
\]
Since both $\pi^{\rm tr}$ and $\pi^{\rm te}$ are probability vectors,
\[
\sum_{k=1}^K \Delta_k = 0.
\]

\afterdisplay By the definition of total variation distance,
\begin{align}
s
:=
d_{\rm TV}(\pi^{\rm te},\pi^{\rm tr})
&=
\frac{1}{2}\sum_{k=1}^K |\Delta_k|
\nonumber\\
&=
\sum_{k=1}^K [\Delta_k]_+
=
\sum_{k=1}^K [-\Delta_k]_+,
\label{eq:proof-tv}
\end{align}
where $[a]_+ := \max\{a,0\}$ denotes the positive part of $a$.
If $s=0$, then $\pi^{\rm te}=\pi^{\rm tr}$ and hence
$R^{\rm te}(\theta)=R^{\rm tr}(\theta)$, so the result holds by taking
$A(\theta;\pi^{\rm tr},\pi^{\rm te})=0$.

Suppose now that $s>0$, and define
\[
q_k^+
:=
\frac{[\Delta_k]_+}{s},
\qquad
q_k^-
:=
\frac{[-\Delta_k]_+}{s},
\qquad k=1,\ldots,K.
\]
By \eqref{eq:proof-tv}, both $q^+$ and $q^-$ belong to
$\Delta_{K-1}$. Moreover, by definition,
\[
\Delta_k = s\left(q_k^+ - q_k^-\right).
\]
Therefore,
\begin{align}
R^{\rm te}(\theta)-R^{\rm tr}(\theta)
=
s
\left(
\sum_{k=1}^K q_k^+R_k(\theta)
-
\sum_{k=1}^K q_k^-R_k(\theta)
\right).
\label{eq:proof-direction}
\end{align}
Here, $q^+$ represents the normalized probability mass added to
subpopulations at test time, whereas $q^-$ represents the normalized
probability mass removed from subpopulations.

Since both $q^+$ and $q^-$ are probability vectors,
\[
R_{\min}(\theta)
\leq \sum_{k=1}^K q_k^{\sigma}R_k(\theta)
\leq R_{\max}(\theta),
\qquad \sigma\in\{+,-\}.
\]
Therefore,
\begin{align}
-
\bigl(R_{\max}(\theta)-R_{\min}(\theta)\bigr)
\leq
\sum_{k=1}^K q_k^+R_k(\theta)
-
\sum_{k=1}^K q_k^-R_k(\theta)
\leq
R_{\max}(\theta)-R_{\min}(\theta).
\label{eq:proof-risk-range}
\end{align}

\afterdisplay If $R_{\max}(\theta)=R_{\min}(\theta)$, all group risks are identical,
and \eqref{eq:proof-direction} is zero. Thus, the result holds by taking
$A(\theta;\pi^{\rm tr},\pi^{\rm te})=0$.

Otherwise, define
\begin{align}
A(\theta;\pi^{\rm tr},\pi^{\rm te})
:=
\frac{
\sum_{k=1}^K q_k^+R_k(\theta)
-
\sum_{k=1}^K q_k^-R_k(\theta)
}{
R_{\max}(\theta)-R_{\min}(\theta)
}.
\label{eq:shift-alignment}
\end{align}
By \eqref{eq:proof-risk-range},
\[
A(\theta;\pi^{\rm tr},\pi^{\rm te})\in[-1,1].
\]
Hence, $A$ measures the alignment between the direction of the mixture
shift and the group risks: $A>0$ indicates a shift toward higher-risk
subpopulations, whereas $A<0$ indicates a shift toward lower-risk
subpopulations.

Substituting \eqref{eq:shift-alignment} into
\eqref{eq:proof-direction} gives
\[
R^{\rm te}(\theta)-R^{\rm tr}(\theta)
=
d_{\rm TV}(\pi^{\rm te},\pi^{\rm tr})
\bigl(R_{\max}(\theta)-R_{\min}(\theta)\bigr)
A(\theta;\pi^{\rm tr},\pi^{\rm te}).
\]
Finally, since $|A(\theta;\pi^{\rm tr},\pi^{\rm te})|\leq 1$,
\[
\left|
R^{\rm te}(\theta)-R^{\rm tr}(\theta)
\right|
\leq
d_{\rm TV}(\pi^{\rm te},\pi^{\rm tr})
\bigl(R_{\max}(\theta)-R_{\min}(\theta)\bigr),
\]
which proves the result.
\end{proof}

\bigskip
The following corollary illustrates this shift-risk alignment more explicitly
in the two-subpopulation case.
\bigskip
\begin{corollary}[Two-subpopulation case]
\label{cor:two-subpopulation}

Suppose $K=2$ and $R_2(\theta)\geq R_1(\theta)$. Let
\[
\Delta
:=
\pi_2^{\rm te}-\pi_2^{\rm tr}.
\]
Then,
\[
R^{\rm te}(\theta)-R^{\rm tr}(\theta)
=
\Delta\bigl(R_2(\theta)-R_1(\theta)\bigr).
\]
Moreover,
\[
d_{\rm TV}(\pi^{\rm te},\pi^{\rm tr})
=
|\Delta|,
\]
and hence
\[
\left|
R^{\rm te}(\theta)-R^{\rm tr}(\theta)
\right|
=
d_{\rm TV}(\pi^{\rm te},\pi^{\rm tr})
\bigl(R_{\max}(\theta)-R_{\min}(\theta)\bigr).
\]
For $\Delta\neq0$ and $R_2(\theta)>R_1(\theta)$, the alignment factor in Proposition~\ref{prop:risk_variation}
satisfies
\[
A(\theta;\pi^{\rm tr},\pi^{\rm te})
=
\operatorname{sign}(\Delta).
\]
\end{corollary}

The result follows directly from Proposition~\ref{prop:risk_variation}
by setting $K=2$.

\bigskip
The preceding results concern a fixed predictor. We next examine how fitting
a predictor to an imbalanced distribution can lead to unequal group risks.
\bigskip

\begin{proposition}[Risk heterogeneity under imbalance-induced spurious correlation]
\label{prop:imbalance-spurious-risk}

Let $Y,S\in\{-1,+1\}$ and let $X=(Z,S)$, where $S$ denotes a binary
spurious feature and $Z$ collects the remaining features.
Let $P_1$ and $P_2$ be two subpopulation distributions over $(X,Y)$
such that
\[
S=Y \quad P_1\text{-a.s.},
\qquad
S=-Y \quad P_2\text{-a.s.},
\]
and assume that $Y$ is balanced under both $P_1$ and $P_2$.

For $\rho\in[1/2,1)$, define
\[
P_\rho
=
\rho P_1+(1-\rho)P_2.
\]
Then
\[
\operatorname{Corr}_{P_\rho}(S,Y)=2\rho-1.
\]

\afterdisplay For $g\in\{1,2\}$, define
\[
R_g(f)
=
\mathbb E_{ P_g}
\left[\ell(f(X),Y)\right],
\]
and let
\[
f_\rho
\in
\arg\min_{f\in\mathcal F}
\left\{
\rho R_1(f)+(1-\rho)R_2(f)
\right\}.
\]
Here $\mathcal F$ is a fixed hypothesis class; the group risks are assumed finite
and the minimum is assumed to be attained.

\afterdisplay Then, for any $1/2\leq\rho_1<\rho_2<1$,
\[
R_2(f_{\rho_1})-R_1(f_{\rho_1})
\leq
R_2(f_{\rho_2})-R_1(f_{\rho_2}).
\]

\afterdisplay Moreover, if
\[
R_1(f_{1/2})=R_2(f_{1/2}),
\]
then, defining
\[
R_{\max}(f)=\max_{g\in\{1,2\}}R_g(f),
\qquad
R_{\min}(f)=\min_{g\in\{1,2\}}R_g(f),
\]
the group risk gap
\[
R_{\max}(f_\rho)-R_{\min}(f_\rho)
\]
is non-decreasing in $\rho$.
\end{proposition}

\bigskip

\begin{proof}
We first compute the correlation between $S$ and $Y$, then compare the group
risks at two mixture proportions.
Since $Y$ is balanced under both subpopulations and $S=Y$ under $P_1$ while
$S=-Y$ under $P_2$, both variables are balanced under $P_\rho$.
Thus, their means are zero and their variances are one, giving
\[
\operatorname{Corr}_{P_\rho}(S,Y)
=\mathbb E_{P_\rho}[SY]
=\rho\,\mathbb E_{P_1}[SY]+(1-\rho)\,\mathbb E_{P_2}[SY]
=2\rho-1.
\]

\afterdisplay Define
$L_\rho(f)=\rho R_1(f)+(1-\rho)R_2(f)$, and let
$f_1=f_{\rho_1}$ and $f_2=f_{\rho_2}$ for
$1/2\leq\rho_1<\rho_2<1$.
By optimality,
\[
L_{\rho_1}(f_1)\leq L_{\rho_1}(f_2),
\qquad
L_{\rho_2}(f_2)\leq L_{\rho_2}(f_1).
\]
Combining these inequalities yields
\[
L_{\rho_2}(f_2)-L_{\rho_1}(f_2)
\leq L_{\rho_2}(f_1)-L_{\rho_1}(f_1).
\]
For any $f\in\mathcal F$,
\[
L_{\rho_2}(f)-L_{\rho_1}(f)
=(\rho_2-\rho_1)\bigl(R_1(f)-R_2(f)\bigr).
\]
Since $\rho_2-\rho_1>0$, we obtain
\[
R_2(f_{\rho_1})-R_1(f_{\rho_1})
\leq R_2(f_{\rho_2})-R_1(f_{\rho_2}).
\]

\afterdisplay Finally, if $R_1(f_{1/2})=R_2(f_{1/2})$, this monotonicity implies
$R_2(f_\rho)-R_1(f_\rho)\geq0$ for every $\rho\in[1/2,1)$.
Consequently,
\[
R_{\max}(f_\rho)-R_{\min}(f_\rho)
=R_2(f_\rho)-R_1(f_\rho),
\]
which is non-decreasing in $\rho$ by the preceding inequality.
\end{proof}

\subsection{Method details}
\label{appendix:method}
We first describe the inputs, group estimation, and data preparation used for
classification. We then detail the optimization procedures for DR-TFM (GR)
and DR-TFM (MS), followed by the query scaling architecture and an extension
to regression.

\subsubsection{Inputs and data preparation}
\label{appendix:method-preparation}
The data preparation below applies to the real-data variants without true
group labels. Appendix~\ref{appendix:ms-sources} specifies the variants with
true groups, and Appendix~\ref{appendix:simulation} describes the simulations.

\paragraph{Inputs.}
Both methods take a training dataset
$\mathcal D=\{(x_i,y_i)\}_{i=1}^{n}$ and a pretrained TabPFN
$f_{\theta_0}$. Training response labels $y_i$ are available, but true group
labels are not required. A separate validation set is used for selecting
hyperparameters and the adapted model; its samples are not included in the
context or the adaptation objective.

\paragraph{Embeddings for group estimation.}
Using the full training dataset $\mathcal D$ as context, we extract a frozen
pretrained TabPFN embedding for each training input, averaging across ensemble members.
We standardize each coordinate using its training mean and standard deviation
to obtain the features $\widetilde z_i$ used for clustering.
The same standardization is applied to validation embeddings.

\paragraph{Class-wise group estimation.}
For each response class $c\in\mathcal Y$, we fit a separate clustering rule
$C_c$ with $G_c$ clusters to
$\{\widetilde z_i:y_i=c\}$.
Each training sample is assigned to the pair
$(y_i,C_{y_i}(\widetilde z_i))$.
We reindex these pairs as $\hat g_i\in[G]$, where
$G=\sum_{c\in\mathcal Y}G_c$.
The main tables use DR-TFM (GR; $k$-means), which assigns each training
sample to the nearest fitted centroid within its class.
For DR-TFM (GR), the number of clusters per class is selected following GEORGE
\citep{sohoni2020no} (Appendix~\ref{appendix:hyperparams}).
The assignments remain fixed throughout adaptation.

The additional GMM variant fits a diagonal-covariance GMM within each class
and assigns each training sample to the component with the largest posterior
probability. It uses the same criterion to select the number of clusters.

\paragraph{Context and loss-evaluation sets.}
We split $\mathcal D$ into a context $\mathcal D^{(c)}$ and a loss-evaluation
set $\mathcal D^{(e)}$ in a $70/30$ ratio, stratifying by $\hat g_i$.
The two sets are disjoint and both retain response labels.
After this split, we recompute embeddings for adaptation using
$\mathcal D^{(c)}$ as context, writing $z(x)$ for the frozen pretrained
embedding of input $x$.
The context samples and their labels remain fixed during adaptation and prediction.
DR-TFM (GR) uses the response labels in $\mathcal D^{(e)}$ to compute group
losses. DR-TFM (MS) uses the inputs in $\mathcal D^{(e)}$ together with soft
labels obtained from source-conditioned pretrained predictions, as described below.

\paragraph{Adapted parameters.}
In both classification methods, $\theta_L$ collects the pretrained query scaling
network parameters selected for fine-tuning. The feature map and all remaining
parameters constitute $\theta_F$ and are kept fixed; the architecture is specified in
Appendix~\ref{appendix:scaling-implementation}.
The algorithms below describe $T$ optimization iterations.
We write $\theta^{(t)}=(\theta_F,\theta_L^{(t)})$ for the full model parameters
after $t$ updates of $\theta_L$.
We select the adapted model with the highest worst-group validation accuracy
over estimated groups. Training schedules and model selection details are
provided in Appendices~\ref{appendix:protocol} and~\ref{appendix:hyperparams}.
For either method, prediction for a new input $x$ uses
$f_\theta(x;\mathcal D^{(c)})$ with the selected parameters and the retained
context. It requires neither a group assignment for $x$ nor adversarial updates.

\subsubsection{Optimization for DR-TFM (GR)}
\label{appendix:alg}
\paragraph{Group losses.}
Using the fixed assignments from Appendix~\ref{appendix:method-preparation},
define the loss-evaluation indices for each estimated group $k$ as
$\widehat{\mathcal I}_k=\{i:(x_i,y_i)\in\mathcal D^{(e)},\hat g_i=k\}$.
Let $n_k=|\widehat{\mathcal I}_k|>0$ denote its size.
Its empirical joint distribution is
$\widehat P_k^{(e)}=n_k^{-1}\sum_{i\in\widehat{\mathcal I}_k}\delta_{(x_i,y_i)}$.
Substituting the ambiguity set in~\eqref{eq:gr_ambiguity} into
\eqref{eq:general_dr_tabpfn} yields
\begin{equation}
\min_{\theta_L}
\max_{q\in\Delta_{G-1}}
\sum_{k=1}^{G} q_k
\left[
\frac{1}{|\widehat{\mathcal I}_k|}
\sum_{i\in\widehat{\mathcal I}_k}
\ell\!\left(f_\theta(x_i;\mathcal D^{(c)}),y_i\right)
\right].
\label{eq:weighted_dr_tabpfn}
\end{equation}
Here, $q_k$ is the weight assigned to estimated group $k$, and the bracketed
quantity is its mean loss, denoted by $L_k(\theta)$.
For classification, $\ell$ is cross-entropy with the observed response label.
Maximizing over $q\in\Delta_{G-1}$ selects the largest empirical group loss.

\paragraph{Alternating updates.}
We initialize $q$ uniformly and $\theta_L$ from the pretrained query scaling network.
Each iteration first computes the mean loss of each estimated group, then
increases the weights of groups with larger losses by exponentiated gradient
ascent. Holding these updated weights fixed, we update $\theta_L$ to reduce the
weighted group losses.
Algorithm~\ref{alg:dr_tabpfn} summarizes the basic updates for
\eqref{eq:weighted_dr_tabpfn}.
In our experiments, we use the group-adjusted DRO variant of
\citet{sagawa2019distributionally}, with group-size adjustments
$C/\sqrt{n_k}$, where we set $C=5$.
Further optimization settings are provided in Appendix~\ref{appendix:hyperparams}.

\begin{algorithm}[t]
\caption{DR-TFM (GR) for~\eqref{eq:weighted_dr_tabpfn}}
\label{alg:dr_tabpfn}
\begin{algorithmic}[1]
\State \textbf{Input:} training dataset $\mathcal D$; validation set; pretrained TabPFN $f_{\theta_0}$;
    group estimation rule; learning rates
    $\eta_{\theta_L},\eta_q$; update steps $T$
\State Estimate groups $\hat g_i\in[G]$ from frozen pretrained embeddings
\State Prepare the fixed context $\mathcal D^{(c)}$ and loss-evaluation set
    $\mathcal D^{(e)}$ as in Appendix~\ref{appendix:method-preparation}
\State Obtain frozen embeddings using $\mathcal D^{(c)}$ as context
\State Define $\widehat{\mathcal I}_k
    \gets
    \{i:(x_i,y_i)\in\mathcal D^{(e)},\hat g_i=k\}$ and
    $n_k\gets|\widehat{\mathcal I}_k|$ for $k\in[G]$
\State Initialize $q^{(0)}\gets (1/G,\ldots,1/G)$ and $\theta_L^{(0)}$ from the
    pretrained query scaling network; freeze all other parameters

\For{$t=0,\ldots,T-1$}
    \State Compute the group-wise losses
    \[
    L_k(\theta^{(t)})
    \gets
    \frac{1}{n_k}
    \sum_{i\in\widehat{\mathcal I}_k}
    \ell\!\left(
    f_{\theta^{(t)}}(x_i;\mathcal D^{(c)}),
    y_i
    \right)
    \]
    \State Update group weights by exponentiated gradient ascent:
    \[
    q_k^{(t+1)}
    \propto
    q_k^{(t)}
    \exp\!\left(\eta_q L_k(\theta^{(t)})\right),
    \qquad k\in[G]
    \]
    \State Normalize $q^{(t+1)}$ such that $\sum_{k=1}^{G}q_k^{(t+1)}=1$
    \State Update the fine-tuning parameters:
    \[
    \theta_L^{(t+1)}
    \gets
    \theta_L^{(t)}
    -
    \eta_{\theta_L}
    \nabla_{\theta_L}
    \sum_{k=1}^{G}
    q_k^{(t+1)}L_k(\theta)
    \bigg|_{\theta_L=\theta_L^{(t)}}
    \]
\EndFor
\State Select an iterate using validation, as in Appendix~\ref{appendix:method-preparation}
\State \textbf{Output:} selected adapted TabPFN and fixed context $\mathcal D^{(c)}$
\end{algorithmic}
\end{algorithm}

\subsubsection{Optimization for DR-TFM (MS)}
\label{appendix:alg-ms}
\paragraph{Source construction.}
We first prepare $\mathcal D^{(c)}$ and $\mathcal D^{(e)}$ as in
Appendix~\ref{appendix:method-preparation}.
We then form labeled source subsets $\mathcal S_1,\ldots,\mathcal S_M$
from $\mathcal D^{(c)}$.
Here, $M$ denotes the number of retained sources.
For MS-C, we apply $k$-means to the fixed context embeddings jointly across
response classes. Each retained cluster forms a source; clusters containing
only one response class are excluded.
For MS-R, we construct each source independently by sampling without replacement
from the full context. All context samples are available again when constructing
the next source, so different sources may overlap.
The initial source count and subset sizes are specified in
Appendix~\ref{appendix:ms-sources}.

\paragraph{Source conditional predictions.}
For each source $\mathcal S_k$, we condition the frozen pretrained TabPFN on
its inputs and response labels and evaluate every input $x_i$ in
$\mathcal D^{(e)}$:
$\widehat P^{(k)}_{Y\mid X}(\cdot\mid x_i)=f_{\theta_0}(x_i;\mathcal S_k)$.
Thus, each source provides a class probability vector for every loss-evaluation
input. We compute these probabilities at the original inputs and keep them
and the source subsets fixed throughout adaptation.
Given mixture weights $\beta$, the soft label is
$y^\circ(\beta,x_i)=\sum_{k=1}^{M}\beta_k
\widehat P^{(k)}_{Y\mid X}(\cdot\mid x_i)$.
The mixture weights represent uncertainty over the relative contributions
of these source-conditioned predictions.
For a predicted probability vector $p$ and a soft label $v$, we use
$\ell(p,v)=-\sum_{y\in\mathcal Y}v_y\log p_y$.
The adapted model always uses the full context $\mathcal D^{(c)}$.

\paragraph{Surrogate objective.}
The Wasserstein distance in~\eqref{eq:ms_ambiguity} uses transport cost
$c(x,x')=\|z(x)-z(x')\|_2$, using the adaptation embeddings defined in
Appendix~\ref{appendix:method-preparation}.
Let $f_{\theta,Z}(z')$ denote the decoder prediction at the input embedding $z'$,
with $\mathcal D^{(c)}$ fixed as context, so that
$f_{\theta,Z}(z(x))=f_\theta(x;\mathcal D^{(c)})$.
We adopt the surrogate objective proposed by \citet{kim2026distributionally}
and optimize it over the query scaling parameters:
\begin{equation*}
\min_{\theta_L}
\sup_{\substack{\beta\in\Delta_{M-1}\\\|\beta-\bar\beta\|_2\leq\epsilon_2}}
\mathbb E_{X\sim\widehat P_X^{(e)}}
\left[
\sup_{\|z'-z(X)\|_2\leq\epsilon_1}
\ell\!\left(f_{\theta,Z}(z'),y^\circ(\beta,X)\right)
\right].
\end{equation*}
Here, the expectation is over the empirical input distribution of
$\mathcal D^{(e)}$. The mixture weights $\beta$ are shared across samples,
while each sample has its own feature perturbation.

\paragraph{Alternating updates.}
We initialize $\theta_L$ from the pretrained query scaling network and
$\beta$ at the uniform nominal vector $\bar\beta$.
For each minibatch, we initialize each perturbed embedding at $z(x_i)$.
Holding $\theta_L$ and $\beta$ fixed, we take a normalized gradient ascent step
on each embedding and project it onto its Euclidean ball of radius $\epsilon_1$.
We then increase the shared weights of sources with larger average losses
and enforce the simplex and radius constraints on $\beta$.
Finally, holding the perturbed embeddings and the updated $\beta$ fixed,
we update $\theta_L$ using their soft-label loss.
The updates of $\theta_L$ and $\beta$ continue across minibatches.

Algorithm~\ref{alg:dr_tabpfn_ms} approximates the inner maximization with one
feature step and one mixture step per model update, as used in our experiments.
The operator $\Pi_{\mathbb B_i}$ projects onto
$\mathbb B_i=\{u:\|u-z(x_i)\|_2\leq\epsilon_1\}$.
For the mixture update, $\mathcal P_\beta$ first projects onto
$\Delta_{M-1}$ and then contracts the result toward $\bar\beta$ if its distance
exceeds $\epsilon_2$. This preserves the simplex constraint while satisfying
the radius constraint.
The normalized gradient $\widehat\nabla$ has unit Euclidean norm for each
sample; a zero gradient leaves its embedding unchanged.

\begin{algorithm}[t]
\caption{DR-TFM (MS)}
\label{alg:dr_tabpfn_ms}
\begin{algorithmic}[1]
\State \textbf{Input:} training dataset $\mathcal D$; validation set; pretrained TabPFN $f_{\theta_0}$;
    source construction rule;
    radii $\epsilon_1,\epsilon_2$; learning rates
    $\eta_{\theta_L},\eta_\beta,\eta_z$; update steps $T$
\State Prepare the fixed context $\mathcal D^{(c)}$ and loss-evaluation set
    $\mathcal D^{(e)}$ as in Appendix~\ref{appendix:method-preparation}
\State Obtain frozen embeddings using $\mathcal D^{(c)}$ as context
\State Construct source subsets $\mathcal S_1,\ldots,\mathcal S_M\subseteq\mathcal D^{(c)}$
\State Compute fixed source predictions
    $\widehat P^{(k)}_{Y\mid X}(\cdot\mid x_i)\gets f_{\theta_0}(x_i;\mathcal S_k)$
    for every $(x_i,y_i)\in\mathcal D^{(e)}$ and $k\in[M]$
\State Initialize $\beta^{(0)}\gets\bar\beta=(1/M,\ldots,1/M)$ and $\theta_L^{(0)}$ from
    the pretrained query scaling network; freeze all other parameters

\For{$t=0,\ldots,T-1$}
    \State Draw minibatch indices $B$ from $\mathcal D^{(e)}$ and set
    $z'_i\gets z(x_i)$ for each $i\in B$
    \State Form the soft labels
    $y^{\circ}(\beta^{(t)},x_i)\gets\sum_{k=1}^{M}\beta_k^{(t)}
    \widehat P^{(k)}_{Y\mid X}(\cdot\mid x_i)$
    \State Ascend on each embedding and enforce its radius constraint, for $i\in B$:
    \[
    z'_i\gets\Pi_{\mathbb B_i}
    \Bigl(z'_i+\eta_z\widehat\nabla_{z'_i}\,
    \ell\bigl(f_{\theta^{(t)},Z}(z'_i),y^{\circ}(\beta^{(t)},x_i)\bigr)\Bigr)
    \]
    \State Ascend on the shared mixture weights and enforce their constraints:
    \[
    \begin{aligned}
    \beta^{(t+1)}&\gets
    \mathcal P_\beta
    \bigl(\beta^{(t)}+\eta_\beta g\bigr),\\
    g_k&=\frac{1}{|B|}\sum_{i\in B}
    \ell\bigl(f_{\theta^{(t)},Z}(z'_i),
    \widehat P^{(k)}_{Y\mid X}(\cdot\mid x_i)\bigr),\quad k\in[M].
    \end{aligned}
    \]
    \State Update $\theta_L$, holding $z'_i$ and $\beta^{(t+1)}$ fixed:
    \[
    \theta_L^{(t+1)}\gets\theta_L^{(t)}-\eta_{\theta_L}\nabla_{\theta_L}
    \frac{1}{|B|}\sum_{i\in B}
    \ell\bigl(f_{\theta,Z}(z'_i),y^{\circ}(\beta^{(t+1)},x_i)\bigr)
    \bigg|_{\theta_L=\theta_L^{(t)}}
    \]
\EndFor
\State Select an iterate using validation, as in Appendix~\ref{appendix:method-preparation}
\State \textbf{Output:} selected adapted TabPFN and fixed context $\mathcal D^{(c)}$
\end{algorithmic}
\end{algorithm}

\subsubsection{Query scaling architecture}
\label{appendix:scaling-implementation}
We describe the architecture of the scaling function in \eqref{eq:query-scaling}.
Let $g_{\theta_L}$ denote the small neural network used to construct $s_{\theta_L}$.
It takes a projected query
$u\in\mathbb R^{r}$ as input and returns a vector of the same dimension.
The network consists of two linear layers, each including a bias, with a GELU activation
between them:
\begin{equation*}
\mathbb R^{r}
\xrightarrow{\text{Linear}}\mathbb R^{64}
\xrightarrow{\text{GELU}}\mathbb R^{64}
\xrightarrow{\text{Linear}}\mathbb R^{r}.
\end{equation*}
Equivalently,
\begin{equation*}
g_{\theta_L}(u)=W_2\operatorname{GELU}(W_1u+a_1)+a_2,
\end{equation*}
where $W_1\in\mathbb R^{64\times r}$, $a_1\in\mathbb R^{64}$,
$W_2\in\mathbb R^{r\times64}$, and $a_2\in\mathbb R^{r}$.
In the TabPFN classification model used in our main experiments, $r=64$ and the decoder has
six attention heads. The same network is applied separately to each head's query vector.

The network output is converted into a coordinate-wise modulation through
$\mathbf 1+\tanh(g_{\theta_L}(u))$. This modulation multiplies the decoder's existing
base scaling. As in the main text, we omit the head index and write
$b(n_c)\in\mathbb R^{r}$ for the base scaling of one head, computed from
$\log n_c$ by a separate frozen pretrained network.
The combined scaling function is
\begin{equation*}
s_{\theta_L}(u)
=\frac{1}{\sqrt{r}}\,b(n_c)\odot\bigl[\mathbf 1+\tanh(g_{\theta_L}(u))\bigr],
\end{equation*}
where $\mathbf 1$ is the vector of ones and $\tanh$ is applied element-wise.
The parameters of the base scaling network are kept fixed. Thus, for a fixed context size $n_c$,
fine-tuning changes only the query-dependent modulation, while $b(n_c)$ remains
unchanged. We suppress the dependence of $s_{\theta_L}$ on $n_c$ in the main text.
We include the standard attention normalization factor $1/\sqrt{r}$ in
$s_{\theta_L}$ for notational convenience.

In both simulated and real-data classification experiments with TabPFN,
we update both linear layers of the query scaling network,
with $\theta_L=(W_1,a_1,W_2,a_2)$.
This gives $8{,}320$ trainable parameters, approximately $0.016\%$ of the
pretrained model's parameters.
The pretrained TabPFN feature map, the key and query projections, the base scaling network, and all other
decoder parameters are kept fixed.

\subsubsection{Illustration of query scaling}
\label{appendix:query-scaling}
Figure~\ref{fig:query-scaling} illustrates the effect of a uniform query scale.
For fixed query--key scores, a smaller positive uniform scale produces a more diffuse
attention distribution, whereas a larger scale concentrates attention on the context
samples with the highest scores.

The learned scaling function in DR-TFM is coordinate-wise,
so it can also change the relative query--key scores.
Figure~\ref{fig:simulation} examines the learned attention distributions empirically.

\begin{figure}[H]
    \centering
    \includegraphics[width=\textwidth, trim=0bp 0bp 0bp 610bp, clip]{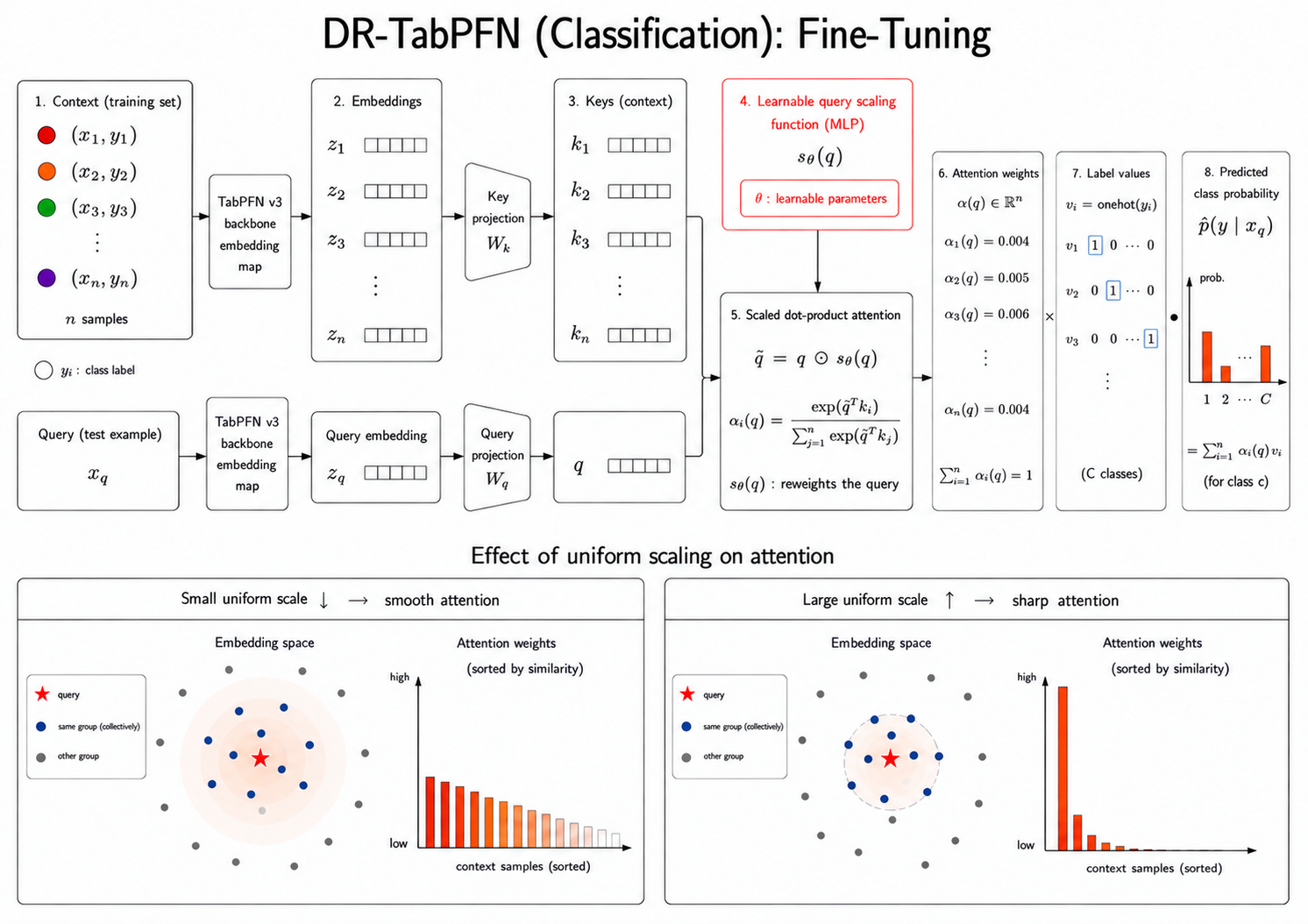}
    \caption{Schematic attention distributions under smaller and larger positive
    uniform query scales.}
    \label{fig:query-scaling}
\end{figure}

\subsubsection{Regression}
\label{appendix:regression}
DR-TFM also applies to regression. The pretrained TabPFN
regression model used here does not include a query-scaling network.
We therefore fine-tune the first linear layer of its output network, keeping
the pretrained TabPFN feature map and the second linear layer fixed.
The objective uses squared-error loss in place of the classification loss.

Given a query embedding $z\in\mathbb R^{512}$, the regression output network produces
logits for $B=5{,}000$ response buckets. It consists of two linear layers, each
including a bias, with a GELU activation between them:
\begin{equation*}
\mathbb R^{512}
\xrightarrow{\text{Linear (updated)}}\mathbb R^{1024}
\xrightarrow{\text{GELU}}\mathbb R^{1024}
\xrightarrow{\text{Linear (fixed)}}\mathbb R^{B}.
\end{equation*}
Equivalently,
\begin{equation*}
t_\theta=W_2\operatorname{GELU}(W_1z+a_1)+a_2,
\qquad p_\theta=\operatorname{softmax}(t_\theta)\in\Delta_{B-1}.
\end{equation*}
The trainable parameters are $\theta_L=(W_1,a_1)$, with $W_1\in\mathbb R^{1024\times512}$ and
$a_1\in\mathbb R^{1024}$, while $W_2\in\mathbb R^{B\times1024}$ and
$a_2\in\mathbb R^{B}$ are fixed.
This updates $525{,}312$ parameters, approximately $0.90\%$ of the
pretrained regression model's $58.3$ million parameters.
The prediction is $\hat y_\theta=\sum_{b=1}^{B}p_{\theta,b}\mu_b$,
where $\mu_b$ is the conditional mean of response bucket $b$.
Figure~\ref{fig:regression-architecture} summarizes the architecture.

\begin{figure}[H]
    \centering
    \includegraphics[width=\textwidth]{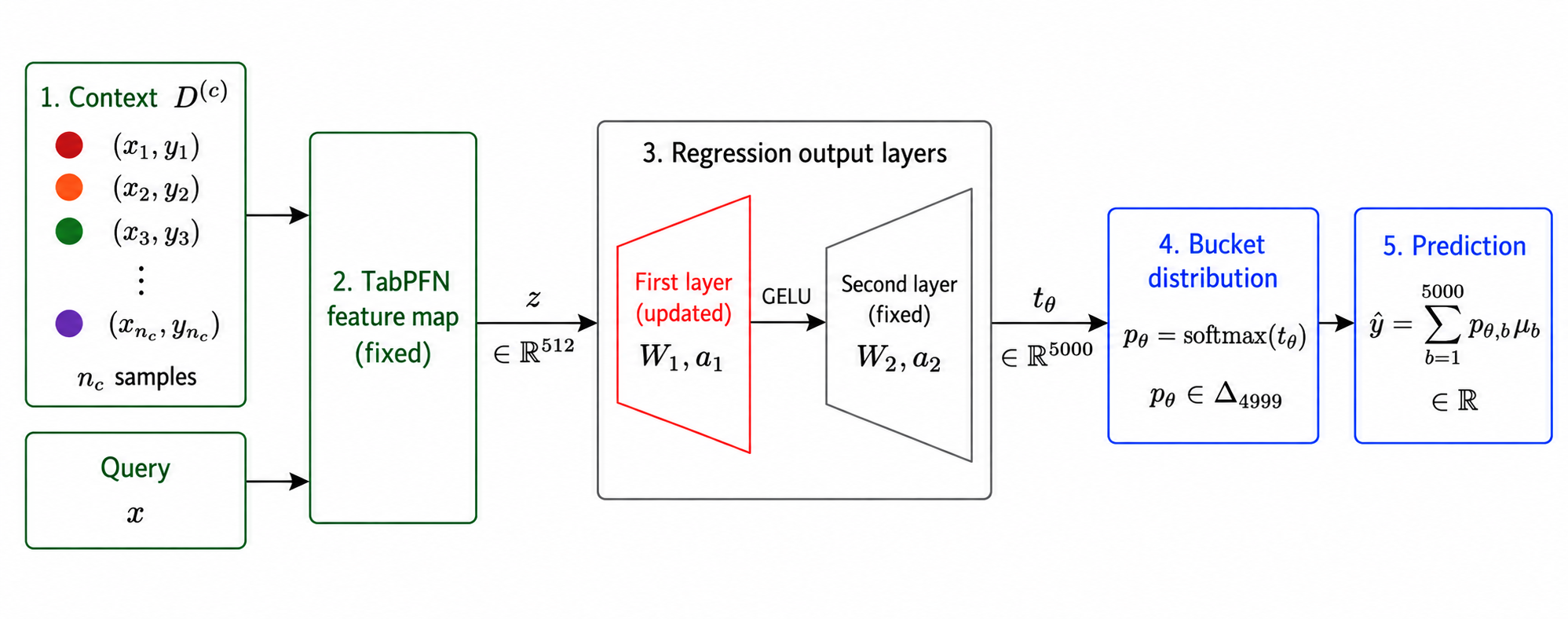}
    \caption{Regression architecture. Only the first linear layer of the output network
    is updated; the pretrained TabPFN feature map and the second linear layer remain fixed.
    Here $\mu_b$ denotes the conditional mean of response bucket $b$.}
    \label{fig:regression-architecture}
\end{figure}

\paragraph{Regression simulation.}
We draw $A\sim\operatorname{Bernoulli}(1/2)$ and generate
\[
\begin{gathered}
Y\mid A\sim\mathcal N\!\left(m_p(2A-1),0.6^2\right),
\quad
X_1\mid A\sim\mathcal N\!\left(0.8(2A-1),0.5^2\right),
\\
X_2=Y+\varepsilon,\qquad \varepsilon\sim\mathcal N(0,0.8^2),
\end{gathered}
\]
where $Y$ and $X_1$ are conditionally independent given $A$, and $\varepsilon$
is independent of $(A,Y,X_1)$. We set $m_p=0.6\Phi^{-1}(p)$, where $\Phi$ is the
standard normal cumulative distribution function. The four evaluation groups
are defined by $(A,\mathbf 1\{Y\geq0\})$, with probabilities
$p/2$, $(1-p)/2$, $(1-p)/2$, and $p/2$.
We vary the training and validation value of $p$ over $\{0.5,0.6,0.7,0.8,0.9\}$
and fix $p=0.5$ at test time. In this regression design, changing $p$ changes
both the group proportions and the distributions within groups.

Each run uses $2{,}000$ training, $1{,}000$ validation, and $2{,}000$ test samples.
We estimate four groups by fitting a two-component GMM to the input features
separately within each response-sign stratum, without using $A$.
Figure~\ref{fig:sim-regression} reports mean squared error (MSE) and mean
absolute error (MAE), averaged over $50$ runs. Mean denotes the average of the
four group errors, and Worst group denotes their maximum.
Consistent with the classification results, increasing training imbalance leads
to a larger deterioration in TabPFN's worst-group performance than in its mean
performance. DR-TFM (GR; GMM) mitigates this deterioration in both MSE and MAE.

\begin{figure}[H]
    \centering
    \includegraphics[width=0.9\textwidth]{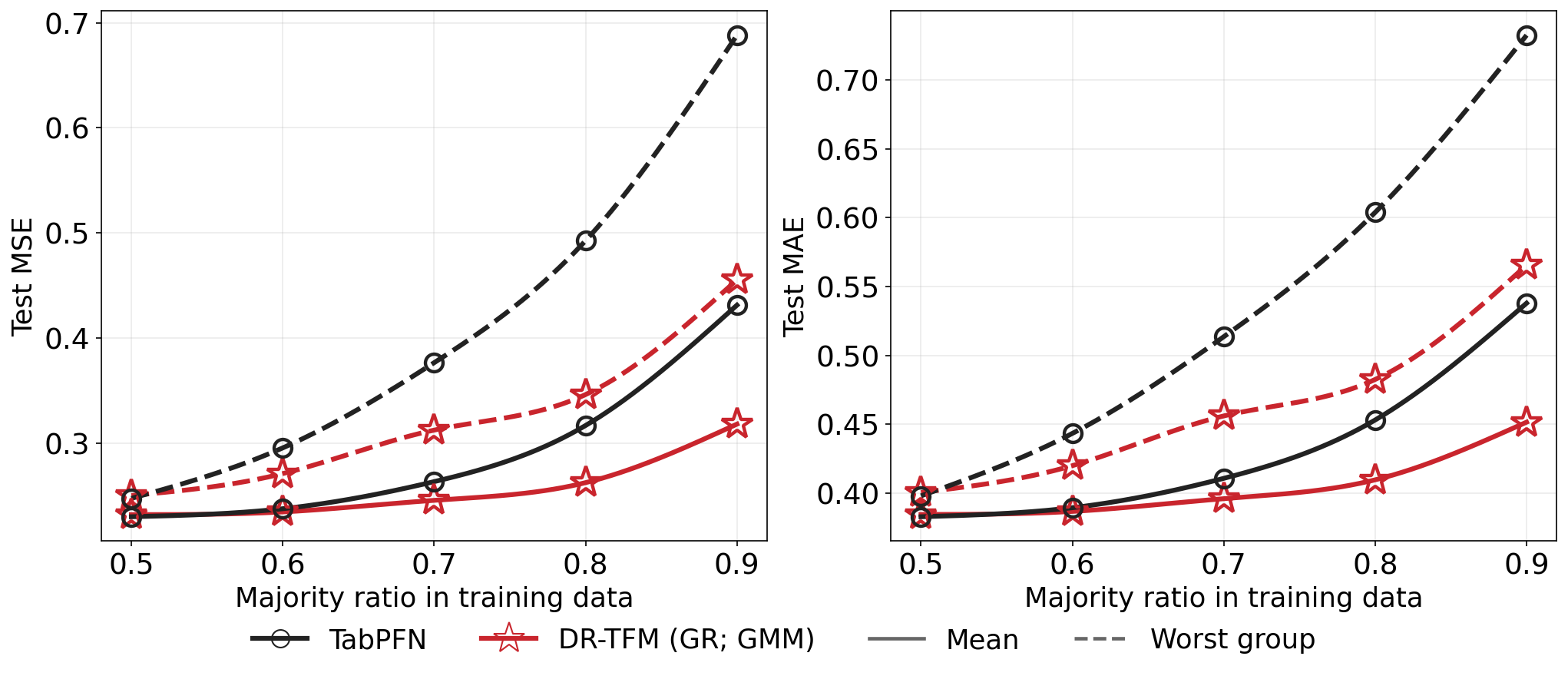}
    \caption{Regression results as the proportion of majority-group samples in training increases.
    Panels show test MSE (left) and MAE (right). Solid and dashed curves denote
    the mean and maximum group errors, respectively.}
    \label{fig:sim-regression}
\end{figure}

\subsection{Experimental setup}
\label{appendix:setup}
We describe the simulation and real-data settings, evaluation metrics, baseline
methods, hyperparameter selection, and an ablation on the number of clusters.
The real-data benchmark and evaluation protocol follow \citet{tong2025latent}.

\subsubsection{Simulation details}
\label{appendix:simulation}
\paragraph{Data generation and groups.}
We generate each sample by first drawing a binary attribute
$A\sim\operatorname{Bernoulli}(1/2)$, then setting $Y=A$ with probability
$p$ and $Y=1-A$ otherwise. Conditional on $(A,Y)$, we independently draw
the two features as
\[
X_1\mid A\sim\mathcal N\!\left(\delta_a(2A-1),\sigma_a^2\right),
\qquad
X_2\mid Y\sim\mathcal N\!\left(\delta_y(2Y-1),\sigma_y^2\right).
\]
For the attention analysis in Figure~\ref{fig:simulation}, we use
$\delta_a=1.0$, $\sigma_a=0.5$, $\delta_y=0.5$, and $\sigma_y=0.6$.
The four groups are $G_1=(A=0,Y=0)$, $G_2=(A=1,Y=0)$,
$G_3=(A=0,Y=1)$, and $G_4=(A=1,Y=1)$, with probabilities
$p/2$, $(1-p)/2$, $(1-p)/2$, and $p/2$, respectively.
The class probabilities remain balanced for every $p$; the imbalance
is across groups rather than labels.

\begin{figure}[!htbp]
    \centering
    \includegraphics[width=\textwidth]{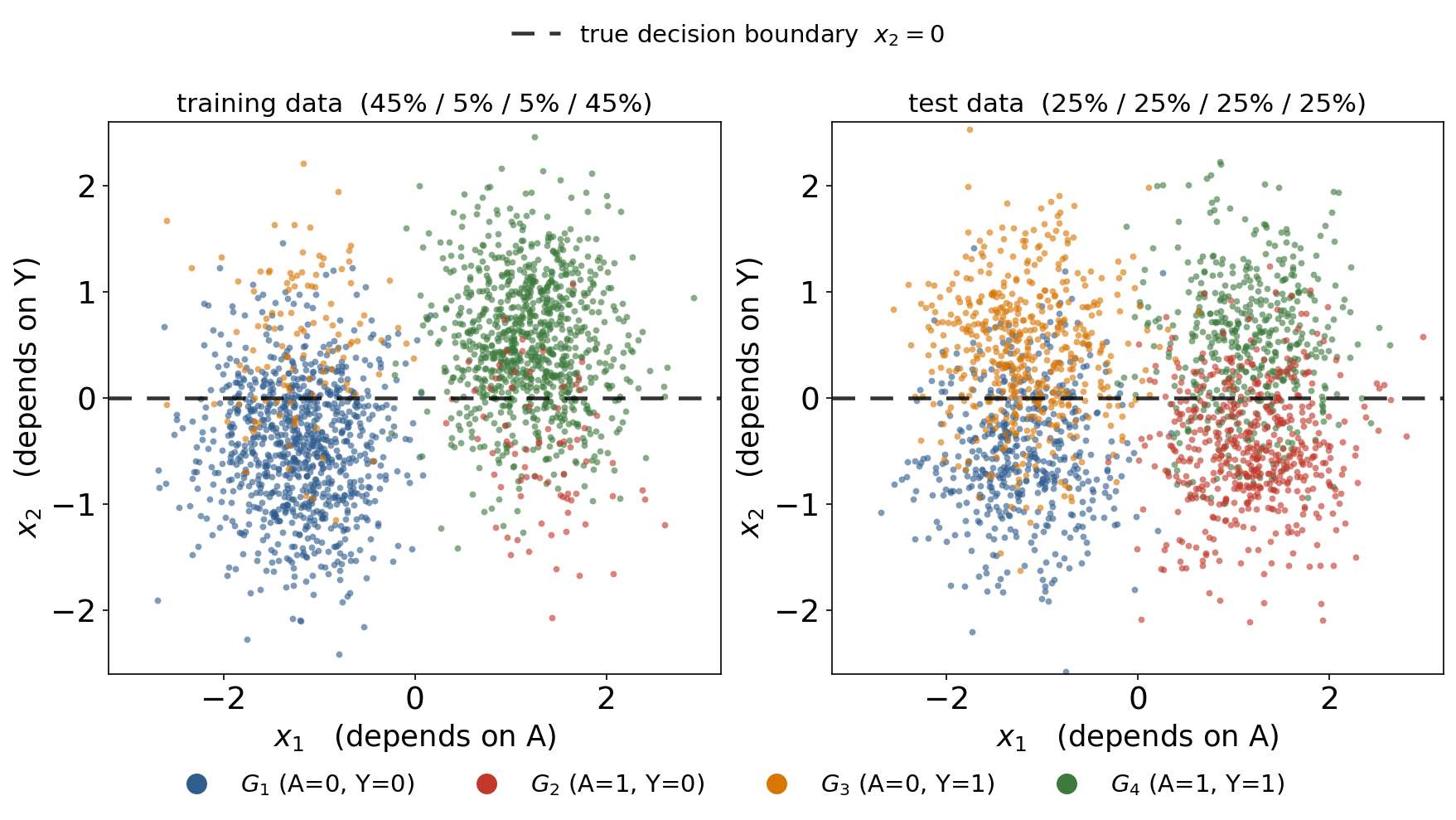}
    \caption{Illustration of the four-group simulation design. Colors identify
    the four $(A,Y)$ groups. Their probabilities are $45\%$, $5\%$, $5\%$,
    and $45\%$ in training (left), and $25\%$ each at test time (right).
    The dashed line shows the Bayes-optimal boundary $x_2=0$ for the balanced
    test distribution, displayed in both panels for reference.}
    \label{fig:simulation-design}
\end{figure}

\paragraph{Sample sizes and subpopulation shift.}
We independently generate $2{,}000$ training samples and $1{,}000$ validation
samples with $p=0.9$, and $2{,}000$ test samples with $p=0.5$.
Training therefore contains an expected $900$ samples from each majority
group ($G_1,G_4$) and $100$ from each minority group ($G_2,G_3$).
The corresponding expected counts are $450$ and $50$ in validation,
and $500$ per group at test time. Realized counts vary across random draws.
Only $p$ changes between training and test; the feature distribution
within each group remains fixed.
Consequently, $X_1$ is associated with $Y$ through $A$ in training
but is independent of $Y$ under the balanced test distribution.
Figure~\ref{fig:simulation-design} illustrates this construction.

For Figure~\ref{fig:intro_failure}, we use the same data-generating parameters
and training, validation, and test set sizes.
We vary the training and validation value of $p$ over
$\{0.5,0.6,0.7,0.8,0.9\}$, keeping the test value fixed at $0.5$.
The reported curves average results over $50$ runs.

\paragraph{Group estimation and adaptation.}
We fit a two-component GMM to the input features within each class,
fixing the number of estimated groups at $G=4$.
We update both linear layers of the query scaling network,
totaling $8{,}320$ trainable parameters
(Appendix~\ref{appendix:scaling-implementation}).
We use AdamW with learning rate $10^{-3}$, weight decay $10^{-4}$, and at most
$500$ updates. We evaluate worst-group validation accuracy over the estimated
groups every $10$ updates and use it for model selection and early stopping,
with patience $10$. Test samples are used only for final evaluation.

For the attention analysis in Figure~\ref{fig:simulation}, the training set
serves as the fixed labeled context $\mathcal D^{(c)}$. We independently draw
an additional $2{,}000$ labeled samples from the same training distribution
to form $\mathcal D^{(e)}$. The context and loss-evaluation sets are disjoint.
The GMMs are fitted only on the context and assign estimated groups to
loss-evaluation and validation samples within each response class.

\paragraph{Labels and the decision boundary.}
The labels are generated by the stochastic rule above, not by deterministically
thresholding the observed features. Under the balanced test distribution,
$X_1$ carries no label information, and the two classes have equal prior
probabilities and equal variances along $X_2$.
The Bayes-optimal classifier for test classification error is therefore
\[
f^{\star}_{\rm te}(x)=\mathbf 1\{x_2>0\},
\]
with decision boundary $x_2=0$.
The Gaussian noise produces overlap between classes, so a sample's label
need not agree with the side of this boundary on which it falls.
This is the test-optimal boundary; the optimal training classifier can also
use $X_1$ because of the training association between $A$ and $Y$.

\paragraph{Attention computation.}
For Figure~\ref{fig:simulation}, a test sample's own group is determined by its
true $(A,Y)$ group defined above.
For each attention head and ensemble member, we sum the attention weights assigned
to context samples in that same group.
We then average these sums across heads and ensemble members to obtain one value
per test sample. The left histogram pools these values for test samples in
$G_1$ and $G_4$, while the right histogram pools them for test samples in
$G_2$ and $G_3$.

\subsubsection{Datasets}
\label{appendix:datasets}
\paragraph{Benchmarks.}
\begin{itemize}[left=0pt]

    \item
    Adult \citep{becker1996adult}.
    Census records. The task is to predict whether a person's annual income exceeds \$50K.

    \item
    Bank \citep{moro2012bank}.
    Marketing calls made by a Portuguese bank. The task is to predict whether the client
    subscribes to a term deposit.

    \item
    Default \citep{yeh2016default}.
    Credit card clients in Taiwan. The task is to predict whether the client defaults in the
    following month.

    \item
    Shoppers \citep{sakar2018online}.
    Online browsing sessions. The task is to predict whether the session ends in a purchase.
    About $85\%$ of sessions do not, so the label itself is heavily imbalanced.

    \item
    Taxi \citep{navas2018taxi}.
    Taxi rides in Mexico City. The task is to predict whether the ride lasts more than thirty
    minutes.

    \item
    ACS Income \citep{ding2021retiring}.
    US Census records, with one table per state. The task is to predict whether income
    exceeds \$50K.

\end{itemize}
All six tasks are binary classification.

\paragraph{Transfer across states.}
ACS Income additionally evaluates transfer across states. We use Arizona (AZ),
Massachusetts (MA), and Michigan (MI) in six settings. The three within-state settings
train and test within one state, following the benchmark splits. The three
transfer settings, AZ$\rightarrow$MA, MA$\rightarrow$MI and MI$\rightarrow$AZ, train on one
state and test on another. In a transfer setting the entire source table is split $80/20$
into training and validation, and the entire target table is used as the test set, so the
target state contributes no training or validation data. Sizes for all six settings are given
in Table~\ref{tab:datasets-acs}.

\paragraph{Groups.}
Evaluation groups are defined by pairing a binary attribute with the label.
During training and model selection, true group labels are used only by the
methods explicitly marked as using them.
The evaluated attributes are the ones selected in Appendix B.2 of
\citet{tong2025latent}: marital status, race, and sex for Adult;
age, housing status, marital status, and last contact duration for Bank;
age, sex, and the amount of the given credit for Default;
traffic type, visitor type, and a weekend indicator for Shoppers;
pickup month, a weekday indicator, and direction for Taxi;
and race and sex for ACS Income.
We use their dataset-specific binary encodings and average the metrics over
these attributes, matching the reported baselines. Each attribute defines four
$(A,Y)$ groups; Appendix~\ref{appendix:metric} gives the metric formulas.

\paragraph{Data preparation.}
For the five tabular benchmarks and ACS Income, we follow the data preparation
procedure of \citet{tong2025latent}, with the evaluated attributes specified above. Taxi
and ACS Income use the \textsc{WhyShift} preprocessing procedure
\citep{liu2024need} adopted in that benchmark. The
resulting sizes are listed in Tables~\ref{tab:datasets-t1} and~\ref{tab:datasets-acs}.
Categorical columns are one-hot encoded for the MLP-based baselines.

\begin{table}[H]
\centering
\caption{The five tabular benchmarks and their splits.}
\label{tab:datasets-t1}
\small
\begin{tabular}{lrrrr}
\toprule
Dataset & Train & Validation & Test & Total \\
\midrule
Adult    & 26,048 & 6,513 & 16,281 & 48,842 \\
Bank     & 32,551 & 8,138 & 4,522  & 45,211 \\
Default  & 21,600 & 5,400 & 3,000  & 30,000 \\
Shoppers & 8,877  & 2,220 & 1,233  & 12,330 \\
Taxi     & 9,139  & 2,285 & 1,270  & 12,694 \\
\bottomrule
\end{tabular}
\end{table}

\begin{table}[H]
\centering
\caption{ACS Income settings. The first three train and test on the same state; the last three
transfer to a held-out state.}
\label{tab:datasets-acs}
\small
\begin{tabular}{lrrrr}
\toprule
Setting & Train & Validation & Test & Total \\
\midrule
AZ                 & 23,959 & 5,990  & 3,328  & 33,277 \\
MA                 & 28,881 & 7,221  & 4,012  & 40,114 \\
MI                 & 36,005 & 9,002  & 5,001  & 50,008 \\
AZ$\rightarrow$MA  & 26,621 & 6,656  & 40,114 & 73,391 \\
MA$\rightarrow$MI  & 32,091 & 8,023  & 50,008 & 90,122 \\
MI$\rightarrow$AZ  & 40,006 & 10,002 & 33,277 & 83,285 \\
\bottomrule
\end{tabular}
\end{table}

\subsubsection{Baselines}
\label{appendix:baselines}
We compare standard classifiers, TFMs and their adaptations, robust methods
without true group labels, and methods that use true group labels.
TabPFN, TabPFN-3.5, TabPFN (fine-tuned), and TabPFN (class-balanced) are counted as separate baselines.
For this count, TabDPT with retrieved and full contexts is treated as one method,
GEORGE with $k$-means and Gaussian mixture clustering as one method, and
Fair-TabICL with group-balanced and uncertainty-based context selection as one method.
Tables~\ref{tab:protocol} and~\ref{tab:hyperparams} summarize the training
schedules and hyperparameter settings.

\paragraph{Standard methods.}
\begin{itemize}[left=0pt]
    \item ERM-MLP. The classifier of \citet{tong2025latent} with the reweighting
    removed, trained by empirical risk minimization.
    \item XGBoost \citep{chen2016xgboost}. Gradient-boosted trees.
    \item CatBoost \citep{prokhorenkova2018catboost}. Gradient-boosted trees with
    ordered boosting and native handling of categorical columns.
\end{itemize}

\paragraph{TFMs.}
\begin{itemize}[left=0pt]
    \item TabPFN \citep{grinsztajn2026tabpfn}. Predicts a query from a labeled
    context in a single forward pass, with no gradient update. This is the pretrained model we adapt.
    \item TabPFN-3.5 \citep{priorlabs2026tabpfn35}. A TabPFN model with
    a shared pretrained checkpoint for classification and regression.
    \item TabICL \citep{qu2025tabicl}. Scales in-context learning to larger tables
    through a two-stage encoder.
    \item TabFM \citep{google2026tabfm}. Combines row and column attention with
    row compression and a separate transformer for in-context learning.
    \item EXAONE \citep{eo2026exaonetabular}. Interleaves feature-axis and
    sample-axis attention through summary tokens for in-context learning.
    \item Causilo \citep{nums2026causilo}. A compact in-context model that
    summarizes the context per feature, mixes features within each row, and
    predicts through a stack of attention layers in which queries read the
    labeled context.
    \item TabDPT \citep{ma2025tabdpt}. Pretrained on real rather than synthetic
    tables. We evaluate it with nearest-neighbor retrieval and with the full training set as
    context.
    \item TuneTables \citep{feuer2024tunetables}. Replaces the context with a
    learned prompt, so the training set is compressed into a small number of tokens.
    Its prompt is tied to the input representation of TabPFN v1, which we use
    for this method. Unless otherwise specified, the other TabPFN-based methods use v3. Comparisons with
    TuneTables therefore reflect differences in both the pretrained model and
    the adaptation method.
    \item TabPFN (fine-tuned) \citep{grinsztajn2026tabpfn}. Adapts the pretrained
    transformer using the fine-tuning procedure, with context and loss-evaluation
    batches drawn from the training set. We use the standard fine-tuning settings,
    including an $80/20$ split into context and loss-evaluation sets.
    \item TabPFN (class-balanced). The same fine-tuning procedure and settings as
    TabPFN (fine-tuned), with the cross-entropy weighted by the inverse class
    frequency of the training labels:
    \begin{equation*}
    \min_{\theta}\;
    \mathbb E_{\widehat P^{(e)}}\!\left[
    w_{Y}\,\ell\!\left(f_\theta(X;\mathcal D^{(c)}),Y\right)\right],
    \qquad w_c=\frac{1}{|\mathcal Y|\,\widehat p_c},
    \end{equation*}
    where $\widehat p_c=n^{-1}\sum_{i=1}^{n}\mathbf 1\{y_i=c\}$ is the empirical
    frequency of class $c$ in the full training dataset $\mathcal D$, and $w_Y$
    is the weight for the observed class $Y$.
    \item MixturePFN \citep{xu2025mixture}. Partitions the training set with $k$-means,
    fine-tunes one expert per cluster on that cluster, and routes a test point to the
    expert of its nearest centroid. We adapt this procedure to the pretrained
    TabPFN model used in our experiments, following \citet{LiuBeta2025} with
    a context budget of $3000$ samples and $\lceil n/3000 \rceil$ experts.
    \item LoCalPFN \citep{thomas2024retrieval}. Replaces the full training set as context
    with the query's $k$ nearest neighbors. We evaluate the retrieval variant
    (TabPFN-kNN) with $k=1000$ and no parameter updates, using the same pretrained
    TabPFN model as our method.
    \item BETA \citep{LiuBeta2025}. Trains a lightweight input adapter placed before
    the frozen pretrained TabPFN transformer. We use a two-layer MLP with $16$
    batch-ensemble members. Each member receives a bootstrapped context, and their
    predictions are averaged. We use the authors' settings with a
    context size of $1000$ and the same pretrained TabPFN model as our method.
    \item DistPFN \citep{lee2026distpfn}. Adjusts the pretrained TabPFN's
    class probabilities at test time using the ratio of their test-set average
    to the class prior of the context, without updating model parameters.
\end{itemize}

\paragraph{Robust methods without true group labels.}
For CVaR-DRO, $\chi^2$-DRO, KL-DRO, EIIL, JTT, FAM, SRDO, and LSR, we use
the dataset-level accuracies and standard deviations across seeds reported by
\citet{tong2025latent}, marked $^{\dagger}$ in the accuracy tables, and compute
the Avg.\ columns across the corresponding datasets or settings.
\begin{itemize}[left=0pt]
    \item CVaR-DRO and $\chi^2$-DRO \citep{levy2020large}, and
    KL-DRO \citep{duchi2021learning}. Optimize the worst risk over an
    $f$-divergence ball around the training distribution, without constructing groups.
    \item JTT \citep{liu2021just}. Trains once, then upweights the samples the first
    model misclassifies and trains again.
    \item EIIL \citep{creager2021environment}. Infers environments that maximally
    violate an invariance criterion, then optimizes over them.
    \item FAM \citep{petzka2021relative,zou2024towards}. Regularizes the flatness of
    the minimum reached by training.
    \item SRDO \citep{shen2020stable}. Reweights samples so that the covariates
    become decorrelated.
    \item LSR \citep{tong2025latent}. Reweights samples using a score-based density
    estimate on a learned latent space. This is the method whose benchmark we adopt.
    \item GEORGE \citep{sohoni2020no}. Clusters the penultimate representation of a
    trained model within each class, then minimizes the GroupDRO objective over
    the resulting clusters. We evaluate both $k$-means and Gaussian mixture clustering.
    \item BPA \citep{seo2022unsupervised}. Clusters a trained representation and
    reweights each cluster by its inverse frequency and its running loss.
    \item SPARE \citep{yang2024identifying}. Clusters the model output early in
    training, when the spurious feature dominates, then resamples by inverse cluster size.
\end{itemize}

\paragraph{Methods using true group labels.}
Results for methods that use true group labels are reported separately
in Tables~\ref{appendix:tab1-known} and~\ref{appendix:tab2-known}.
A superscript $^{*}$ denotes access to true group labels.
\begin{itemize}[left=0pt]
    \item GroupDRO \citep{sagawa2019distributionally}. Minimizes the worst risk over
    the true groups.
    \item Fair-TabICL \citep{kenfack2026towards}. Selects in-context examples using
    the group label, either to balance the groups or by uncertainty in predicting the
    attribute.
\end{itemize}

\subsubsection{Group and source construction for DR-TFM}
\label{appendix:ms-sources}
\paragraph{MS source construction.}
We use the MS-C and MS-R constructions described in
Appendix~\ref{appendix:alg-ms}. Following \citet{kim2026distributionally},
we start with ten candidate sources on each dataset.
Each MS-R source contains approximately $n_c/5$ context samples.
The main tables report DR-TFM (MS-C) as DR-TFM (MS).

\paragraph{True groups for GR.}
On the five tabular benchmarks, DR-TFM (GR)$^{*}$ and GroupDRO$^{*}$ use
the four groups formed by each binary attribute $A\in\mathcal A$ and the
class label $Y$. This yields $4|\mathcal A|$ groups: twelve on Adult, Default,
Shoppers, and Taxi, and sixteen on Bank. Each sample belongs to one group per
attribute, so groups associated with different attributes may overlap.
This per-attribute construction avoids sparse intersections across three or
four attributes.

On ACS Income, both methods use the eight disjoint $(Y,\mathrm{sex},\mathrm{race})$
groups. Each evaluation group is a union of two of these groups.
These true-group variants require no cluster selection.

\paragraph{Sampling with true groups.}
To stratify the context and loss-evaluation split on the five tabular benchmarks,
we use the four cells of $Y$ paired with sex (Adult and Default), marital status (Bank),
weekend status (Shoppers), or pickup month (Taxi).
On ACS Income, we use the eight joint cells above.
GR$^{*}$ and MS$^{*}$ split the training set $70/30$ into context and
loss-evaluation sets, stratified by these cells. These sampling partitions are
distinct from the overlapping groups used in the GR objective on the five
tabular benchmarks.

\paragraph{MS sources with true groups.}
DR-TFM (MS)$^{*}$ forms sources from the joint attribute values, omitting $Y$:
eight combinations on Adult, Default, Shoppers, and Taxi, sixteen on Bank,
and four sex--race combinations on ACS Income. We retain a source only if
its context contains both response classes.
Model selection for MS$^{*}$ maximizes worst-group validation accuracy over the same
groups used for GR$^{*}$: the $4|\mathcal A|$ overlapping groups on the five tabular
benchmarks and the eight joint groups on ACS Income.

\subsubsection{Evaluation metrics}
\label{appendix:metric}
We report the evaluation metrics of \citet{tong2025latent}. Let $\mathcal{A}$ denote the evaluated
attributes of a dataset. For an attribute $A \in \mathcal{A}$, the test set splits into the
four groups
\[
\mathcal{C}_A(v,c)
=
\bigl\{\, i \;:\; A_i = v,\; Y_i = c \,\bigr\},
\qquad
(v,c) \in \{0,1\}^2 ,
\]
each with accuracy
\[
\mathrm{acc}_A(v,c)
=
\frac{1}{\lvert \mathcal{C}_A(v,c) \rvert}
\sum_{i \in \mathcal{C}_A(v,c)}
\mathbf{1}\bigl\{ \widehat{Y}_i = Y_i \bigr\} .
\]
The two reported quantities are
\[
\mathrm{Acc}_{\mathrm{worst}}
=
\frac{1}{\lvert \mathcal{A} \rvert}
\sum_{A \in \mathcal{A}}
\min_{(v,c) \in \{0,1\}^2} \mathrm{acc}_A(v,c),
\qquad
\mathrm{Acc}_{\mathrm{mean}}
=
\frac{1}{\lvert \mathcal{A} \rvert}
\sum_{A \in \mathcal{A}}
\frac{1}{4}
\sum_{(v,c) \in \{0,1\}^2} \mathrm{acc}_A(v,c) .
\]
Here, $\mathrm{Acc}_{\mathrm{mean}}$ assigns equal weight to the four groups,
while $\mathrm{Acc}_{\mathrm{worst}}$ measures the accuracy
of the least accurate group for each attribute. Both metrics are averaged over
attributes and reported as percentages.
For summaries across datasets or settings, we report the average of mean group
accuracies and the average of worst-group accuracies, giving each dataset or
setting equal weight.

\subsubsection{Training protocol}
\label{appendix:protocol}
\paragraph{Group estimation and data splitting.}
\label{appendix:clustering}
The TabPFN-based real-data variants without true group labels use
$512$-dimensional pretrained embeddings and the preparation procedure in
Appendix~\ref{appendix:method-preparation}. The true-group variants follow
Appendix~\ref{appendix:ms-sources}.

\paragraph{Classifier architecture.}
The MLP-based baselines use the classifier architecture of \citet{tong2025latent}: a four-layer neural network
$d \rightarrow 1024 \rightarrow 1024 \rightarrow 512 \rightarrow 2$ with LeakyReLU
activations. These baselines use the one-hot-encoded original features, except
for LSR, which retains its VAE-based latent representation.

\paragraph{Optimization.}
DR-TFM (GR) uses Adam and DR-TFM (MS) uses AdamW.
ERM-MLP and GroupDRO use Adam, following \citet{tong2025latent}.
Learning rates, batch sizes, and schedules are listed in
Tables~\ref{tab:protocol} and~\ref{tab:hyperparams}.
Unless otherwise noted, results are averaged over seeds $0$, $1$, and $2$.

\paragraph{Schedule.}
Table~\ref{tab:protocol} lists the training budgets and model selection criteria.
The DR-TFM variants without true group labels use early stopping based on worst-group
validation accuracy over estimated groups.
For MS-C and MS-R, we evaluate this metric every five parameter updates and
stop after five consecutive validation checks without improvement.
GR$^{*}$ and MS$^{*}$ use the true
groups specified in Appendix~\ref{appendix:ms-sources}.
Table~\ref{tab:hyperparams} specifies the patience.

\begin{table}[H]
\centering
\caption{Training budgets and model selection criteria for DR-TFM and the MLP-based baselines.}
\label{tab:protocol}
\resizebox{\textwidth}{!}{%
\begin{tabular}{l >{\raggedright\arraybackslash}p{4.0cm} c >{\raggedright\arraybackslash}p{4.4cm} >{\raggedright\arraybackslash}p{3.8cm}}
\toprule
Method & Training & Epochs / steps & Selection criterion & Learning rate schedule \\
\midrule
ERM-MLP & ERM, from scratch & 1000 & validation accuracy &
  multiply by $0.9$ after $20$ epochs without improvement in mean training loss \\
GroupDRO & GroupDRO over the true groups, from scratch & 1000 &
  validation worst-group accuracy (true groups) & same as above \\
GEORGE stage 1 & ERM, from scratch (feature extractor) & 1000 & validation accuracy & none \\
GEORGE stage 2 & GroupDRO over the estimated clusters, newly initialized model & 300 &
  worst-group accuracy over the estimated validation clusters & none \\
BPA & base ERM, then retraining with cluster reweighting & 1000 / 300 &
  worst-group accuracy over the estimated validation clusters & none \\
SPARE & base ERM, then retraining with importance sampling & 1000 / 300 &
  worst-group accuracy over the estimated validation clusters & none \\
LSR & original LSR training procedure & 4000 / 4000 / 1000 &
  lowest training loss & multiply by $0.9$ after $20$ epochs without improvement \\
DR-TFM & query scaling with a fixed TabPFN feature map &
  at most $1000$ steps &
  validation worst-group accuracy over the estimated clusters & --- \\
\bottomrule
\end{tabular}}
\end{table}

\subsubsection{Model selection and hyperparameters}
\label{appendix:hyperparams}
\paragraph{Selection rule.}
The following settings apply to the real-data experiments.
We follow the original methods' hyperparameter settings and selection procedures
where applicable. Group-based robust baselines select remaining hyperparameters
using worst-group validation accuracy, evaluated on true groups when available
and estimated groups otherwise. XGBoost and CatBoost use validation log-loss.
Test data are used only for final evaluation.

\paragraph{Selected hyperparameters.}
Table~\ref{tab:hyperparams} lists the search spaces. For SPARE, we select $T_{\mathrm{init}}$ using worst-group validation accuracy
over estimated groups; the numbers of clusters of SPARE and BPA follow GEORGE. ERM-MLP uses the fixed hyperparameters
of \citet{tong2025latent}, including zero weight decay.

\paragraph{Cluster selection for DR-TFM (GR).}
Following GEORGE \citep{sohoni2020no}, we select the number of clusters
separately for each class over $G_c\in\{2,4,6,8\}$ by the silhouette score.
This selection uses neither true group labels nor validation data.
After selection, validation samples are assigned to the nearest cluster
center within their class.

For both clustering methods, this criterion selects two clusters per class
($G=4$) across all datasets and three seeds reported in
Tables~\ref{appendix:tab1-free} and~\ref{appendix:tab2-free}.

\paragraph{Hyperparameters for DR-TFM (MS).}
For both source constructions in Appendix~\ref{appendix:ms-sources}, we fix
$\epsilon_2=1$, following \citet{kim2026distributionally}.
With the uniform nominal vector, this radius allows all convex mixtures of
the source conditionals.

\citet{kim2026distributionally} develop their method for unsupervised domain
adaptation and also evaluate it under subpopulation shift.
In their formulation, $\epsilon_1$ controls a Wasserstein neighborhood around
the empirical target input distribution, accounting for uncertainty when
target data are limited.
Our setting focuses on subpopulation shift and does not use target-domain
data for adaptation or model selection.
Here, $\epsilon_1$ controls the size of perturbations to the loss-evaluation
embeddings.
Table~\ref{tab:ms-radius-selection} shows that worst-group validation accuracy
on Adult varies little across the tested values of $\epsilon_1$.
Given this limited sensitivity, we select $\epsilon_1=0.2$ once using these
validation results and keep it fixed across all datasets.
Each model update uses one feature ascent step and one mixture-weight ascent
step, with $\eta_z=\epsilon_1$ and $\eta_\beta=0.05$.

\begin{table}[H]
\centering
\caption{Worst-group validation accuracy of DR-TFM (MS) on Adult for selecting $\epsilon_1$,
with $\epsilon_2=1$ and ten initial candidate sources.
Accuracy is evaluated over estimated groups.
Results are averaged over three seeds; parentheses report standard deviations.}
\label{tab:ms-radius-selection}
\small
\setlength{\tabcolsep}{6pt}
\begin{tabular}{l cccccc}
\toprule
$\epsilon_1$ & $0$ & $0.2$ & $0.4$ & $0.6$ & $0.8$ & $1$ \\
\midrule
Validation & 70.43 \sd{1.21} & 70.44 \sd{1.18} & 70.44 \sd{1.13} & 70.38 \sd{1.15} & 70.37 \sd{1.15} & 70.39 \sd{1.12} \\
\bottomrule
\end{tabular}
\end{table}

\paragraph{GroupDRO updates and regularization.}
The GroupDRO updates use the group-size adjustment described in
Appendix~\ref{appendix:alg}, with an adversarial learning rate of $0.1$
and $C=5$.
DR-TFM (GR) fixes weight decay at $5\times10^{-3}$.
For GEORGE, BPA, SPARE, and GroupDRO, weight decay is selected for the robust
training stage. The initial ERM stages of GEORGE, BPA, and SPARE retain zero
weight decay, so their estimated partitions remain fixed across the search.

\begin{table}[H]
\centering
\caption{Hyperparameters. The final column lists selected quantities;
all other settings are fixed. LR denotes the model learning rate.}
\label{tab:hyperparams}
\resizebox{\textwidth}{!}{%
\begin{tabular}{l cccccc c >{\raggedright\arraybackslash}p{3.6cm}}
\toprule
Method & $\eta$ & $C$ & LR & Batch size & Budget & Patience & Weight decay & Selected quantities \\
\midrule
DR-TFM (GR; $k$-means)      & 0.1 & 5 & 0.01 & full & 1000 & 5 & 0.005 & $G_c$ per class by silhouette over $\{2,4,6,8\}$ \\
DR-TFM (GR; GMM)          & 0.1 & 5 & 0.01 & full & 1000 & 5 & 0.005 & $G_c$ per class by silhouette over $\{2,4,6,8\}$ \\
DR-TFM (GR)$^{*}$          & 0.1 & 5 & 0.01 & full & 1000 & 5 & 0.005 & none (groups given) \\
DR-TFM (MS)           & --- & --- & 0.001 & 1024 & 300 & 5 & $10^{-4}$ & $\epsilon_1$ once on Adult; fixed thereafter \\
DR-TFM (MS)$^{*}$         & --- & --- & 0.001 & 1024 & 300 & 5 & $10^{-4}$ & none (groups given) \\
ERM-MLP              & --- & --- & 0.001 & 4096 & 1000 & 20 & 0 & none \\
GroupDRO             & 0.1 & 5 & 0.001 & 4096 & 1000 & 20 & $\{10^{-4}, 5\!\times\!10^{-3}, 10^{-2}\}$ & weight decay \\
GEORGE ($k$-means / GMM)   & 0.1 & 5 & 0.001 & 4096 & 1000 + 300 epochs & 20 & $\{10^{-4}, 5\!\times\!10^{-3}, 10^{-2}\}$ & $G_{\mathrm{cl}}$ by silhouette over $\{2,4,6,8\}$; weight decay \\
SPARE                    & 0.1 & 5 & 0.001 & 4096 & 1000 + 300 epochs & 20 & $\{10^{-4}, 5\!\times\!10^{-3}, 10^{-2}\}$ & $T_{\mathrm{init}}$; weight decay \\
BPA                      & 0.1 & 5 & 0.001 & 4096 & 1000 + 300 epochs & 20 & $\{10^{-4}, 5\!\times\!10^{-3}, 10^{-2}\}$ & weight decay; clusters as GEORGE \\
XGBoost                  & --- & --- & $\{0.03, 0.05, 0.1\}$ & --- & $\le 600$ & --- & --- & depth, learning rate, tree count \\
CatBoost                 & --- & --- & $\{0.03, 0.05, 0.1\}$ & --- & $\le 600$ & --- & --- & depth, learning rate, tree count \\
TabICL                   & --- & --- & --- & --- & --- & --- & --- & none \\
TabDPT                   & --- & --- & --- & 256 & --- & --- & --- & none \\
\bottomrule
\end{tabular}}
\end{table}

\begin{table}[H]
\centering
\caption{Additional baseline settings.}
\label{tab:hyperparams-adapt}
\footnotesize
\setlength{\tabcolsep}{4pt}
\begin{tabularx}{\textwidth}{@{}ll>{\raggedright\arraybackslash}X@{}}
\toprule
Method & Trained component & Settings \\
\midrule
TabPFN (fine-tuned) & Full backbone & Learning rate $10^{-5}$, weight decay $0.02$, 30 epochs, patience 8 \\
TabPFN (class-balanced) & Full backbone & As TabPFN (fine-tuned), with inverse-class-frequency weights in the cross-entropy \\
BETA & Input adapter & 16 ensemble members, learning rate $0.003$, weight decay $0.02$, batch size 1024, 30 epochs, patience 5 \\
MixturePFN & Expert models & Context size 3000; experts use the TabPFN fine-tuning settings above \\
LoCalPFN & None & Retrieval context of 1000 neighbors per query batch \\
TuneTables & Soft prompt & 10 tokens, $\le 31$ epochs, patience 5, TabPFN v1 \\
Fair-TabICL & None & Context balanced by group or selected by predictive uncertainty \\
TabICL / TabDPT & None & TabDPT inference batch size 256 \\
TabFM & None & Full training context; 32 ensemble members \\
EXAONE & None & Full training context; 8 ensemble members \\
XGBoost / CatBoost & Trees & Tree counts $\{25,50,100,200,300,400,600\}$, selected by minimum validation log loss \\
\bottomrule
\end{tabularx}
\end{table}

\subsubsection{Ablation on the number of clusters}
\label{appendix:kablation}
We vary the number of clusters per class, denoted by $G_{\mathrm{cl}}$, over $\{2,4,6,8\}$ and compare
DR-TFM (GR) with GEORGE using class-wise $k$-means and GMM on the five tabular
benchmarks and three within-state ACS Income settings. The ablation isolates
$G_{\mathrm{cl}}$, so every remaining hyperparameter is held fixed and both methods use a weight decay of
$5 \times 10^{-3}$.
For this ablation, each class has $G_c=G_{\mathrm{cl}}$ clusters,
giving $G=2G_{\mathrm{cl}}$ estimated groups.
Tables~\ref{tab:k-sensitivity} and~\ref{tab:k-sensitivity-acs} report both
clustering methods. Their standard deviations are computed across seeds after
averaging accuracies over datasets or within-state settings.

\paragraph{Results.}
DR-TFM (GR) outperforms GEORGE in both metrics at every tested value of
$G_{\mathrm{cl}}$ for both clustering methods.
With $k$-means, its gains in the average of worst-group accuracies range from
$8.77$ to $24.41$ pp on the five tabular benchmarks and from $4.91$ to $25.66$ pp
on ACS within-state settings.
Compared with GEORGE, DR-TFM (GR) shows substantially more stable performance
as the number of clusters varies, with smaller ranges in both metrics for
both clustering methods.

\begin{table}[H]
\centering
\caption{Ablation on $G_{\mathrm{cl}}$ for the five tabular benchmarks, averaged over
datasets and three seeds. Parentheses report standard deviations across seeds
of the dataset-averaged accuracies. $\Delta$ is DR-TFM (GR) minus GEORGE.
Range is the maximum minus minimum within each row across the four values of $G_{\mathrm{cl}}$.
Ranges are computed before rounding.}
\label{tab:k-sensitivity}
\setlength{\tabcolsep}{4pt}
\resizebox{\textwidth}{!}{%
\begin{tabular}{l cc cc cc cc | cc}
\toprule
\multirow{2}{*}{Method} & \multicolumn{2}{c}{$G_{\mathrm{cl}}=2$} & \multicolumn{2}{c}{$G_{\mathrm{cl}}=4$} & \multicolumn{2}{c}{$G_{\mathrm{cl}}=6$} & \multicolumn{2}{c}{$G_{\mathrm{cl}}=8$} & \multicolumn{2}{c}{Range (pp)} \\
\cmidrule(lr){2-3}\cmidrule(lr){4-5}\cmidrule(lr){6-7}\cmidrule(lr){8-9}\cmidrule(lr){10-11}
 & Mean & Worst & Mean & Worst & Mean & Worst & Mean & Worst & Mean & Worst \\
\midrule
DR-TFM (GR; $k$-means) & 79.31 & 72.57 & 77.42 & 68.37 & 75.97 & 63.90 & 75.09 & 61.39 & 4.22 & 11.18 \\
 & \sd{0.17} & \sd{0.64} & \sd{0.33} & \sd{1.29} & \sd{0.63} & \sd{3.00} & \sd{0.77} & \sd{1.12} & & \\
GEORGE ($k$-means) & 73.79 & 63.80 & 61.80 & 43.96 & 64.22 & 48.25 & 58.64 & 44.45 & 15.15 & 19.84 \\
 & \sd{0.48} & \sd{1.21} & \sd{2.28} & \sd{4.56} & \sd{1.04} & \sd{2.78} & \sd{5.89} & \sd{9.45} & & \\
\cmidrule[0.3pt](lr){1-11}
\quad $\Delta$ (pp) & \textbf{+5.52} & \textbf{+8.77} & \textbf{+15.62} & \textbf{+24.41} & \textbf{+11.75} & \textbf{+15.65} & \textbf{+16.45} & \textbf{+16.94} & 10.93 & 15.64 \\
\addlinespace[2pt]
\midrule
DR-TFM (GR; GMM) & 79.19 & 71.81 & 77.12 & 67.71 & 75.99 & 64.09 & 74.57 & 59.67 & 4.62 & 12.13 \\
 & \sd{0.15} & \sd{0.38} & \sd{1.08} & \sd{1.29} & \sd{0.90} & \sd{2.26} & \sd{0.38} & \sd{3.08} & & \\
GEORGE (GMM) & 73.75 & 65.26 & 68.68 & 59.69 & 63.40 & 49.30 & 62.26 & 46.10 & 11.49 & 19.16 \\
 & \sd{1.18} & \sd{1.66} & \sd{1.37} & \sd{0.56} & \sd{1.66} & \sd{4.99} & \sd{2.29} & \sd{8.85} & & \\
\cmidrule[0.3pt](lr){1-11}
\quad $\Delta$ (pp) & \textbf{+5.44} & \textbf{+6.55} & \textbf{+8.44} & \textbf{+8.02} & \textbf{+12.59} & \textbf{+14.79} & \textbf{+12.31} & \textbf{+13.57} & 7.15 & 8.24 \\
\addlinespace[2pt]
\bottomrule
\end{tabular}}
\end{table}

\begin{table}[H]
\centering
\caption{Ablation on $G_{\mathrm{cl}}$ for ACS Income, averaged over the three
within-state settings and three seeds. Parentheses report standard deviations
across seeds of the setting-averaged accuracies. $\Delta$ is DR-TFM (GR) minus GEORGE. Range is the maximum minus minimum within each row across the four values of $G_{\mathrm{cl}}$.
Ranges are computed before rounding.}
\label{tab:k-sensitivity-acs}
\setlength{\tabcolsep}{4pt}
\resizebox{\textwidth}{!}{%
\begin{tabular}{l cc cc cc cc | cc}
\toprule
\multirow{2}{*}{Method} & \multicolumn{2}{c}{$G_{\mathrm{cl}}=2$} & \multicolumn{2}{c}{$G_{\mathrm{cl}}=4$} & \multicolumn{2}{c}{$G_{\mathrm{cl}}=6$} & \multicolumn{2}{c}{$G_{\mathrm{cl}}=8$} & \multicolumn{2}{c}{Range (pp)} \\
\cmidrule(lr){2-3}\cmidrule(lr){4-5}\cmidrule(lr){6-7}\cmidrule(lr){8-9}\cmidrule(lr){10-11}
 & Mean & Worst & Mean & Worst & Mean & Worst & Mean & Worst & Mean & Worst \\
\midrule
DR-TFM (GR; $k$-means) & 79.21 & 71.06 & 80.00 & 72.99 & 79.53 & 72.22 & 79.70 & 72.59 & 0.78 & 1.93 \\
 & \sd{0.14} & \sd{0.12} & \sd{0.12} & \sd{0.40} & \sd{0.38} & \sd{1.11} & \sd{0.04} & \sd{0.39} & & \\
GEORGE ($k$-means) & 74.36 & 66.15 & 64.70 & 54.67 & 65.41 & 56.41 & 66.75 & 46.94 & 9.66 & 19.22 \\
 & \sd{3.03} & \sd{2.31} & \sd{1.35} & \sd{4.89} & \sd{2.71} & \sd{1.27} & \sd{6.77} & \sd{9.64} & & \\
\cmidrule[0.3pt](lr){1-11}
\quad $\Delta$ (pp) & \textbf{+4.85} & \textbf{+4.91} & \textbf{+15.29} & \textbf{+18.32} & \textbf{+14.11} & \textbf{+15.81} & \textbf{+12.95} & \textbf{+25.66} & 10.44 & 20.75 \\
\addlinespace[2pt]
\midrule
DR-TFM (GR; GMM) & 79.82 & 72.17 & 79.96 & 72.86 & 79.42 & 71.85 & 79.69 & 72.32 & 0.54 & 1.01 \\
 & \sd{0.16} & \sd{0.82} & \sd{0.24} & \sd{0.38} & \sd{0.47} & \sd{1.18} & \sd{0.43} & \sd{0.47} & & \\
GEORGE (GMM) & 64.98 & 59.52 & 63.69 & 56.20 & 64.28 & 55.68 & 65.08 & 55.01 & 1.40 & 4.51 \\
 & \sd{2.71} & \sd{2.97} & \sd{1.40} & \sd{2.65} & \sd{2.39} & \sd{3.18} & \sd{3.25} & \sd{1.81} & & \\
\cmidrule[0.3pt](lr){1-11}
\quad $\Delta$ (pp) & \textbf{+14.84} & \textbf{+12.65} & \textbf{+16.27} & \textbf{+16.66} & \textbf{+15.13} & \textbf{+16.18} & \textbf{+14.61} & \textbf{+17.31} & 1.66 & 4.66 \\
\addlinespace[2pt]
\bottomrule
\end{tabular}}
\end{table}

\subsection{Full result tables}
\label{appendix:full-tables}
Method names follow Table~\ref{tab:main-table1}; MS-C and MS-R denote
clustered and randomly sampled sources, respectively.
Tables~\ref{appendix:tab1-free}--\ref{appendix:tab2-known} report the full results,
grouped by method family and access to true group labels.
Detailed results for the additional TFMs are reported separately in
Appendix~\ref{appendix:tfm-extension-results}.
Results are averaged over three seeds.
Values in parentheses are standard deviations over seeds where available.
The Avg.\ columns report the average of mean group accuracies and the average of
worst-group accuracies across the constituent datasets or settings.
Parenthesized values in these columns are the averages of the per-dataset or per-setting standard deviations across seeds.

\paragraph{Comparisons with true group labels.}
The true-group protocols for Tables~\ref{appendix:tab1-known}
and~\ref{appendix:tab2-known} are described in
Appendices~\ref{appendix:baselines} and~\ref{appendix:ms-sources}.
On the five tabular benchmarks, DR-TFM (GR) with estimated groups achieves
an average of worst-group accuracies of $72.57\%$, exceeding GroupDRO with true
groups ($71.79\%$) and coming within $1.09$ pp of DR-TFM (GR) with true
groups ($73.66\%$).

The recovery with true groups under the same query-scaling architecture suggests
that the degradation on MA$\rightarrow$MI may reflect a mismatch between the
estimated groups used for adaptation and model selection and the target evaluation
groups (Tables~\ref{appendix:tab2-free} and~\ref{appendix:tab2-known}).

\begingroup
\scriptsize
\setlength{\tabcolsep}{3.4pt}
\setlength{\LTleft}{\fill}
\setlength{\LTright}{\fill}
\setlength{\LTcapwidth}{\textwidth}
\begin{longtable}{@{\extracolsep{\fill}}lcccccccccc|cc}
\caption{Results on the five tabular benchmarks for methods without true group labels.
Avg.\ denotes the average of each metric across datasets.
Values in parentheses are standard deviations over three seeds.
Accuracy results marked $^{\dagger}$, including standard deviations, are taken from \citet{tong2025latent}.
Bold method names indicate our variants.}\label{appendix:tab1-free}\\
\toprule
\multirow{2}{*}{Methods} & \multicolumn{2}{c}{Adult} & \multicolumn{2}{c}{Bank} & \multicolumn{2}{c}{Default} & \multicolumn{2}{c}{Shoppers} & \multicolumn{2}{c|}{Taxi} & \multicolumn{2}{c}{Avg.} \\
\cmidrule(lr){2-3}\cmidrule(lr){4-5}\cmidrule(lr){6-7}\cmidrule(lr){8-9}\cmidrule(lr){10-11}\cmidrule(lr){12-13}
 & Mean & Worst & Mean & Worst & Mean & Worst & Mean & Worst & Mean & Worst & Mean & Worst \\
\midrule
\endfirsthead
\caption[]{Five tabular benchmarks, methods without true group labels (continued).}\\
\toprule
\multirow{2}{*}{Methods} & \multicolumn{2}{c}{Adult} & \multicolumn{2}{c}{Bank} & \multicolumn{2}{c}{Default} & \multicolumn{2}{c}{Shoppers} & \multicolumn{2}{c|}{Taxi} & \multicolumn{2}{c}{Avg.} \\
\cmidrule(lr){2-3}\cmidrule(lr){4-5}\cmidrule(lr){6-7}\cmidrule(lr){8-9}\cmidrule(lr){10-11}\cmidrule(lr){12-13}
 & Mean & Worst & Mean & Worst & Mean & Worst & Mean & Worst & Mean & Worst & Mean & Worst \\
\midrule
\endhead
\midrule
\multicolumn{13}{r}{\footnotesize Continued on next page}\\
\endfoot
\bottomrule
\endlastfoot
\multicolumn{13}{@{}l}{\emph{Standard methods}}\\*
ERM-MLP & 73.45 & 45.29 & 71.86 & 42.03 & 63.70 & 28.52 & 76.44 & 50.11 & 75.27 & 57.37 & 72.14 & 44.67\\*[-1.2pt]
 & \sd{0.16} & \sd{0.90} & \sd{1.80} & \sd{4.18} & \sd{0.48} & \sd{0.21} & \sd{0.47} & \sd{1.85} & \sd{0.60} & \sd{1.44} & \sd{0.70} & \sd{1.72}\\[2pt]
XGBoost & 77.09 & 52.58 & 68.69 & 33.87 & 64.78 & 30.72 & 78.03 & 55.42 & 75.81 & 57.57 & 72.88 & 46.03\\*[-1.2pt]
 & \sd{0.17} & \sd{0.48} & \sd{0.31} & \sd{0.92} & \sd{0.11} & \sd{0.13} & \sd{0.19} & \sd{0.24} & \sd{0.21} & \sd{0.16} & \sd{0.20} & \sd{0.39}\\[2pt]
CatBoost & 76.50 & 51.33 & 68.94 & 34.43 & 64.75 & 31.03 & 77.50 & 54.24 & 76.04 & 58.71 & 72.75 & 45.95\\*[-1.2pt]
 & \sd{0.22} & \sd{0.73} & \sd{0.49} & \sd{0.65} & \sd{0.05} & \sd{0.00} & \sd{0.62} & \sd{2.30} & \sd{0.07} & \sd{0.54} & \sd{0.29} & \sd{0.84}\\[2pt]
\addlinespace[1pt]
\midrule
\addlinespace[1pt]
\multicolumn{13}{@{}l}{\emph{TFMs}}\\*
TabICL & 76.40 & 51.17 & 71.90 & 40.38 & 64.70 & 31.48 & 78.71 & 56.70 & 76.07 & 57.86 & 73.56 & 47.52\\*[-1.2pt]
 & \sd{0.66} & \sd{1.26} & \sd{0.18} & \sd{0.54} & \sd{0.09} & \sd{0.19} & \sd{0.05} & \sd{0.00} & \sd{0.18} & \sd{0.35} & \sd{0.23} & \sd{0.47}\\[2pt]
TabDPT (retrieval) & 74.13 & 46.80 & 70.06 & 37.47 & 64.25 & 30.18 & 78.70 & 58.11 & 75.90 & 58.44 & 72.61 & 46.20\\*[-1.2pt]
 & \sd{0.22} & \sd{0.66} & \sd{0.25} & \sd{1.12} & \sd{0.18} & \sd{0.30} & \sd{0.35} & \sd{1.07} & \sd{0.12} & \sd{0.20} & \sd{0.23} & \sd{0.67}\\[2pt]
TabDPT (full context) & 72.58 & 41.77 & 62.28 & 23.68 & 63.54 & 28.96 & 77.36 & 52.98 & 75.39 & 57.44 & 70.23 & 40.97\\*[-1.2pt]
 & \sd{1.46} & \sd{4.91} & \sd{2.03} & \sd{4.62} & \sd{2.63} & \sd{7.04} & \sd{1.09} & \sd{5.21} & \sd{0.28} & \sd{1.06} & \sd{1.50} & \sd{4.57}\\[2pt]
TuneTables & 68.98 & 34.03 & 64.61 & 27.84 & 60.84 & 21.53 & 77.64 & 53.82 & 71.51 & 46.10 & 68.72 & 36.66\\*[-1.2pt]
 & \sd{1.31} & \sd{2.61} & \sd{2.28} & \sd{4.18} & \sd{2.42} & \sd{6.60} & \sd{1.28} & \sd{2.99} & \sd{2.07} & \sd{5.42} & \sd{1.87} & \sd{4.36}\\[2pt]
MixturePFN & 74.17 & 46.00 & 70.32 & 38.24 & 64.16 & 30.55 & 78.60 & 55.82 & 75.54 & 57.51 & 72.56 & 45.62\\*[-1.2pt]
 & \sd{0.59} & \sd{1.67} & \sd{0.84} & \sd{2.03} & \sd{0.17} & \sd{0.37} & \sd{1.62} & \sd{4.95} & \sd{0.28} & \sd{0.63} & \sd{0.70} & \sd{1.93}\\[2pt]
LoCalPFN (retrieval) & 75.61 & 49.34 & 71.75 & 40.86 & 64.45 & 31.13 & 79.24 & 58.11 & 75.72 & 57.57 & 73.35 & 47.40\\*[-1.2pt]
 & \sd{0.14} & \sd{0.39} & \sd{0.17} & \sd{0.42} & \sd{0.15} & \sd{0.14} & \sd{0.38} & \sd{1.11} & \sd{0.10} & \sd{0.29} & \sd{0.19} & \sd{0.47}\\[2pt]
BETA & 74.70 & 47.36 & 66.26 & 30.51 & 64.46 & 30.93 & 78.44 & 56.84 & 75.03 & 57.14 & 71.78 & 44.56\\*[-1.2pt]
 & \sd{0.35} & \sd{0.87} & \sd{1.28} & \sd{2.49} & \sd{0.12} & \sd{0.72} & \sd{0.54} & \sd{2.08} & \sd{0.05} & \sd{0.25} & \sd{0.47} & \sd{1.28}\\[2pt]
DistPFN & 75.02 & 48.37 & 72.63 & 42.04 & 64.61 & 30.71 & 78.00 & 54.63 & 75.76 & 57.26 & 73.20 & 46.60\\*[-1.2pt]
 & \sd{0.17} & \sd{0.59} & \sd{0.10} & \sd{0.19} & \sd{0.05} & \sd{0.05} & \sd{0.13} & \sd{0.60} & \sd{0.13} & \sd{0.44} & \sd{0.12} & \sd{0.37}\\[2pt]
\addlinespace[1pt]
\midrule
\addlinespace[1pt]
\multicolumn{13}{@{}l}{\emph{Robust methods}}\\*
CVaR-DRO$^{\dagger}$ & 71.95 & 49.02 & 68.93 & 38.71 & 62.43 & 34.60 & 76.65 & 51.26 & 65.55 & 54.51 & 69.10 & 45.62\\*[-1.2pt]
 & \sd{0.68} & \sd{2.63} & \sd{1.52} & \sd{3.59} & \sd{0.33} & \sd{2.44} & \sd{0.31} & \sd{2.01} & \sd{2.00} & \sd{5.93} & \sd{0.97} & \sd{3.32} \\[2pt]
$\chi^2$-DRO$^{\dagger}$ & 71.85 & 50.09 & 70.88 & 40.04 & 62.03 & 34.16 & 76.95 & 52.40 & 68.27 & 62.43 & 70.00 & 47.82\\*[-1.2pt]
 & \sd{0.73} & \sd{1.78} & \sd{0.81} & \sd{3.56} & \sd{0.19} & \sd{2.12} & \sd{1.72} & \sd{3.60} & \sd{0.99} & \sd{1.08} & \sd{0.89} & \sd{2.43} \\[2pt]
KL-DRO$^{\dagger}$ & 74.33 & 49.17 & 69.95 & 39.89 & 59.03 & 19.39 & 79.53 & 59.74 & 62.30 & 47.93 & 69.03 & 43.22\\*[-1.2pt]
 & \sd{0.14} & \sd{0.90} & \sd{0.60} & \sd{1.75} & \sd{3.25} & \sd{5.10} & \sd{1.08} & \sd{3.98} & \sd{8.96} & \sd{23.68} & \sd{2.81} & \sd{7.08} \\[2pt]
EIIL$^{\dagger}$ & 69.37 & 38.97 & 61.85 & 21.69 & 65.05 & 28.91 & 74.82 & 46.18 & 69.52 & 58.00 & 68.12 & 38.75\\*[-1.2pt]
 & \sd{2.50} & \sd{8.15} & \sd{2.49} & \sd{6.91} & \sd{1.06} & \sd{7.55} & \sd{7.04} & \sd{15.02} & \sd{0.07} & \sd{4.43} & \sd{2.63} & \sd{8.41} \\[2pt]
JTT$^{\dagger}$ & 71.46 & 49.93 & 68.78 & 37.77 & 62.47 & 35.51 & 78.10 & 52.59 & 67.47 & 60.01 & 69.66 & 47.16\\*[-1.2pt]
 & \sd{1.89} & \sd{3.65} & \sd{0.67} & \sd{1.50} & \sd{0.90} & \sd{1.58} & \sd{0.75} & \sd{3.35} & \sd{0.42} & \sd{2.60} & \sd{0.93} & \sd{2.54} \\[2pt]
FAM$^{\dagger}$ & 72.85 & 49.59 & 71.03 & 41.09 & 62.50 & 37.12 & 76.60 & 53.51 & 67.90 & 62.24 & 70.18 & 48.71\\*[-1.2pt]
 & \sd{1.65} & \sd{1.95} & \sd{3.89} & \sd{5.25} & \sd{0.19} & \sd{2.07} & \sd{0.09} & \sd{2.05} & \sd{0.75} & \sd{0.63} & \sd{1.31} & \sd{2.39} \\[2pt]
SRDO$^{\dagger}$ & 71.27 & 46.44 & 66.34 & 32.08 & 62.38 & 33.34 & 76.24 & 53.11 & 64.57 & 58.18 & 68.16 & 44.63\\*[-1.2pt]
 & \sd{2.43} & \sd{6.82} & \sd{2.55} & \sd{5.44} & \sd{8.06} & \sd{10.94} & \sd{0.35} & \sd{2.31} & \sd{2.07} & \sd{2.50} & \sd{3.09} & \sd{5.60} \\[2pt]
LSR$^{\dagger}$ & 74.33 & 54.79 & 69.50 & 41.28 & 62.62 & 38.78 & 79.68 & 60.73 & 67.85 & 63.14 & 70.80 & 51.74\\*[-1.2pt]
 & \sd{0.28} & \sd{2.01} & \sd{0.39} & \sd{1.56} & \sd{0.12} & \sd{1.22} & \sd{1.44} & \sd{3.39} & \sd{0.16} & \sd{2.02} & \sd{0.48} & \sd{2.04} \\[2pt]
GEORGE ($k$-means) & 77.48 & 60.18 & 81.72 & 73.32 & 62.73 & 56.60 & 77.47 & 71.56 & 64.03 & 52.82 & 72.69 & 62.89\\*[-1.2pt]
 & \sd{0.42} & \sd{2.37} & \sd{0.54} & \sd{2.36} & \sd{2.89} & \sd{3.62} & \sd{10.40} & \sd{14.47} & \sd{2.31} & \sd{6.26} & \sd{3.31} & \sd{5.82}\\[2pt]
GEORGE (GMM) & 76.67 & 57.68 & 70.44 & 58.98 & 64.95 & 59.42 & 75.14 & 69.90 & 75.30 & \underline{71.95} & 72.50 & 63.59\\*[-1.2pt]
 & \sd{0.37} & \sd{3.46} & \sd{3.73} & \sd{4.11} & \sd{1.24} & \sd{1.80} & \sd{6.57} & \sd{6.98} & \sd{0.34} & \sd{0.63} & \sd{2.45} & \sd{3.40}\\[2pt]
BPA & 64.21 & 54.34 & 65.52 & 54.04 & 57.45 & 51.97 & 62.88 & 50.18 & 59.96 & 50.19 & 62.00 & 52.14\\*[-1.2pt]
 & \sd{2.69} & \sd{4.11} & \sd{6.83} & \sd{6.41} & \sd{1.05} & \sd{1.95} & \sd{1.96} & \sd{3.19} & \sd{1.32} & \sd{3.25} & \sd{2.77} & \sd{3.78}\\[2pt]
SPARE & 69.81 & 46.64 & 76.34 & 57.52 & 57.68 & 34.70 & 79.93 & 69.67 & 66.08 & 51.24 & 69.97 & 51.95\\*[-1.2pt]
 & \sd{0.84} & \sd{2.84} & \sd{1.41} & \sd{3.99} & \sd{5.48} & \sd{15.51} & \sd{0.87} & \sd{5.36} & \sd{8.24} & \sd{9.91} & \sd{3.37} & \sd{7.52}\\[2pt]
\addlinespace[1pt]
\midrule
\addlinespace[1pt]
\multicolumn{13}{@{}l}{\emph{TFM adaptation}}\\*
TabPFN & 75.33 & 49.17 & 73.07 & 42.96 & 64.77 & 31.09 & 79.59 & 58.29 & 76.01 & 58.19 & 73.75 & 47.94\\*[-1.2pt]
 & \sd{0.05} & \sd{0.28} & \sd{0.31} & \sd{0.76} & \sd{0.04} & \sd{0.04} & \sd{0.32} & \sd{0.52} & \sd{0.25} & \sd{0.53} & \sd{0.19} & \sd{0.43}\\[2pt]
TabPFN (fine-tuned) & 76.33 & 50.80 & 73.10 & 42.99 & 65.03 & 31.71 & 78.88 & 56.25 & 76.11 & 58.23 & 73.89 & 48.00\\*[-1.2pt]
 & \sd{0.35} & \sd{0.47} & \sd{0.02} & \sd{0.21} & \sd{0.40} & \sd{2.04} & \sd{0.76} & \sd{1.98} & \sd{0.27} & \sd{0.66} & \sd{0.36} & \sd{1.07}\\[2pt]
TabPFN (class-balanced) & 75.92 & 50.98 & 73.22 & 43.42 & 65.27 & 32.76 & 77.79 & 53.24 & 76.48 & 61.37 & 73.74 & 48.35\\*[-1.2pt]
 & \sd{0.40} & \sd{1.14} & \sd{0.51} & \sd{1.42} & \sd{0.60} & \sd{1.73} & \sd{2.89} & \sd{8.24} & \sd{0.96} & \sd{5.68} & \sd{1.07} & \sd{3.64}\\[2pt]
\textbf{DR-TFM (GR; $k$-means)} & 79.64 & \textbf{67.90} & \underline{85.01} & \textbf{80.14} & \underline{69.42} & \underline{60.37} & \underline{85.79} & \underline{81.66} & 76.69 & \textbf{72.79} & \textbf{79.31} & \textbf{72.57}\\*[-1.2pt]
 & \sd{0.20} & \sd{1.61} & \sd{0.36} & \sd{0.53} & \sd{0.62} & \sd{1.70} & \sd{0.32} & \sd{0.50} & \sd{0.19} & \sd{0.62} & \sd{0.34} & \sd{0.99}\\[2pt]
\textbf{DR-TFM (GR; GMM)} & 79.10 & 65.58 & 84.74 & \underline{79.12} & \textbf{69.80} & \textbf{60.45} & \textbf{85.86} & \textbf{82.40} & 76.45 & 71.48 & \underline{79.19} & \underline{71.81}\\*[-1.2pt]
 & \sd{0.16} & \sd{1.47} & \sd{0.40} & \sd{1.03} & \sd{0.37} & \sd{1.66} & \sd{0.32} & \sd{1.77} & \sd{0.28} & \sd{1.06} & \sd{0.31} & \sd{1.40}\\[2pt]
\textbf{DR-TFM (MS-C)} & \textbf{79.78} & \underline{67.41} & \textbf{85.02} & 78.41 & 69.18 & 57.86 & 84.48 & 77.09 & \underline{77.13} & 65.80 & 79.12 & 69.31\\*[-1.2pt]
 & \sd{0.24} & \sd{0.64} & \sd{1.33} & \sd{1.49} & \sd{0.52} & \sd{4.76} & \sd{0.35} & \sd{0.66} & \sd{0.13} & \sd{0.70} & \sd{0.51} & \sd{1.65}\\[2pt]
\textbf{DR-TFM (MS-R)} & \underline{79.70} & 66.87 & 82.99 & 75.71 & 69.05 & 57.72 & 84.97 & 78.35 & \textbf{77.37} & 67.91 & 78.82 & 69.31\\*[-1.2pt]
 & \sd{0.05} & \sd{0.32} & \sd{1.52} & \sd{3.24} & \sd{0.49} & \sd{5.56} & \sd{0.73} & \sd{1.21} & \sd{0.52} & \sd{1.79} & \sd{0.66} & \sd{2.42}\\[2pt]
\end{longtable}
\endgroup

\begin{table}[H]
\caption{Results on the five tabular benchmarks for methods using true group labels.
Avg.\ denotes the average of each metric across datasets.
Values in parentheses are standard deviations over three seeds.
Bold method names indicate our variants.}
\label{appendix:tab1-known}
\centering
\scriptsize
\setlength{\tabcolsep}{3.4pt}
\begin{tabular*}{\textwidth}{@{\extracolsep{\fill}}lcccccccccc|cc}
\toprule
\multirow{2}{*}{Methods} & \multicolumn{2}{c}{Adult} & \multicolumn{2}{c}{Bank} & \multicolumn{2}{c}{Default} & \multicolumn{2}{c}{Shoppers} & \multicolumn{2}{c|}{Taxi} & \multicolumn{2}{c}{Avg.} \\
\cmidrule(lr){2-3}\cmidrule(lr){4-5}\cmidrule(lr){6-7}\cmidrule(lr){8-9}\cmidrule(lr){10-11}\cmidrule(lr){12-13}
 & Mean & Worst & Mean & Worst & Mean & Worst & Mean & Worst & Mean & Worst & Mean & Worst \\
\midrule
\multicolumn{13}{@{}l}{\emph{Robust methods}}\\
Fair-TabICL (group-balanced) & \textbf{80.11} & \underline{72.34} & \textbf{86.58} & \underline{78.54} & \textbf{69.86} & 57.73 & \textbf{86.24} & \textbf{83.38} & \textbf{77.45} & 68.03 & \textbf{80.05} & \underline{72.00}\\[-1.2pt]
 & \sd{0.23} & \sd{0.91} & \sd{0.09} & \sd{0.10} & \sd{0.23} & \sd{0.71} & \sd{0.85} & \sd{0.91} & \sd{0.32} & \sd{0.49} & \sd{0.34} & \sd{0.62}\\[2pt]
Fair-TabICL (uncertainty) & 61.80 & 21.35 & 71.52 & 40.51 & 64.93 & 31.18 & 77.80 & 52.81 & 75.94 & 57.61 & 70.40 & 40.69\\[-1.2pt]
 & \sd{0.74} & \sd{1.65} & \sd{0.14} & \sd{0.32} & \sd{0.12} & \sd{0.13} & \sd{0.43} & \sd{1.57} & \sd{0.14} & \sd{0.52} & \sd{0.31} & \sd{0.84}\\[2pt]
GroupDRO & 78.16 & \textbf{73.90} & 81.56 & 75.47 & 67.16 & \underline{61.82} & 80.30 & 76.28 & 75.21 & \textbf{71.47} & 76.48 & 71.79\\[-1.2pt]
 & \sd{0.22} & \sd{0.69} & \sd{0.57} & \sd{1.25} & \sd{0.45} & \sd{1.09} & \sd{1.26} & \sd{2.80} & \sd{0.04} & \sd{1.26} & \sd{0.51} & \sd{1.42}\\[2pt]
\addlinespace[1pt]
\midrule
\addlinespace[1pt]
\multicolumn{13}{@{}l}{\emph{TFM adaptation}}\\
\textbf{DR-TFM (GR)} & 78.36 & 72.26 & 86.12 & \textbf{81.94} & 69.58 & \textbf{65.30} & \underline{85.21} & \underline{80.52} & \underline{77.03} & 68.28 & 79.26 & \textbf{73.66}\\[-1.2pt]
 & \sd{0.52} & \sd{0.79} & \sd{0.93} & \sd{0.34} & \sd{0.26} & \sd{1.19} & \sd{0.39} & \sd{2.17} & \sd{0.39} & \sd{2.32} & \sd{0.50} & \sd{1.36}\\[2pt]
\textbf{DR-TFM (MS)} & \underline{78.86} & 72.32 & \underline{86.56} & 77.72 & \underline{69.85} & 54.71 & 85.09 & 79.71 & \textbf{77.45} & \underline{68.48} & \underline{79.56} & 70.59\\[-1.2pt]
 & \sd{1.55} & \sd{1.08} & \sd{0.39} & \sd{0.63} & \sd{0.39} & \sd{1.70} & \sd{0.99} & \sd{3.23} & \sd{0.27} & \sd{0.96} & \sd{0.72} & \sd{1.52}\\[2pt]
\bottomrule
\end{tabular*}
\end{table}

\begingroup
\scriptsize
\setlength{\tabcolsep}{1.4pt}
\setlength{\LTleft}{\fill}
\setlength{\LTright}{\fill}
\setlength{\LTcapwidth}{\textwidth}
\begin{longtable}{@{\extracolsep{\fill}}lcccccc|cc|cccccc|cc}
\caption{ACS Income results for methods without true group labels.
The left block reports within-state results and the right block reports transfers between states.
Avg.\ denotes the average of each metric across the three settings in its block.
Values in parentheses are standard deviations over three seeds.
Accuracy results marked $^{\dagger}$, including standard deviations, are taken from \citet{tong2025latent}.
Bold method names indicate our variants.}\label{appendix:tab2-free}\\
\toprule
\multirow{2}{*}{Methods} & \multicolumn{2}{c}{AZ} & \multicolumn{2}{c}{MA} & \multicolumn{2}{c|}{MI} & \multicolumn{2}{c|}{Avg.} & \multicolumn{2}{c}{AZ $\rightarrow$ MA} & \multicolumn{2}{c}{MA $\rightarrow$ MI} & \multicolumn{2}{c|}{MI $\rightarrow$ AZ} & \multicolumn{2}{c}{Avg.} \\
\cmidrule(lr){2-3}\cmidrule(lr){4-5}\cmidrule(lr){6-7}\cmidrule(lr){8-9}\cmidrule(lr){10-11}\cmidrule(lr){12-13}\cmidrule(lr){14-15}\cmidrule(lr){16-17}
 & Mean & Worst & Mean & Worst & Mean & Worst & Mean & Worst & Mean & Worst & Mean & Worst & Mean & Worst & Mean & Worst \\
\midrule
\endfirsthead
\caption[]{ACS Income, methods without true group labels (continued).}\\
\toprule
\multirow{2}{*}{Methods} & \multicolumn{2}{c}{AZ} & \multicolumn{2}{c}{MA} & \multicolumn{2}{c|}{MI} & \multicolumn{2}{c|}{Avg.} & \multicolumn{2}{c}{AZ $\rightarrow$ MA} & \multicolumn{2}{c}{MA $\rightarrow$ MI} & \multicolumn{2}{c|}{MI $\rightarrow$ AZ} & \multicolumn{2}{c}{Avg.} \\
\cmidrule(lr){2-3}\cmidrule(lr){4-5}\cmidrule(lr){6-7}\cmidrule(lr){8-9}\cmidrule(lr){10-11}\cmidrule(lr){12-13}\cmidrule(lr){14-15}\cmidrule(lr){16-17}
 & Mean & Worst & Mean & Worst & Mean & Worst & Mean & Worst & Mean & Worst & Mean & Worst & Mean & Worst & Mean & Worst \\
\midrule
\endhead
\midrule
\multicolumn{17}{r}{\footnotesize Continued on next page}\\
\endfoot
\bottomrule
\endlastfoot
\multicolumn{17}{@{}l}{\emph{Standard methods}}\\*
ERM-MLP & 77.20 & 59.42 & 80.82 & 73.82 & 75.69 & 57.99 & 77.90 & 63.74 & 76.82 & 60.23 & 79.13 & 71.34 & 75.56 & 56.31 & 77.17 & 62.63\\*[-1.2pt]
 & \sd{0.54} & \sd{1.59} & \sd{0.07} & \sd{0.61} & \sd{0.32} & \sd{1.88} & \sd{0.31} & \sd{1.36} & \sd{0.62} & \sd{2.91} & \sd{0.09} & \sd{0.66} & \sd{0.30} & \sd{2.40} & \sd{0.34} & \sd{1.99} \\[2pt]
XGBoost & 77.64 & 58.99 & 82.26 & 76.38 & 77.50 & 62.77 & 79.13 & 66.05 & 77.96 & 61.99 & 79.67 & 72.01 & 76.82 & 58.99 & 78.15 & 64.33\\*[-1.2pt]
 & \sd{0.27} & \sd{0.76} & \sd{0.25} & \sd{0.12} & \sd{0.11} & \sd{0.40} & \sd{0.21} & \sd{0.43} & \sd{0.14} & \sd{0.36} & \sd{0.06} & \sd{0.13} & \sd{0.10} & \sd{0.30} & \sd{0.10} & \sd{0.27} \\[2pt]
CatBoost & 77.35 & 58.26 & 81.97 & 75.59 & 77.27 & 62.43 & 78.86 & 65.43 & 77.66 & 61.15 & 79.60 & 72.31 & 76.68 & 58.71 & 77.98 & 64.06\\*[-1.2pt]
 & \sd{0.20} & \sd{0.60} & \sd{0.20} & \sd{0.19} & \sd{0.12} & \sd{0.13} & \sd{0.17} & \sd{0.30} & \sd{0.14} & \sd{0.25} & \sd{0.05} & \sd{0.42} & \sd{0.09} & \sd{0.04} & \sd{0.09} & \sd{0.23} \\[2pt]
\addlinespace[1pt]
\midrule
\addlinespace[1pt]
\multicolumn{17}{@{}l}{\emph{TFMs}}\\*
TabICL & 78.14 & 59.46 & 82.89 & 76.50 & 77.84 & 64.30 & 79.62 & 66.75 & 78.03 & 62.08 & \textbf{79.73} & 71.78 & 76.98 & 59.69 & 78.25 & 64.52\\*[-1.2pt]
 & \sd{0.09} & \sd{0.21} & \sd{0.03} & \sd{0.33} & \sd{0.23} & \sd{0.38} & \sd{0.12} & \sd{0.31} & \sd{0.07} & \sd{0.19} & \sd{0.06} & \sd{0.17} & \sd{0.05} & \sd{0.22} & \sd{0.06} & \sd{0.19} \\[2pt]
TabDPT (retrieval) & 77.25 & 57.85 & 81.39 & 75.43 & 76.80 & 61.13 & 78.48 & 64.80 & 76.92 & 59.51 & 78.89 & 70.70 & 75.21 & 54.67 & 77.01 & 61.63\\*[-1.2pt]
 & \sd{0.33} & \sd{0.58} & \sd{0.25} & \sd{0.24} & \sd{0.20} & \sd{0.52} & \sd{0.26} & \sd{0.44} & \sd{0.03} & \sd{0.20} & \sd{0.04} & \sd{0.19} & \sd{0.06} & \sd{0.22} & \sd{0.04} & \sd{0.21} \\[2pt]
TabDPT (full context) & 74.13 & 50.99 & 79.51 & 72.34 & 74.95 & 58.01 & 76.20 & 60.44 & 75.89 & 58.09 & 77.70 & 71.51 & 73.51 & 51.16 & 75.70 & 60.25\\*[-1.2pt]
 & \sd{0.41} & \sd{1.71} & \sd{0.85} & \sd{3.16} & \sd{0.58} & \sd{3.56} & \sd{0.61} & \sd{2.81} & \sd{0.73} & \sd{2.07} & \sd{0.33} & \sd{2.02} & \sd{0.47} & \sd{2.25} & \sd{0.51} & \sd{2.11} \\[2pt]
TuneTables & 74.34 & 52.56 & 78.13 & 70.04 & 72.22 & 49.83 & 74.90 & 57.48 & 72.98 & 46.66 & 77.35 & 67.66 & 72.42 & 46.06 & 74.25 & 53.46\\*[-1.2pt]
 & \sd{1.02} & \sd{3.03} & \sd{0.29} & \sd{1.74} & \sd{1.95} & \sd{4.90} & \sd{1.09} & \sd{3.22} & \sd{0.03} & \sd{1.77} & \sd{0.27} & \sd{2.57} & \sd{0.51} & \sd{1.21} & \sd{0.27} & \sd{1.85} \\[2pt]
MixturePFN & 76.62 & 56.62 & 81.89 & 76.01 & 77.77 & 64.52 & 78.76 & 65.72 & 77.18 & 60.52 & 79.32 & 72.02 & 76.07 & 58.36 & 77.52 & 63.63\\*[-1.2pt]
 & \sd{0.61} & \sd{2.21} & \sd{0.19} & \sd{0.68} & \sd{0.18} & \sd{0.85} & \sd{0.33} & \sd{1.25} & \sd{0.22} & \sd{0.62} & \sd{0.03} & \sd{0.42} & \sd{0.07} & \sd{0.23} & \sd{0.11} & \sd{0.42} \\[2pt]
LoCalPFN (retrieval) & 78.05 & 60.44 & 82.58 & \textbf{77.31} & 77.42 & 63.22 & 79.35 & 66.99 & 77.48 & 60.85 & 79.63 & \textbf{72.64} & 76.42 & 58.32 & 77.84 & 63.94\\*[-1.2pt]
 & \sd{0.26} & \sd{0.11} & \sd{0.07} & \sd{0.14} & \sd{0.09} & \sd{0.53} & \sd{0.14} & \sd{0.26} & \sd{0.07} & \sd{0.12} & \sd{0.05} & \sd{0.10} & \sd{0.07} & \sd{0.20} & \sd{0.07} & \sd{0.14} \\[2pt]
BETA & 76.73 & 58.08 & 80.67 & 73.25 & 76.80 & 62.12 & 78.07 & 64.49 & 77.25 & 60.70 & 79.39 & \underline{72.55} & 76.30 & 58.35 & 77.65 & 63.87\\*[-1.2pt]
 & \sd{0.17} & \sd{0.87} & \sd{0.22} & \sd{0.75} & \sd{0.21} & \sd{1.29} & \sd{0.20} & \sd{0.97} & \sd{0.49} & \sd{1.65} & \sd{0.04} & \sd{0.83} & \sd{0.03} & \sd{0.42} & \sd{0.19} & \sd{0.97} \\[2pt]
DistPFN & 78.09 & 59.55 & \underline{82.98} & 76.96 & 77.82 & 63.78 & 79.63 & 66.76 & 78.76 & 65.16 & 79.43 & 71.58 & 76.70 & 58.56 & 78.29 & 65.10\\*[-1.2pt]
 & \sd{0.19} & \sd{0.26} & \sd{0.11} & \sd{0.27} & \sd{0.07} & \sd{0.27} & \sd{0.13} & \sd{0.27} & \sd{0.03} & \sd{0.02} & \sd{0.03} & \sd{0.13} & \sd{0.11} & \sd{0.23} & \sd{0.05} & \sd{0.13}\\[2pt]
\addlinespace[1pt]
\midrule
\addlinespace[1pt]
\multicolumn{17}{@{}l}{\emph{Robust methods}}\\*
CVaR-DRO$^{\dagger}$ & 77.00 & 64.41 & 77.85 & 72.85 & 73.48 & 61.88 & 76.11 & 66.38 & 74.10 & 59.42 & 74.85 & 68.57 & 71.75 & 54.67 & 73.57 & 60.89\\*[-1.2pt]
 & \sd{0.14} & \sd{3.70} & \sd{1.27} & \sd{1.88} & \sd{0.74} & \sd{4.97} & \sd{0.72} & \sd{3.52} & \sd{0.35} & \sd{2.51} & \sd{0.14} & \sd{0.65} & \sd{0.28} & \sd{2.55} & \sd{0.26} & \sd{1.90} \\[2pt]
$\chi^2$-DRO$^{\dagger}$ & 76.95 & 64.07 & 77.55 & 71.82 & 73.13 & 59.03 & 75.88 & 64.97 & 74.53 & 61.18 & 74.23 & 68.53 & 71.33 & 53.05 & 73.36 & 60.92\\*[-1.2pt]
 & \sd{0.14} & \sd{4.60} & \sd{0.92} & \sd{1.58} & \sd{1.38} & \sd{2.06} & \sd{0.81} & \sd{2.75} & \sd{0.25} & \sd{1.15} & \sd{0.18} & \sd{0.67} & \sd{0.60} & \sd{1.76} & \sd{0.34} & \sd{1.19} \\[2pt]
KL-DRO$^{\dagger}$ & 75.68 & 60.98 & 78.23 & 71.22 & 76.10 & 62.88 & 76.67 & 65.03 & 74.20 & 57.78 & 75.35 & 68.72 & 74.73 & 57.20 & 74.76 & 61.23\\*[-1.2pt]
 & \sd{1.23} & \sd{3.00} & \sd{0.32} & \sd{2.41} & \sd{0.07} & \sd{0.37} & \sd{0.54} & \sd{1.93} & \sd{1.20} & \sd{2.91} & \sd{0.78} & \sd{1.71} & \sd{0.25} & \sd{0.97} & \sd{0.74} & \sd{1.86} \\[2pt]
EIIL$^{\dagger}$ & 74.98 & 55.22 & 78.18 & 68.42 & 75.60 & 63.77 & 76.25 & 62.47 & 75.35 & 56.80 & 76.92 & 65.62 & 74.78 & 57.32 & 75.68 & 59.91\\*[-1.2pt]
 & \sd{1.80} & \sd{6.65} & \sd{0.67} & \sd{5.86} & \sd{0.49} & \sd{3.19} & \sd{0.99} & \sd{5.23} & \sd{1.70} & \sd{5.52} & \sd{1.59} & \sd{5.88} & \sd{1.51} & \sd{4.95} & \sd{1.60} & \sd{5.45} \\[2pt]
JTT$^{\dagger}$ & 75.70 & 57.73 & 77.48 & 68.48 & 74.48 & 60.63 & 75.89 & 62.28 & 73.90 & 55.17 & 74.00 & 66.57 & 72.65 & 55.17 & 73.52 & 58.97\\*[-1.2pt]
 & \sd{1.06} & \sd{2.49} & \sd{0.25} & \sd{2.70} & \sd{0.53} & \sd{1.84} & \sd{0.61} & \sd{2.34} & \sd{0.78} & \sd{3.14} & \sd{0.28} & \sd{2.33} & \sd{0.35} & \sd{0.62} & \sd{0.47} & \sd{2.03} \\[2pt]
FAM$^{\dagger}$ & 75.00 & 57.82 & 78.73 & 70.12 & 74.38 & 59.45 & 76.04 & 62.46 & 72.93 & 54.60 & 74.68 & 66.35 & 73.40 & 54.53 & 73.67 & 58.49\\*[-1.2pt]
 & \sd{0.57} & \sd{2.17} & \sd{0.03} & \sd{3.33} & \sd{2.09} & \sd{6.29} & \sd{0.90} & \sd{3.93} & \sd{0.53} & \sd{2.45} & \sd{0.11} & \sd{2.44} & \sd{1.27} & \sd{5.24} & \sd{0.64} & \sd{3.38} \\[2pt]
SRDO$^{\dagger}$ & 75.17 & 60.03 & 77.85 & 69.52 & 73.30 & 54.63 & 75.44 & 61.39 & 74.00 & 58.37 & 74.78 & 61.77 & 73.83 & 60.05 & 74.20 & 60.06\\*[-1.2pt]
 & \sd{2.22} & \sd{10.23} & \sd{1.01} & \sd{2.49} & \sd{1.22} & \sd{5.75} & \sd{1.48} & \sd{6.16} & \sd{1.83} & \sd{8.57} & \sd{0.49} & \sd{4.26} & \sd{0.35} & \sd{0.96} & \sd{0.89} & \sd{4.60} \\[2pt]
LSR$^{\dagger}$ & 76.28 & 66.42 & 78.35 & 73.33 & 75.53 & 67.10 & 76.72 & 68.95 & 75.00 & 62.85 & 74.75 & 68.75 & 74.73 & 64.10 & 74.83 & 65.23\\*[-1.2pt]
 & \sd{0.11} & \sd{0.89} & \sd{0.14} & \sd{1.48} & \sd{1.77} & \sd{1.61} & \sd{0.67} & \sd{1.33} & \sd{0.07} & \sd{0.40} & \sd{0.35} & \sd{0.30} & \sd{0.35} & \sd{0.35} & \sd{0.26} & \sd{0.35} \\[2pt]
GEORGE ($k$-means) & 74.76 & 58.70 & 76.65 & 68.82 & 72.82 & 66.52 & 74.74 & 64.68 & 73.25 & 61.69 & 73.22 & 64.80 & 74.21 & 62.10 & 73.56 & 62.86\\*[-1.2pt]
 & \sd{2.95} & \sd{14.27} & \sd{1.68} & \sd{5.90} & \sd{1.78} & \sd{1.05} & \sd{2.14} & \sd{7.07} & \sd{4.66} & \sd{3.06} & \sd{1.18} & \sd{4.67} & \sd{2.42} & \sd{5.66} & \sd{2.75} & \sd{4.46} \\[2pt]
GEORGE (GMM) & 65.41 & 61.03 & 61.45 & 55.58 & 65.06 & 59.28 & 63.97 & 58.63 & 66.73 & 61.40 & 62.24 & 58.05 & 61.35 & 55.30 & 63.44 & 58.25\\*[-1.2pt]
 & \sd{5.34} & \sd{5.03} & \sd{1.48} & \sd{2.59} & \sd{0.83} & \sd{0.87} & \sd{2.55} & \sd{2.83} & \sd{2.57} & \sd{3.87} & \sd{0.36} & \sd{1.89} & \sd{4.65} & \sd{8.68} & \sd{2.53} & \sd{4.81} \\[2pt]
BPA & 63.26 & 54.44 & 65.78 & 58.90 & 63.42 & 58.18 & 64.15 & 57.17 & 65.65 & 59.36 & 62.72 & 57.64 & 63.57 & 58.96 & 63.98 & 58.65\\*[-1.2pt]
 & \sd{1.32} & \sd{4.12} & \sd{1.46} & \sd{1.97} & \sd{0.68} & \sd{1.31} & \sd{1.15} & \sd{2.47} & \sd{0.57} & \sd{2.86} & \sd{1.20} & \sd{1.08} & \sd{0.16} & \sd{0.71} & \sd{0.65} & \sd{1.55}\\[2pt]
SPARE & 74.35 & 69.94 & 77.51 & 70.17 & 71.05 & 61.67 & 74.31 & 67.26 & 74.23 & 65.44 & 73.23 & 65.67 & 70.27 & 58.96 & 72.58 & 63.36\\*[-1.2pt]
 & \sd{2.61} & \sd{4.73} & \sd{1.49} & \sd{1.32} & \sd{1.82} & \sd{5.26} & \sd{1.97} & \sd{3.77} & \sd{0.89} & \sd{2.30} & \sd{0.97} & \sd{1.74} & \sd{1.20} & \sd{5.30} & \sd{1.02} & \sd{3.11} \\[2pt]
\addlinespace[1pt]
\midrule
\addlinespace[1pt]
\multicolumn{17}{@{}l}{\emph{TFM adaptation}}\\*
TabPFN & 78.19 & 59.87 & \textbf{83.01} & 77.02 & 77.72 & 63.50 & 79.64 & 66.80 & 77.95 & 61.57 & \underline{79.69} & 71.86 & 76.81 & 59.08 & 78.15 & 64.17\\*[-1.2pt]
 & \sd{0.32} & \sd{0.50} & \sd{0.12} & \sd{0.27} & \sd{0.03} & \sd{0.20} & \sd{0.16} & \sd{0.32} & \sd{0.07} & \sd{0.18} & \sd{0.08} & \sd{0.25} & \sd{0.05} & \sd{0.12} & \sd{0.07} & \sd{0.18}\\[2pt]
TabPFN (fine-tuned) & 78.23 & 60.35 & 82.95 & \underline{77.03} & 78.23 & 65.02 & 79.80 & 67.47 & 77.95 & 62.05 & 79.62 & 72.20 & 76.76 & 58.77 & 78.11 & 64.34\\*[-1.2pt]
 & \sd{0.20} & \sd{0.39} & \sd{0.11} & \sd{0.33} & \sd{0.24} & \sd{0.56} & \sd{0.18} & \sd{0.43} & \sd{0.10} & \sd{0.68} & \sd{0.07} & \sd{0.21} & \sd{0.08} & \sd{0.63} & \sd{0.08} & \sd{0.50} \\[2pt]
TabPFN (class-balanced) & 78.52 & 61.44 & \underline{82.98} & 76.72 & 78.12 & 64.74 & 79.87 & 67.63 & 78.55 & 64.46 & 79.67 & 70.23 & 76.84 & 59.15 & 78.35 & 64.61\\*[-1.2pt]
 & \sd{0.35} & \sd{1.76} & \sd{0.09} & \sd{0.20} & \sd{0.55} & \sd{2.05} & \sd{0.33} & \sd{1.34} & \sd{1.13} & \sd{5.11} & \sd{0.04} & \sd{0.17} & \sd{0.10} & \sd{0.23} & \sd{0.42} & \sd{1.84}\\[2pt]
\textbf{DR-TFM (GR; $k$-means)} & 78.81 & \underline{70.36} & 79.96 & 70.93 & 78.87 & \textbf{71.90} & 79.21 & 71.06 & 78.43 & 69.70 & 78.73 & 66.43 & 78.46 & 68.84 & 78.54 & 68.32\\*[-1.2pt]
 & \sd{0.13} & \sd{1.00} & \sd{0.59} & \sd{1.31} & \sd{0.30} & \sd{0.86} & \sd{0.34} & \sd{1.06} & \sd{0.15} & \sd{0.72} & \sd{0.68} & \sd{0.34} & \sd{0.15} & \sd{0.59} & \sd{0.32} & \sd{0.55}\\[2pt]
\textbf{DR-TFM (GR; GMM)} & 79.22 & \textbf{71.29} & 81.51 & 73.97 & 78.74 & 71.25 & 79.82 & \textbf{72.17} & 78.88 & 70.08 & 79.10 & 66.99 & 78.34 & 68.97 & 78.77 & 68.68\\*[-1.2pt]
 & \sd{0.26} & \sd{1.61} & \sd{0.42} & \sd{0.90} & \sd{0.65} & \sd{1.46} & \sd{0.44} & \sd{1.33} & \sd{0.42} & \sd{1.01} & \sd{0.55} & \sd{1.32} & \sd{0.32} & \sd{0.85} & \sd{0.43} & \sd{1.06}\\[2pt]
\textbf{DR-TFM (MS-C)} & \textbf{79.46} & 69.30 & 81.24 & 72.75 & \textbf{79.30} & \underline{71.42} & \underline{80.00} & 71.16 & \underline{80.08} & \underline{73.75} & 78.82 & 70.53 & \textbf{78.55} & \underline{71.38} & \underline{79.15} & \textbf{71.89}\\*[-1.2pt]
 & \sd{0.36} & \sd{2.39} & \sd{1.00} & \sd{2.43} & \sd{0.22} & \sd{1.84} & \sd{0.53} & \sd{2.22} & \sd{0.05} & \sd{1.44} & \sd{0.48} & \sd{1.12} & \sd{0.09} & \sd{0.35} & \sd{0.21} & \sd{0.97}\\[2pt]
\textbf{DR-TFM (MS-R)} & \underline{79.43} & 70.31 & 81.91 & 73.94 & \underline{78.90} & 71.10 & \textbf{80.08} & \underline{71.78} & \textbf{80.15} & \textbf{74.93} & 79.51 & 66.86 & \underline{78.54} & \textbf{71.50} & \textbf{79.40} & \underline{71.10}\\*[-1.2pt]
 & \sd{0.16} & \sd{0.58} & \sd{0.15} & \sd{1.09} & \sd{0.12} & \sd{1.90} & \sd{0.14} & \sd{1.19} & \sd{0.17} & \sd{0.26} & \sd{0.09} & \sd{0.32} & \sd{0.06} & \sd{0.31} & \sd{0.11} & \sd{0.30}\\[2pt]
\end{longtable}
\endgroup

\begin{table}[H]
\caption{ACS Income results for methods using true group labels.
Avg.\ denotes the average of each metric across the three settings in its block.
Values in parentheses are standard deviations over three seeds.
Bold method names indicate our variants.}
\label{appendix:tab2-known}
\centering
\scriptsize
\setlength{\tabcolsep}{1.5pt}
\begin{tabular*}{\textwidth}{@{\extracolsep{\fill}}lcccccc|cc|cccccc|cc}
\toprule
\multirow{2}{*}{Methods} & \multicolumn{2}{c}{AZ} & \multicolumn{2}{c}{MA} & \multicolumn{2}{c|}{MI} & \multicolumn{2}{c|}{Avg.} & \multicolumn{2}{c}{AZ $\rightarrow$ MA} & \multicolumn{2}{c}{MA $\rightarrow$ MI} & \multicolumn{2}{c|}{MI $\rightarrow$ AZ} & \multicolumn{2}{c}{Avg.} \\
\cmidrule(lr){2-3}\cmidrule(lr){4-5}\cmidrule(lr){6-7}\cmidrule(lr){8-9}\cmidrule(lr){10-11}\cmidrule(lr){12-13}\cmidrule(lr){14-15}\cmidrule(lr){16-17}
 & Mean & Worst & Mean & Worst & Mean & Worst & Mean & Worst & Mean & Worst & Mean & Worst & Mean & Worst & Mean & Worst \\
\midrule
\multicolumn{17}{@{}l}{\emph{Robust methods}}\\
Fair-TabICL (group-balanced) & 79.43 & 74.38 & 81.42 & 78.68 & \textbf{78.96} & 75.19 & \textbf{79.94} & 76.09 & \underline{80.10} & \underline{78.18} & \underline{79.56} & 75.01 & \textbf{78.87} & 75.41 & \textbf{79.51} & 76.20 \\[-1.2pt]
 & \sd{0.24} & \sd{0.09} & \sd{0.03} & \sd{0.82} & \sd{0.20} & \sd{1.21} & \sd{0.16} & \sd{0.71} & \sd{0.08} & \sd{0.03} & \sd{0.06} & \sd{0.54} & \sd{0.15} & \sd{1.33} & \sd{0.10} & \sd{0.63} \\[2pt]
Fair-TabICL (uncertainty) & 77.72 & 58.40 & \textbf{82.36} & 77.46 & 76.71 & 62.07 & 78.93 & 65.98 & 77.41 & 60.40 & \textbf{79.59} & 73.40 & 76.24 & 57.21 & 77.75 & 63.67 \\[-1.2pt]
 & \sd{0.17} & \sd{0.57} & \sd{0.05} & \sd{0.21} & \sd{0.31} & \sd{0.61} & \sd{0.18} & \sd{0.47} & \sd{0.08} & \sd{0.39} & \sd{0.12} & \sd{0.25} & \sd{0.13} & \sd{0.13} & \sd{0.11} & \sd{0.26} \\[2pt]
GroupDRO & \textbf{79.74} & \textbf{78.39} & 81.39 & \underline{78.91} & 78.23 & \textbf{76.70} & 79.79 & \textbf{78.00} & 79.19 & 76.22 & 79.26 & 76.24 & 78.40 & \underline{75.50} & 78.95 & 75.99 \\[-1.2pt]
 & \sd{0.09} & \sd{0.20} & \sd{0.43} & \sd{0.62} & \sd{0.34} & \sd{0.38} & \sd{0.29} & \sd{0.40} & \sd{0.10} & \sd{0.26} & \sd{0.09} & \sd{0.30} & \sd{0.17} & \sd{0.53} & \sd{0.12} & \sd{0.37} \\[2pt]
\addlinespace[1pt]
\midrule
\addlinespace[1pt]
\multicolumn{17}{@{}l}{\emph{TFM adaptation}}\\
\textbf{DR-TFM (GR)} & 79.55 & \underline{76.36} & \underline{81.45} & \textbf{79.27} & \underline{78.64} & \underline{76.09} & \underline{79.88} & \underline{77.24} & 79.61 & 78.02 & 79.37 & \underline{76.36} & 78.50 & \textbf{76.25} & 79.16 & \textbf{76.88} \\[-1.2pt]
 & \sd{0.40} & \sd{0.91} & \sd{0.34} & \sd{0.49} & \sd{0.25} & \sd{0.80} & \sd{0.33} & \sd{0.73} & \sd{0.47} & \sd{1.54} & \sd{0.03} & \sd{0.93} & \sd{0.05} & \sd{0.37} & \sd{0.19} & \sd{0.95} \\[2pt]
\textbf{DR-TFM (MS)} & \underline{79.57} & 75.05 & 81.35 & 77.59 & 78.47 & 75.83 & 79.80 & 76.16 & \textbf{80.13} & \textbf{78.60} & 79.31 & \textbf{76.60} & \underline{78.76} & 74.77 & \underline{79.40} & \underline{76.66} \\[-1.2pt]
 & \sd{0.26} & \sd{1.67} & \sd{0.40} & \sd{0.57} & \sd{0.26} & \sd{0.44} & \sd{0.30} & \sd{0.89} & \sd{0.33} & \sd{0.31} & \sd{0.05} & \sd{0.64} & \sd{0.03} & \sd{0.95} & \sd{0.14} & \sd{0.63} \\[2pt]
\bottomrule
\end{tabular*}
\end{table}

\subsection{Computational cost}
\label{appendix:runtime}
All runtimes were measured on a server equipped with two Intel Xeon Gold 6230R
CPUs (2.10 GHz), using a single NVIDIA GeForce RTX 3090 GPU.
Each experiment was timed separately, with no other experiments running concurrently.

\begin{table}[H]
\centering
\caption{Runtime in seconds per reported result, averaged over three seeds.
Times include hyperparameter selection and all training or adaptation stages;
TabPFN requires no parameter updates.}
\label{tab:runtime}
\small
\begin{tabular}{l ccccc | c}
\toprule
Method & Adult & Bank & Default & Shoppers & Taxi & Avg. \\
\midrule
XGBoost             & 5.5 & 6.4 & 7.1 & 4.6 & 4.1 & 5.5 \\
CatBoost            & 22.8 & 25.2 & 24.1 & 12.8 & 12.8 & 19.5 \\
TabPFN           & 7.5 & 7.8 & 5.5 & 3.2 & 3.1 & 5.4 \\
ERM-MLP             & 32.0 & 36.6 & 29.9 & 15.1 & 14.8 & 25.7 \\
DR-TFM (GR; $k$-means) & 25.6 & 50.2 & 24.6 & 9.1 & 9.4 & 23.8 \\
DR-TFM (MS; $k$-means) & 35.9 & 45.0 & 36.6 & 16.9 & 18.8 & 30.6 \\
GEORGE ($k$-means)    & 98.4 & 117.4 & 89.1 & 40.2 & 39.8 & 77.0 \\
SPARE               & 136.6 & 212.5 & 104.4 & 34.4 & 29.7 & 103.5 \\
GroupDRO        & 183.0 & 213.0 & 163.8 & 76.2 & 84.3 & 144.1 \\
TabPFN (fine-tuned) & 452.3 & 347.0 & 123.1 & 53.1 & 94.6 & 214.1 \\
LSR (all stages) & 6,602.0 & 6,171.6 & 5,619.8 & 3,700.6 & 3,688.0 & 5,156.4 \\
MixturePFN          & 327.9 & 322.0 & 290.8 & 91.4 & 115.4 & 229.5 \\
LoCalPFN (retrieval) & 354.3 & 126.1 & 74.1 & 29.1 & 26.1 & 121.9 \\
BETA                & 1,903.5 & 1,399.6 & 887.5 & 657.3 & 255.0 & 1,020.6 \\
\bottomrule
\end{tabular}
\end{table}

Table~\ref{tab:memory} reports peak GPU memory for the methods in
Table~\ref{tab:main-runtime}, measured as the maximum memory allocated by PyTorch
during a run and averaged over three seeds.

\begin{table}[H]
\centering
\caption{Peak GPU memory in GB per reported result, averaged over three seeds.}
\label{tab:memory}
\small
\begin{tabular}{l ccccc | c}
\toprule
Method & Adult & Bank & Default & Shoppers & Taxi & Avg. \\
\midrule
TabPFN           & 0.77 & 0.77 & 0.62 & 0.44 & 0.43 & 0.61 \\
DR-TFM (GR; $k$-means) & 1.37 & 1.62 & 1.15 & 0.60 & 0.61 & 1.07 \\
DR-TFM (MS; $k$-means) & 1.27 & 1.41 & 1.02 & 0.61 & 0.58 & 0.98 \\
GEORGE ($k$-means)    & 1.62 & 1.63 & 1.60 & 0.71 & 0.51 & 1.21 \\
TabPFN (fine-tuned) & 10.98 & 14.45 & 11.81 & 5.25 & 3.75 & 9.25 \\
LSR (all stages) & 3.56 & 5.07 & 4.82 & 1.50 & 0.67 & 3.12 \\
MixturePFN          & 2.14 & 2.17 & 2.46 & 2.25 & 1.83 & 2.17 \\
BETA                & 14.32 & 14.32 & 14.40 & 14.34 & 14.30 & 14.33 \\
\bottomrule
\end{tabular}
\end{table}

\subsection{Training objective and adaptation strategy}
\label{appendix:objswap}
For the comparison in Table~\ref{tab:main-objswap}, DRO-based adaptation uses
the objective and model selection of DR-TFM (GR),
while ERM-based adaptation follows the standard TabPFN fine-tuning objective
in \eqref{eq:standard_finetuning}.
Input-adapter tuning uses a single ensemble member of the BETA adapter architecture
\citep{LiuBeta2025} and the training schedule used for query scaling.
Table~\ref{tab:objswap} reports the per-dataset results for input-adapter,
full-backbone, encoder, decoder, and query-scaling adaptation.
Tables~\ref{tab:objswap-runtime} and~\ref{tab:objswap-memory} report the corresponding
runtime and peak GPU memory per result.

\paragraph{Fine-tuned parameters.}
The pretrained TabPFN contains $53{,}153{,}144$ parameters.
Encoder fine-tuning updates all parameters outside the decoder ($52{,}725{,}752$).
Decoder fine-tuning updates its key and query projections and scaling networks
($427{,}392$), whereas query scaling updates only $8{,}320$ parameters.
The input adapter contains $33{,}481$--$82{,}536$ parameters across the five
benchmarks ($0.06\%$--$0.16\%$ of the combined adapter and pretrained model),
depending on the input dimension and categorical cardinalities.

\begin{table}[H]
\centering
\caption{Mean group accuracy and worst-group accuracy under ERM-based and DRO-based adaptation (GR) for
five adaptation strategies, averaged over three seeds. Avg.\ denotes the average
of each metric across the five tabular benchmarks.
Values in parentheses are standard deviations over three seeds.
Parenthesized values in the Avg.\ columns are the averages of the per-dataset standard deviations.}
\label{tab:objswap}
\scriptsize
\setlength{\tabcolsep}{3.4pt}
\begin{tabular*}{\textwidth}{@{\extracolsep{\fill}}lcccccccccc|cc}
\toprule
\multirow{2}{*}{Adaptation} & \multicolumn{2}{c}{Adult} & \multicolumn{2}{c}{Bank} & \multicolumn{2}{c}{Default} & \multicolumn{2}{c}{Shoppers} & \multicolumn{2}{c|}{Taxi} & \multicolumn{2}{c}{Avg.} \\
\cmidrule(lr){2-3}\cmidrule(lr){4-5}\cmidrule(lr){6-7}\cmidrule(lr){8-9}\cmidrule(lr){10-11}\cmidrule(lr){12-13}
 & Mean & Worst & Mean & Worst & Mean & Worst & Mean & Worst & Mean & Worst & Mean & Worst \\
\midrule
\multicolumn{13}{@{}l}{\emph{Input adapter}}\\
\quad ERM & 74.34 & 46.89 & 69.39 & 36.98 & 64.39 & 30.34 & 78.31 & 55.67 & 75.38 & 57.36 & 72.36 & 45.45 \\[-1.2pt]
 & \sd{0.67} & \sd{1.58} & \sd{1.53} & \sd{2.81} & \sd{0.10} & \sd{0.53} & \sd{0.91} & \sd{2.12} & \sd{0.14} & \sd{0.22} & \sd{0.67} & \sd{1.45} \\[2pt]
\quad \textbf{DRO} & 78.83 & \underline{65.90} & 82.70 & 73.21 & 67.92 & 47.05 & 84.96 & 79.60 & \textbf{77.23} & 67.43 & 78.33 & 66.64 \\[-1.2pt]
 & \sd{0.51} & \sd{1.58} & \sd{1.34} & \sd{1.38} & \sd{0.38} & \sd{1.44} & \sd{0.61} & \sd{0.57} & \sd{0.27} & \sd{0.97} & \sd{0.62} & \sd{1.19} \\[2pt]
\addlinespace[2pt]
\multicolumn{13}{@{}l}{\emph{Full backbone}}\\
\quad ERM & 76.33 & 50.80 & 73.10 & 42.99 & 65.03 & 31.71 & 78.88 & 56.25 & 76.11 & 58.23 & 73.89 & 48.00 \\[-1.2pt]
 & \sd{0.35} & \sd{0.47} & \sd{0.02} & \sd{0.21} & \sd{0.40} & \sd{2.04} & \sd{0.76} & \sd{1.98} & \sd{0.27} & \sd{0.66} & \sd{0.36} & \sd{1.07} \\[2pt]
\quad \textbf{DRO} & 78.01 & 59.56 & \textbf{85.32} & \underline{79.81} & 69.02 & 60.18 & 83.34 & 72.73 & 75.72 & 72.64 & 78.28 & 68.98 \\[-1.2pt]
 & \sd{0.86} & \sd{3.82} & \sd{0.80} & \sd{1.36} & \sd{0.70} & \sd{2.86} & \sd{1.99} & \sd{8.66} & \sd{0.91} & \sd{1.65} & \sd{1.05} & \sd{3.67} \\[2pt]
\addlinespace[2pt]
\multicolumn{13}{@{}l}{\emph{Encoder}}\\
\quad ERM & 77.05 & 52.15 & 73.50 & 44.01 & 65.02 & 31.29 & 78.88 & 56.25 & 76.13 & 58.23 & 74.11 & 48.39 \\[-1.2pt]
 & \sd{0.13} & \sd{0.56} & \sd{0.20} & \sd{1.13} & \sd{0.38} & \sd{1.32} & \sd{0.76} & \sd{1.98} & \sd{0.33} & \sd{0.66} & \sd{0.36} & \sd{1.13} \\[2pt]
\quad \textbf{DRO} & 77.96 & 59.40 & \underline{85.23} & 79.41 & 69.18 & \textbf{60.42} & 83.30 & 72.94 & 75.40 & 72.41 & 78.22 & 68.92 \\[-1.2pt]
 & \sd{0.85} & \sd{3.70} & \sd{0.69} & \sd{0.96} & \sd{0.47} & \sd{2.34} & \sd{1.93} & \sd{9.01} & \sd{0.73} & \sd{1.44} & \sd{0.93} & \sd{3.49} \\[2pt]
\addlinespace[2pt]
\multicolumn{13}{@{}l}{\emph{Decoder}}\\
\quad ERM & 74.84 & 48.17 & 72.78 & 42.91 & 64.89 & 31.66 & 79.89 & 61.31 & 76.30 & 59.32 & 73.74 & 48.68 \\[-1.2pt]
 & \sd{0.48} & \sd{1.27} & \sd{1.02} & \sd{1.70} & \sd{1.03} & \sd{2.44} & \sd{0.19} & \sd{0.80} & \sd{0.08} & \sd{0.26} & \sd{0.56} & \sd{1.29} \\[2pt]
\quad \textbf{DRO} & \underline{79.24} & 65.50 & 84.49 & 78.88 & \underline{69.33} & 59.42 & \textbf{85.86} & \textbf{82.51} & 76.25 & \textbf{73.31} & \underline{79.03} & \underline{71.93} \\[-1.2pt]
 & \sd{0.33} & \sd{0.88} & \sd{0.48} & \sd{1.35} & \sd{0.58} & \sd{1.13} & \sd{0.40} & \sd{0.17} & \sd{0.69} & \sd{0.93} & \sd{0.49} & \sd{0.89} \\[2pt]
\addlinespace[2pt]
\multicolumn{13}{@{}l}{\emph{Query scaling}}\\
\quad ERM & 74.86 & 48.58 & 72.29 & 41.83 & 64.85 & 30.94 & 78.62 & 56.42 & 76.05 & 58.92 & 73.34 & 47.34 \\[-1.2pt]
 & \sd{0.16} & \sd{0.52} & \sd{1.02} & \sd{1.99} & \sd{0.54} & \sd{1.31} & \sd{0.32} & \sd{1.45} & \sd{0.27} & \sd{0.79} & \sd{0.46} & \sd{1.21} \\[2pt]
\quad \textbf{DRO} & \textbf{79.64} & \textbf{67.90} & 85.01 & \textbf{80.14} & \textbf{69.42} & \underline{60.37} & \underline{85.79} & \underline{81.66} & \underline{76.69} & \underline{72.79} & \textbf{79.31} & \textbf{72.57} \\[-1.2pt]
 & \sd{0.20} & \sd{1.61} & \sd{0.36} & \sd{0.53} & \sd{0.62} & \sd{1.70} & \sd{0.32} & \sd{0.50} & \sd{0.19} & \sd{0.62} & \sd{0.34} & \sd{0.99} \\[2pt]
\bottomrule
\end{tabular*}
\end{table}

\begin{table}[H]
\centering
\caption{Runtime in seconds per reported result under DRO-based adaptation for the
adaptation strategies in Table~\ref{tab:main-objswap}, averaged over three seeds.
Times include group estimation, adaptation, and prediction.}
\label{tab:objswap-runtime}
\small
\begin{tabular}{l ccccc | c}
\toprule
Fine-tuned parameters & Adult & Bank & Default & Shoppers & Taxi & Avg. \\
\midrule
Input adapter & 56.5 & 63.0 & 40.8 & 26.8 & 30.3 & 43.5 \\
Encoder & 214.2 & 439.1 & 364.7 & 80.0 & 107.8 & 241.2 \\
Full backbone & 213.8 & 438.3 & 383.8 & 90.3 & 106.8 & 246.6 \\
Decoder & 29.9 & 45.3 & 25.9 & 9.6 & 9.5 & 24.0 \\
Query scaling & 25.6 & 50.2 & 24.6 & 9.1 & 9.4 & 23.8 \\
\bottomrule
\end{tabular}
\end{table}

\begin{table}[H]
\centering
\caption{Peak GPU memory in GB per reported result under DRO-based adaptation for the
adaptation strategies in Table~\ref{tab:main-objswap}, averaged over three seeds.}
\label{tab:objswap-memory}
\small
\begin{tabular}{l ccccc | c}
\toprule
Fine-tuned parameters & Adult & Bank & Default & Shoppers & Taxi & Avg. \\
\midrule
Input adapter & 12.58 & 15.45 & 10.86 & 5.02 & 5.09 & 9.80 \\
Encoder & 12.22 & 16.10 & 13.20 & 5.32 & 4.27 & 10.22 \\
Full backbone & 12.34 & 16.24 & 13.29 & 5.36 & 4.31 & 10.31 \\
Decoder & 1.38 & 1.68 & 1.18 & 0.62 & 0.63 & 1.10 \\
Query scaling & 1.37 & 1.62 & 1.15 & 0.60 & 0.61 & 1.07 \\
\bottomrule
\end{tabular}
\end{table}

\subsection{Discussion of the \texorpdfstring{$\chi^2$}{chi-squared}-DRO formulation}
\label{appendix:fdro}
We examine a variant that reweights individual samples within a $\chi^2$
ambiguity set, without estimating groups.

\paragraph{Formulation.}
The variant retains the fixed pretrained TabPFN feature map and adapts the same
query scaling network. Write
$\ell_\theta(x,y)=\ell(f_\theta(x;\mathcal D^{(c)}),y)$.
We minimize the robust risk
\begin{equation*}
  \sup_{Q\,:\,D_{\chi^2}(Q\,\|\,\widehat P^{(e)})\,\le\,\rho} \mathbb{E}_{Q}[\ell_\theta]
  \;=\;
  \inf_{\eta\in\mathbb{R}}
  \Big\{ \sqrt{2\rho+1}\,\big(\mathbb{E}_{\widehat P^{(e)}}\big[(\ell_\theta-\eta)_{+}^{2}\big]\big)^{1/2} + \eta \Big\}
\end{equation*}
over $\theta_L$. Here $D_{\chi^2}(Q\|\widehat P^{(e)})=\tfrac12\mathbb E_{\widehat P^{(e)}}[(dQ/d\widehat P^{(e)}-1)^2]$ for
$Q\ll\widehat P^{(e)}$, $[a]_+=\max\{a,0\}$, and $\eta$ is a scalar threshold.
The right-hand side is the dual formulation of \citet{duchi2021statistics}.
We parameterize the radius as $\rho_\alpha=(1/\alpha-1)^2$ for $\alpha\in(0,1]$.
Thus, $\alpha=1$ recovers the empirical mean, while smaller values allow more
weight to be assigned to high-loss samples.

\paragraph{Empirical results.}
Table~\ref{tab:fdro} reports results for each tested $\alpha$.
In both evaluation suites, the empirical mean objective ($\alpha=1$) achieves a
higher average of worst-group accuracies than every tested $\alpha<1$.
These results suggest that emphasizing high-loss samples alone may be
insufficient to improve worst-group performance.
This is consistent with the observation of \citet{tong2025latent} that
upweighting misclassified samples does not necessarily address the underlying
data imbalance.

\begin{table}[H]
\centering
\caption{Average of mean group accuracies and average of worst-group accuracies for the $\chi^2$-DRO variant at each tested
$\alpha$, over the five tabular datasets or all six ACS settings, and over three seeds.
$\alpha=1$ uses the empirical mean loss. All tested values are reported.
$^{\dagger}$ marks averages computed from accuracy results reported by \citet{tong2025latent}.}
\label{tab:fdro}
\small
\begin{tabular}{l cc cc}
\toprule
\multirow{2}{*}{Method} & \multicolumn{2}{c}{Five tabular benchmarks} & \multicolumn{2}{c}{ACS Income} \\
\cmidrule(lr){2-3}\cmidrule(lr){4-5}
 & Mean & Worst & Mean & Worst \\
\midrule
ERM-MLP & 72.14 & 44.67 & 77.54 & 63.19 \\
$\chi^2$-DRO$^{\dagger}$ & 70.00 & 47.82 & 74.62 & 62.95 \\
TabPFN-$\chi^2$-DRO ($\alpha=1$) & 73.34 & 47.34 & 78.73 & 65.26 \\
TabPFN-$\chi^2$-DRO ($\alpha=0.6$) & 73.16 & 47.19 & 78.02 & 63.27 \\
TabPFN-$\chi^2$-DRO ($\alpha=0.4$) & 69.65 & 37.75 & 76.93 & 59.40 \\
TabPFN-$\chi^2$-DRO ($\alpha=0.2$) & 69.49 & 37.21 & 76.96 & 59.55 \\
TabPFN-$\chi^2$-DRO ($\alpha=0.1$) & 69.43 & 37.14 & 77.15 & 60.08 \\
\bottomrule
\end{tabular}
\end{table}

\subsection{Extension to other tabular foundation models}
\label{appendix:tfm-extension-results}
Tables~\ref{appendix:tab1-tfm-extension} and~\ref{appendix:tab2-tfm-extension}
report the per-dataset and per-setting results summarized in Table~\ref{tab:tfm-extension}.
For each model, we list the pretrained model, DR-TFM (GR) without true group labels
(``+ Ours''), and DR-TFM (GR) with true group labels (``+ Ours$^{*}$'').
The starred rows follow the true-group protocol of Appendix~\ref{appendix:ms-sources};
the TabPFN + Ours$^{*}$ results are those of
Tables~\ref{appendix:tab1-known} and~\ref{appendix:tab2-known}.
Parenthesized values in the Avg.\ columns are the averages of the per-dataset or per-setting standard deviations across seeds.

\paragraph{Settings.}
\label{appendix:tfm-extension}
We follow the training settings in Appendix~\ref{appendix:protocol} and use
the inference settings specified for each model.
Query scaling is applied to the queries in the final attention layer that
attends to the context. For models without a pretrained scaling network, we
add the network $g_{\theta_L}$ of
Appendix~\ref{appendix:scaling-implementation} with its second layer
initialized at zero; its modulation $\mathbf 1+\tanh(g_{\theta_L}(u))$
multiplies the queries while each model's own attention scaling is kept. The
remaining model parameters are kept fixed. Without true group labels, groups
are estimated from the representations entering the output network, averaged
over ensemble members, as for TabPFN.
\begin{itemize}[left=0pt]
    \item TabPFN-3.5. This model uses the same decoder architecture as TabPFN.
    We use 4 ensemble members and fine-tune its pretrained query scaling
    network, updating $0.004\%$ of the model's parameters.
    \item EXAONE. We add the scaling network to scale the queries in
    the final attention layer operating across samples. We use 8 ensemble
    members and update $0.020\%$ of the adapted model's parameters.
    \item TabFM. We use 32 ensemble members and update $0.002\%$ of the
    adapted model's parameters.
    \item Causilo. We add the scaling network to scale the queries in the
    final attention layer of its prediction network. We use 8 ensemble members
    and update $0.046\%$ of the adapted model's parameters. Groups are instead
    estimated from the representations entering the final linear layer of the
    output network, since context representations form a separate cluster at
    the input to the output network.
\end{itemize}

\begin{table}[H]
\caption{GR adaptation across tabular foundation models on the five tabular benchmarks.
``+ Ours'' denotes DR-TFM (GR) applied to the model above; $^{*}$ indicates use of true group labels.
Avg.\ denotes the average of each metric across datasets.
Non-bold parentheses report standard deviations over three seeds.
Bold parentheses report changes from the corresponding pretrained model, in percentage points.}
\label{appendix:tab1-tfm-extension}
\centering
\scriptsize
\newcommand{\ppgain}[1]{{\fontsize{6}{6}\selectfont\bfseries (#1)}}
\setlength{\tabcolsep}{2pt}
\begin{tabular*}{\textwidth}{@{\extracolsep{\fill}}lcccccccccc|cc}
\toprule
\multirow{2}{*}{Methods} & \multicolumn{2}{c}{Adult} & \multicolumn{2}{c}{Bank} & \multicolumn{2}{c}{Default} & \multicolumn{2}{c}{Shoppers} & \multicolumn{2}{c|}{Taxi} & \multicolumn{2}{c}{Avg.} \\
\cmidrule(lr){2-3}\cmidrule(lr){4-5}\cmidrule(lr){6-7}\cmidrule(lr){8-9}\cmidrule(lr){10-11}\cmidrule(lr){12-13}
 & Mean & Worst & Mean & Worst & Mean & Worst & Mean & Worst & Mean & Worst & Mean & Worst \\
\midrule
TabPFN & 75.33 & 49.17 & 73.07 & 42.96 & 64.77 & 31.09 & 79.59 & 58.29 & 76.01 & 58.19 & 73.75 & 47.94\\[-1.2pt]
 & \sd{0.05} & \sd{0.28} & \sd{0.31} & \sd{0.76} & \sd{0.04} & \sd{0.04} & \sd{0.32} & \sd{0.52} & \sd{0.25} & \sd{0.53} & \sd{0.19} & \sd{0.43}\\[2pt]
\textbf{+ Ours} & 79.64 & 67.90 & 85.01 & 80.14 & 69.42 & 60.37 & 85.79 & 81.66 & 76.69 & 72.79 & 79.31 & 72.57\\[-1.2pt]
 & \sd{0.20} & \sd{1.61} & \sd{0.36} & \sd{0.53} & \sd{0.62} & \sd{1.70} & \sd{0.32} & \sd{0.50} & \sd{0.19} & \sd{0.62} & \sd{0.34} & \sd{0.99}\\[-1.2pt]
 & \ppgain{+4.31} & \ppgain{+18.73} & \ppgain{+11.94} & \ppgain{+37.18} & \ppgain{+4.65} & \ppgain{+29.28} & \ppgain{+6.20} & \ppgain{+23.37} & \ppgain{+0.68} & \ppgain{+14.60} & \ppgain{+5.56} & \ppgain{+24.63}\\[2pt]
\textbf{+ Ours}$^{*}$ & 78.36 & 72.26 & 86.12 & 81.94 & 69.58 & 65.30 & 85.21 & 80.52 & 77.03 & 68.28 & 79.26 & 73.66\\[-1.2pt]
 & \sd{0.52} & \sd{0.79} & \sd{0.93} & \sd{0.34} & \sd{0.26} & \sd{1.19} & \sd{0.39} & \sd{2.17} & \sd{0.39} & \sd{2.32} & \sd{0.50} & \sd{1.36}\\[-1.2pt]
 & \ppgain{+3.03} & \ppgain{+23.09} & \ppgain{+13.05} & \ppgain{+38.98} & \ppgain{+4.81} & \ppgain{+34.21} & \ppgain{+5.62} & \ppgain{+22.23} & \ppgain{+1.02} & \ppgain{+10.09} & \ppgain{+5.51} & \ppgain{+25.72}\\[2pt]
\cmidrule[0.03em](lr){1-13}
TabPFN-3.5 & 77.79 & 54.25 & 73.83 & 45.31 & 64.87 & 31.60 & 78.92 & 56.12 & 76.56 & 59.00 & 74.39 & 49.26\\[-1.2pt]
 & \sd{0.09} & \sd{0.30} & \sd{0.06} & \sd{0.38} & \sd{0.11} & \sd{0.26} & \sd{0.16} & \sd{0.55} & \sd{0.16} & \sd{0.59} & \sd{0.12} & \sd{0.42}\\[2pt]
\textbf{+ Ours} & 81.96 & 71.82 & 86.40 & 79.81 & 68.92 & 58.94 & 85.17 & 79.92 & 74.66 & 66.83 & 79.42 & 71.47\\[-1.2pt]
 & \sd{0.26} & \sd{0.44} & \sd{0.41} & \sd{0.40} & \sd{0.27} & \sd{0.17} & \sd{0.46} & \sd{1.13} & \sd{2.81} & \sd{2.06} & \sd{0.84} & \sd{0.84}\\[-1.2pt]
 & \ppgain{+4.17} & \ppgain{+17.57} & \ppgain{+12.57} & \ppgain{+34.50} & \ppgain{+4.05} & \ppgain{+27.34} & \ppgain{+6.25} & \ppgain{+23.80} & \ppgain{-1.90} & \ppgain{+7.83} & \ppgain{+5.03} & \ppgain{+22.21}\\[2pt]
\textbf{+ Ours}$^{*}$ & 80.20 & 74.69 & 86.90 & 81.81 & 69.00 & 65.94 & 85.88 & 82.58 & 77.47 & 68.68 & 79.89 & 74.74\\[-1.2pt]
 & \sd{0.84} & \sd{1.08} & \sd{0.22} & \sd{0.52} & \sd{0.33} & \sd{0.99} & \sd{0.21} & \sd{0.57} & \sd{0.33} & \sd{0.23} & \sd{0.39} & \sd{0.68}\\[-1.2pt]
 & \ppgain{+2.41} & \ppgain{+20.44} & \ppgain{+13.07} & \ppgain{+36.50} & \ppgain{+4.13} & \ppgain{+34.34} & \ppgain{+6.96} & \ppgain{+26.46} & \ppgain{+0.91} & \ppgain{+9.68} & \ppgain{+5.50} & \ppgain{+25.48}\\[2pt]
\cmidrule[0.03em](lr){1-13}
EXAONE & 77.47 & 53.46 & 72.67 & 42.35 & 64.74 & 31.10 & 77.74 & 53.72 & 75.99 & 57.77 & 73.72 & 47.68\\[-1.2pt]
 & \sd{0.02} & \sd{0.05} & \sd{0.18} & \sd{0.26} & \sd{0.12} & \sd{0.27} & \sd{0.16} & \sd{0.29} & \sd{0.06} & \sd{0.16} & \sd{0.11} & \sd{0.21} \\[2pt]
\textbf{+ Ours} & 81.12 & 67.82 & 84.53 & 77.46 & 68.86 & 56.15 & 85.37 & 80.48 & 77.32 & 65.34 & 79.44 & 69.45\\[-1.2pt]
 & \sd{0.07} & \sd{0.35} & \sd{0.53} & \sd{1.39} & \sd{0.07} & \sd{5.45} & \sd{0.16} & \sd{1.35} & \sd{0.50} & \sd{2.48} & \sd{0.27} & \sd{2.21}\\[-1.2pt]
 & \ppgain{+3.65} & \ppgain{+14.36} & \ppgain{+11.86} & \ppgain{+35.11} & \ppgain{+4.12} & \ppgain{+25.05} & \ppgain{+7.63} & \ppgain{+26.76} & \ppgain{+1.33} & \ppgain{+7.57} & \ppgain{+5.72} & \ppgain{+21.77}\\[2pt]
\textbf{+ Ours}$^{*}$ & 81.24 & 74.34 & 86.73 & 80.45 & 69.56 & 63.70 & 85.61 & 80.85 & 77.37 & 68.49 & 80.10 & 73.57\\[-1.2pt]
 & \sd{0.24} & \sd{1.32} & \sd{0.04} & \sd{0.69} & \sd{0.39} & \sd{3.41} & \sd{0.29} & \sd{2.79} & \sd{0.27} & \sd{0.70} & \sd{0.25} & \sd{1.78}\\[-1.2pt]
 & \ppgain{+3.77} & \ppgain{+20.88} & \ppgain{+14.06} & \ppgain{+38.10} & \ppgain{+4.82} & \ppgain{+32.60} & \ppgain{+7.87} & \ppgain{+27.13} & \ppgain{+1.38} & \ppgain{+10.72} & \ppgain{+6.38} & \ppgain{+25.89}\\[2pt]
\cmidrule[0.03em](lr){1-13}
TabFM & 77.39 & 52.88 & 73.75 & 44.42 & 65.02 & 31.91 & 78.28 & 54.96 & 76.34 & 58.11 & 74.15 & 48.46\\[-1.2pt]
 & \sd{0.02} & \sd{0.10} & \sd{0.06} & \sd{0.10} & \sd{0.03} & \sd{0.30} & \sd{0.25} & \sd{0.63} & \sd{0.13} & \sd{0.28} & \sd{0.10} & \sd{0.28} \\[2pt]
\textbf{+ Ours} & 80.06 & 64.25 & 86.21 & 73.96 & 69.07 & 52.03 & 85.38 & 81.67 & 76.42 & 60.44 & 79.43 & 66.47\\[-1.2pt]
 & \sd{2.09} & \sd{7.86} & \sd{0.66} & \sd{1.28} & \sd{0.75} & \sd{0.94} & \sd{0.56} & \sd{1.97} & \sd{0.65} & \sd{1.65} & \sd{0.94} & \sd{2.74}\\[-1.2pt]
 & \ppgain{+2.67} & \ppgain{+11.37} & \ppgain{+12.46} & \ppgain{+29.54} & \ppgain{+4.05} & \ppgain{+20.12} & \ppgain{+7.10} & \ppgain{+26.71} & \ppgain{+0.08} & \ppgain{+2.33} & \ppgain{+5.28} & \ppgain{+18.01}\\[2pt]
\textbf{+ Ours}$^{*}$ & 81.38 & 72.66 & 87.49 & 78.83 & 69.87 & 59.67 & 85.71 & 81.36 & 77.62 & 68.70 & 80.41 & 72.25\\[-1.2pt]
 & \sd{0.24} & \sd{0.66} & \sd{0.23} & \sd{0.18} & \sd{0.26} & \sd{1.13} & \sd{0.57} & \sd{0.55} & \sd{0.38} & \sd{0.65} & \sd{0.33} & \sd{0.63}\\[-1.2pt]
 & \ppgain{+3.99} & \ppgain{+19.78} & \ppgain{+13.74} & \ppgain{+34.41} & \ppgain{+4.85} & \ppgain{+27.76} & \ppgain{+7.43} & \ppgain{+26.40} & \ppgain{+1.28} & \ppgain{+10.59} & \ppgain{+6.26} & \ppgain{+23.79}\\[2pt]
\cmidrule[0.03em](lr){1-13}
Causilo & 77.25 & 52.89 & 72.95 & 43.27 & 64.70 & 31.27 & 78.60 & 56.44 & 76.28 & 58.98 & 73.95 & 48.57\\[-1.2pt]
 & \sd{0.12} & \sd{0.40} & \sd{0.23} & \sd{0.47} & \sd{0.13} & \sd{0.29} & \sd{0.05} & \sd{0.00} & \sd{0.11} & \sd{0.07} & \sd{0.13} & \sd{0.25} \\[2pt]
\textbf{+ Ours} & 80.53 & 65.96 & 83.25 & 65.25 & 66.82 & 44.44 & 84.96 & 79.77 & 76.77 & 62.41 & 78.47 & 63.56\\[-1.2pt]
 & \sd{0.13} & \sd{0.62} & \sd{0.13} & \sd{0.97} & \sd{0.31} & \sd{0.65} & \sd{0.14} & \sd{1.70} & \sd{0.35} & \sd{1.58} & \sd{0.21} & \sd{1.10}\\[-1.2pt]
 & \ppgain{+3.28} & \ppgain{+13.07} & \ppgain{+10.30} & \ppgain{+21.98} & \ppgain{+2.12} & \ppgain{+13.17} & \ppgain{+6.36} & \ppgain{+23.33} & \ppgain{+0.49} & \ppgain{+3.43} & \ppgain{+4.52} & \ppgain{+14.99}\\[2pt]
\textbf{+ Ours}$^{*}$ & 80.68 & 73.52 & 86.40 & 81.82 & 69.61 & 64.09 & 85.59 & 81.59 & 77.68 & 68.38 & 79.99 & 73.88\\[-1.2pt]
 & \sd{0.49} & \sd{0.61} & \sd{0.36} & \sd{0.50} & \sd{0.51} & \sd{4.86} & \sd{1.01} & \sd{0.95} & \sd{0.40} & \sd{0.79} & \sd{0.55} & \sd{1.54}\\[-1.2pt]
 & \ppgain{+3.43} & \ppgain{+20.63} & \ppgain{+13.45} & \ppgain{+38.55} & \ppgain{+4.91} & \ppgain{+32.82} & \ppgain{+6.99} & \ppgain{+25.15} & \ppgain{+1.40} & \ppgain{+9.40} & \ppgain{+6.04} & \ppgain{+25.31}\\[2pt]
\bottomrule
\end{tabular*}
\end{table}

\begin{table}[H]
\caption{GR adaptation across tabular foundation models on ACS Income.
``+ Ours'' denotes DR-TFM (GR) applied to the model above; $^{*}$ indicates use of true group labels.
The left block reports within-state results and the right block reports transfers between states.
Avg.\ denotes the average of each metric across the three settings in its block.
Non-bold parentheses report standard deviations over three seeds.
Bold parentheses report changes from the corresponding pretrained model, in percentage points.}
\label{appendix:tab2-tfm-extension}
\centering
\scriptsize
\newcommand{\ppgain}[1]{{\fontsize{6}{6}\selectfont\bfseries (#1)}}
\setlength{\tabcolsep}{0.5pt}
\begin{tabular*}{\textwidth}{@{\extracolsep{\fill}}lcccccc|cc|cccccc|cc}
\toprule
\multirow{2}{*}{Methods} & \multicolumn{2}{c}{AZ} & \multicolumn{2}{c}{MA} & \multicolumn{2}{c|}{MI} & \multicolumn{2}{c|}{Avg.} & \multicolumn{2}{c}{AZ $\rightarrow$ MA} & \multicolumn{2}{c}{MA $\rightarrow$ MI} & \multicolumn{2}{c|}{MI $\rightarrow$ AZ} & \multicolumn{2}{c}{Avg.} \\
\cmidrule(lr){2-3}\cmidrule(lr){4-5}\cmidrule(lr){6-7}\cmidrule(lr){8-9}\cmidrule(lr){10-11}\cmidrule(lr){12-13}\cmidrule(lr){14-15}\cmidrule(lr){16-17}
 & Mean & Worst & Mean & Worst & Mean & Worst & Mean & Worst & Mean & Worst & Mean & Worst & Mean & Worst & Mean & Worst \\
\midrule
TabPFN & 78.19 & 59.87 & 83.01 & 77.02 & 77.72 & 63.50 & 79.64 & 66.80 & 77.95 & 61.57 & 79.69 & 71.86 & 76.81 & 59.08 & 78.15 & 64.17\\[-1.2pt]
 & \sd{0.32} & \sd{0.50} & \sd{0.12} & \sd{0.27} & \sd{0.03} & \sd{0.20} & \sd{0.16} & \sd{0.32} & \sd{0.07} & \sd{0.18} & \sd{0.08} & \sd{0.25} & \sd{0.05} & \sd{0.12} & \sd{0.07} & \sd{0.18}\\[2pt]
\textbf{+ Ours} & 78.81 & 70.36 & 79.96 & 70.93 & 78.87 & 71.90 & 79.21 & 71.06 & 78.43 & 69.70 & 78.73 & 66.43 & 78.46 & 68.84 & 78.54 & 68.32\\[-1.2pt]
 & \sd{0.13} & \sd{1.00} & \sd{0.59} & \sd{1.31} & \sd{0.30} & \sd{0.86} & \sd{0.34} & \sd{1.06} & \sd{0.15} & \sd{0.72} & \sd{0.68} & \sd{0.34} & \sd{0.15} & \sd{0.59} & \sd{0.32} & \sd{0.55}\\[-1.2pt]
 & \ppgain{+0.62} & \ppgain{+10.49} & \ppgain{-3.05} & \ppgain{-6.09} & \ppgain{+1.15} & \ppgain{+8.40} & \ppgain{-0.43} & \ppgain{+4.26} & \ppgain{+0.48} & \ppgain{+8.13} & \ppgain{-0.96} & \ppgain{-5.43} & \ppgain{+1.65} & \ppgain{+9.76} & \ppgain{+0.39} & \ppgain{+4.15}\\[2pt]
\textbf{+ Ours}$^{*}$ & 79.55 & 76.36 & 81.45 & 79.27 & 78.64 & 76.09 & 79.88 & 77.24 & 79.61 & 78.02 & 79.37 & 76.36 & 78.50 & 76.25 & 79.16 & 76.88\\[-1.2pt]
 & \sd{0.40} & \sd{0.91} & \sd{0.34} & \sd{0.49} & \sd{0.25} & \sd{0.80} & \sd{0.33} & \sd{0.73} & \sd{0.47} & \sd{1.54} & \sd{0.03} & \sd{0.93} & \sd{0.05} & \sd{0.37} & \sd{0.19} & \sd{0.95}\\[-1.2pt]
 & \ppgain{+1.36} & \ppgain{+16.49} & \ppgain{-1.56} & \ppgain{+2.25} & \ppgain{+0.92} & \ppgain{+12.59} & \ppgain{+0.24} & \ppgain{+10.44} & \ppgain{+1.66} & \ppgain{+16.45} & \ppgain{-0.32} & \ppgain{+4.50} & \ppgain{+1.69} & \ppgain{+17.17} & \ppgain{+1.01} & \ppgain{+12.71}\\[2pt]
\cmidrule[0.03em](lr){1-17}
TabPFN-3.5 & 78.18 & 59.60 & 83.19 & 77.23 & 78.00 & 64.47 & 79.79 & 67.10 & 78.00 & 61.70 & 79.80 & 72.23 & 76.94 & 59.08 & 78.25 & 64.34\\[-1.2pt]
 & \sd{0.29} & \sd{1.01} & \sd{0.25} & \sd{0.44} & \sd{0.03} & \sd{0.20} & \sd{0.19} & \sd{0.55} & \sd{0.02} & \sd{0.06} & \sd{0.08} & \sd{0.23} & \sd{0.04} & \sd{0.16} & \sd{0.05} & \sd{0.15}\\[2pt]
\textbf{+ Ours} & 79.77 & 71.01 & 81.28 & 73.80 & 78.60 & 72.62 & 79.88 & 72.48 & 79.77 & 74.91 & 79.19 & 66.85 & 78.59 & 72.77 & 79.19 & 71.51\\[-1.2pt]
 & \sd{0.40} & \sd{0.78} & \sd{0.11} & \sd{1.25} & \sd{0.46} & \sd{1.73} & \sd{0.32} & \sd{1.25} & \sd{0.25} & \sd{1.06} & \sd{0.20} & \sd{3.39} & \sd{0.01} & \sd{1.60} & \sd{0.15} & \sd{2.02}\\[-1.2pt]
 & \ppgain{+1.59} & \ppgain{+11.41} & \ppgain{-1.91} & \ppgain{-3.43} & \ppgain{+0.60} & \ppgain{+8.15} & \ppgain{+0.09} & \ppgain{+5.38} & \ppgain{+1.77} & \ppgain{+13.21} & \ppgain{-0.61} & \ppgain{-5.38} & \ppgain{+1.65} & \ppgain{+13.69} & \ppgain{+0.94} & \ppgain{+7.17}\\[2pt]
\textbf{+ Ours}$^{*}$ & 79.90 & 76.20 & 81.42 & 78.95 & 78.96 & 75.11 & 80.09 & 76.75 & 79.73 & 77.74 & 79.37 & 74.99 & 78.56 & 76.24 & 79.22 & 76.32\\[-1.2pt]
 & \sd{0.29} & \sd{0.39} & \sd{0.28} & \sd{0.89} & \sd{0.23} & \sd{0.65} & \sd{0.27} & \sd{0.64} & \sd{0.26} & \sd{0.39} & \sd{0.07} & \sd{2.30} & \sd{0.22} & \sd{0.57} & \sd{0.18} & \sd{1.09}\\[-1.2pt]
 & \ppgain{+1.72} & \ppgain{+16.60} & \ppgain{-1.77} & \ppgain{+1.72} & \ppgain{+0.96} & \ppgain{+10.64} & \ppgain{+0.30} & \ppgain{+9.65} & \ppgain{+1.73} & \ppgain{+16.04} & \ppgain{-0.43} & \ppgain{+2.76} & \ppgain{+1.62} & \ppgain{+17.16} & \ppgain{+0.97} & \ppgain{+11.98}\\[2pt]
\cmidrule[0.03em](lr){1-17}
EXAONE & 78.06 & 58.93 & 82.70 & 76.05 & 77.80 & 63.50 & 79.52 & 66.16 & 78.06 & 61.52 & 79.85 & 72.02 & 76.90 & 59.03 & 78.27 & 64.19\\[-1.2pt]
 & \sd{0.02} & \sd{0.19} & \sd{0.05} & \sd{0.16} & \sd{0.07} & \sd{0.05} & \sd{0.05} & \sd{0.13} & \sd{0.02} & \sd{0.11} & \sd{0.02} & \sd{0.06} & \sd{0.08} & \sd{0.21} & \sd{0.04} & \sd{0.13} \\[2pt]
\textbf{+ Ours} & 80.26 & 71.46 & 81.73 & 74.28 & 78.85 & 68.43 & 80.28 & 71.39 & 80.33 & 73.44 & 79.52 & 66.19 & 79.03 & 71.48 & 79.63 & 70.37\\[-1.2pt]
 & \sd{0.17} & \sd{1.12} & \sd{0.22} & \sd{1.42} & \sd{0.45} & \sd{3.68} & \sd{0.28} & \sd{2.07} & \sd{0.41} & \sd{1.22} & \sd{0.43} & \sd{3.40} & \sd{0.13} & \sd{0.91} & \sd{0.32} & \sd{1.84}\\[-1.2pt]
 & \ppgain{+2.20} & \ppgain{+12.53} & \ppgain{-0.97} & \ppgain{-1.77} & \ppgain{+1.05} & \ppgain{+4.93} & \ppgain{+0.76} & \ppgain{+5.23} & \ppgain{+2.27} & \ppgain{+11.92} & \ppgain{-0.33} & \ppgain{-5.83} & \ppgain{+2.13} & \ppgain{+12.45} & \ppgain{+1.36} & \ppgain{+6.18}\\[2pt]
\textbf{+ Ours}$^{*}$ & 79.93 & 76.07 & 81.47 & 78.65 & 79.08 & 75.23 & 80.16 & 76.65 & 80.02 & 78.41 & 79.50 & 73.49 & 78.73 & 76.05 & 79.42 & 75.99\\[-1.2pt]
 & \sd{0.57} & \sd{1.05} & \sd{0.15} & \sd{0.98} & \sd{0.60} & \sd{0.38} & \sd{0.44} & \sd{0.80} & \sd{0.37} & \sd{0.43} & \sd{0.03} & \sd{1.32} & \sd{0.06} & \sd{0.37} & \sd{0.16} & \sd{0.71}\\[-1.2pt]
 & \ppgain{+1.87} & \ppgain{+17.14} & \ppgain{-1.23} & \ppgain{+2.60} & \ppgain{+1.28} & \ppgain{+11.73} & \ppgain{+0.64} & \ppgain{+10.49} & \ppgain{+1.96} & \ppgain{+16.89} & \ppgain{-0.35} & \ppgain{+1.47} & \ppgain{+1.83} & \ppgain{+17.02} & \ppgain{+1.15} & \ppgain{+11.80}\\[2pt]
\cmidrule[0.03em](lr){1-17}
TabFM & 77.90 & 59.64 & 82.95 & 77.10 & 78.05 & 64.60 & 79.63 & 67.11 & 77.97 & 61.11 & 79.90 & 72.14 & 77.09 & 59.86 & 78.32 & 64.37\\[-1.2pt]
 & \sd{0.06} & \sd{0.24} & \sd{0.05} & \sd{0.08} & \sd{0.03} & \sd{0.13} & \sd{0.05} & \sd{0.15} & \sd{0.01} & \sd{0.02} & \sd{0.01} & \sd{0.03} & \sd{0.01} & \sd{0.08} & \sd{0.01} & \sd{0.04} \\[2pt]
\textbf{+ Ours} & 79.38 & 69.88 & 80.25 & 64.83 & 79.32 & 73.51 & 79.65 & 69.41 & 79.46 & 70.70 & 78.50 & 58.87 & 78.96 & 70.22 & 78.98 & 66.60\\[-1.2pt]
 & \sd{0.36} & \sd{1.42} & \sd{0.32} & \sd{0.99} & \sd{0.10} & \sd{0.98} & \sd{0.26} & \sd{1.13} & \sd{0.22} & \sd{3.47} & \sd{0.04} & \sd{0.17} & \sd{0.13} & \sd{0.55} & \sd{0.13} & \sd{1.40}\\[-1.2pt]
 & \ppgain{+1.48} & \ppgain{+10.24} & \ppgain{-2.70} & \ppgain{-12.27} & \ppgain{+1.27} & \ppgain{+8.91} & \ppgain{+0.02} & \ppgain{+2.30} & \ppgain{+1.49} & \ppgain{+9.59} & \ppgain{-1.40} & \ppgain{-13.27} & \ppgain{+1.87} & \ppgain{+10.36} & \ppgain{+0.66} & \ppgain{+2.23}\\[2pt]
\textbf{+ Ours}$^{*}$ & 79.59 & 75.88 & 81.40 & 79.29 & 79.20 & 75.52 & 80.06 & 76.90 & 80.08 & 78.43 & 79.43 & 73.69 & 78.75 & 76.24 & 79.42 & 76.12\\[-1.2pt]
 & \sd{0.26} & \sd{0.83} & \sd{0.09} & \sd{0.89} & \sd{0.51} & \sd{0.34} & \sd{0.29} & \sd{0.69} & \sd{0.30} & \sd{0.45} & \sd{0.19} & \sd{1.93} & \sd{0.09} & \sd{0.25} & \sd{0.19} & \sd{0.88}\\[-1.2pt]
 & \ppgain{+1.69} & \ppgain{+16.24} & \ppgain{-1.55} & \ppgain{+2.19} & \ppgain{+1.15} & \ppgain{+10.92} & \ppgain{+0.43} & \ppgain{+9.79} & \ppgain{+2.11} & \ppgain{+17.32} & \ppgain{-0.47} & \ppgain{+1.55} & \ppgain{+1.66} & \ppgain{+16.38} & \ppgain{+1.10} & \ppgain{+11.75}\\[2pt]
\cmidrule[0.03em](lr){1-17}
Causilo & 77.77 & 59.45 & 82.87 & 76.52 & 77.77 & 64.04 & 79.47 & 66.67 & 77.90 & 61.53 & 79.85 & 72.02 & 76.93 & 59.31 & 78.23 & 64.29\\[-1.2pt]
 & \sd{0.24} & \sd{0.49} & \sd{0.14} & \sd{0.25} & \sd{0.16} & \sd{0.52} & \sd{0.18} & \sd{0.42} & \sd{0.09} & \sd{0.25} & \sd{0.04} & \sd{0.25} & \sd{0.05} & \sd{0.16} & \sd{0.06} & \sd{0.22} \\[2pt]
\textbf{+ Ours} & 79.53 & 70.25 & 82.26 & 75.04 & 78.88 & 71.16 & 80.22 & 72.15 & 78.51 & 71.61 & 79.11 & 67.17 & 77.50 & 65.83 & 78.37 & 68.21\\[-1.2pt]
 & \sd{0.33} & \sd{2.34} & \sd{0.21} & \sd{0.40} & \sd{0.07} & \sd{1.34} & \sd{0.20} & \sd{1.36} & \sd{1.18} & \sd{0.62} & \sd{0.79} & \sd{1.57} & \sd{1.35} & \sd{0.58} & \sd{1.11} & \sd{0.92}\\[-1.2pt]
 & \ppgain{+1.76} & \ppgain{+10.80} & \ppgain{-0.61} & \ppgain{-1.48} & \ppgain{+1.11} & \ppgain{+7.12} & \ppgain{+0.75} & \ppgain{+5.48} & \ppgain{+0.61} & \ppgain{+10.08} & \ppgain{-0.74} & \ppgain{-4.85} & \ppgain{+0.57} & \ppgain{+6.52} & \ppgain{+0.14} & \ppgain{+3.92}\\[2pt]
\textbf{+ Ours}$^{*}$ & 79.83 & 76.43 & 81.47 & 79.31 & 79.15 & 75.89 & 80.15 & 77.21 & 80.03 & 78.44 & 79.39 & 75.21 & 78.61 & 76.58 & 79.35 & 76.74\\[-1.2pt]
 & \sd{0.40} & \sd{0.88} & \sd{0.14} & \sd{0.28} & \sd{0.18} & \sd{0.17} & \sd{0.24} & \sd{0.44} & \sd{0.27} & \sd{0.66} & \sd{0.08} & \sd{2.06} & \sd{0.12} & \sd{0.13} & \sd{0.16} & \sd{0.95}\\[-1.2pt]
 & \ppgain{+2.06} & \ppgain{+16.98} & \ppgain{-1.40} & \ppgain{+2.79} & \ppgain{+1.38} & \ppgain{+11.85} & \ppgain{+0.68} & \ppgain{+10.54} & \ppgain{+2.13} & \ppgain{+16.91} & \ppgain{-0.46} & \ppgain{+3.19} & \ppgain{+1.68} & \ppgain{+17.27} & \ppgain{+1.12} & \ppgain{+12.45}\\[2pt]
\bottomrule
\end{tabular*}
\end{table}

\subsection*{AI use statement}
We used generative AI tools to assist with drafting and editing the manuscript and with implementation details.
The authors reviewed all AI-assisted content and take responsibility for the final manuscript.

\end{document}